%% file: neurips_2025.tex
\documentclass{article}

\usepackage[preprint]{neurips_2025}

\usepackage[utf8]{inputenc} 
\usepackage[T1]{fontenc}    
\usepackage{hyperref}       
\usepackage{url}            
\usepackage{booktabs}       
\usepackage{amsfonts}       
\usepackage{nicefrac}       
\usepackage{microtype}      
\usepackage{xcolor}         

\input{0_packages}

\title{Source-Guided Flow Matching}

\author{%
  Zifan Wang\thanks{First two authors have equal contribution.} \\
  KTH Royal Institute of Technology\\
  \texttt{zifanw@kth.se} \\
  \And
  Alice Harting\footnotemark[1] \\
  KTH Royal Institute of Technology\\
  \texttt{aharting@kth.se} \\
  \And
  Matthieu Barreau \\
  KTH Royal Institute of Technology\\
  \texttt{barreau@kth.se} \\
  \And
   Michael M. Zavlanos \\
  Duke University \\
   \texttt{michael.zavlanos@duke.edu} \\
  \And
   Karl H. Johansson \\
   KTH Royal Institute of Technology\\
  \texttt{kallej@kth.se} \\
}

\begin{document}

\maketitle

\begin{abstract}

Guidance of generative models is typically achieved by modifying the probability flow vector field through the addition of a guidance field. In this paper, we instead propose the Source-Guided Flow Matching (SGFM) framework, which modifies the source distribution directly while keeping the pre-trained vector field intact. This reduces the guidance problem to a well-defined problem of sampling from the source distribution. We theoretically show that SGFM recovers the desired target distribution exactly. Furthermore, we provide bounds on the Wasserstein error for the generated distribution when using an approximate sampler of the source distribution and an approximate vector field. The key benefit of our approach is that it allows the user to flexibly choose the sampling method depending on their specific problem. To illustrate this, we systematically compare different sampling methods and discuss conditions for asymptotically exact guidance. Moreover, our framework integrates well with optimal flow matching models since the straight transport map generated by the vector field is preserved. Experimental results on synthetic 2D benchmarks, physics-informed generative tasks, and imaging inverse problems demonstrate the effectiveness and flexibility of the proposed framework.
\end{abstract}

\input{1_intro}

\input{2_prelim}

\input{3_result}

\input{4_sampling}
\input{5_experiment}
\input{6_conclusion}

\clearpage
\newpage
\bibliographystyle{unsrt}
\bibliography{ref}

\include{Appendix}

\end{document}

%% file: 0_packages.tex
\usepackage{amsthm,amsmath,amssymb}
\usepackage{subfigure}
\usepackage{pgfplots}
\usepackage{cleveref}
\usepackage{multirow}
\usepackage{algorithmic}
\usepackage[ruled,vlined]{algorithm2e}       
\usepackage{wrapfig}
\usepackage{subcaption}
\usepackage{makecell}
\newtheorem{lemma}{Lemma}
\newtheorem{theorem}{Theorem}

\usepackage[export]{adjustbox}
\usepackage{graphicx}
\usepackage{booktabs}
\usepackage{caption}
\usepackage{subcaption}
\usepackage{arydshln} 

\DeclareMathOperator*{\argmax}{arg\,max}

\newcommand{\EE}{\mathbb{E}}

\usepackage{amssymb}
\usepackage{pifont}


%% file: 1_intro.tex
\section{Introduction}
Flow matching \cite{lipman2022flow} is a generative modeling framework to learn a vector field that drives the probability flow from a source distribution $q_0$ to a target distribution $q_1$ in some fixed time. It has demonstrated state-of-the-art computational efficiency and sample 
quality across a range of applications, from image generation \cite{lipman2022flow} and molecular structure generation \cite{chen2023flow} to decision‐making tasks \cite{zheng2023guided}.
In particular, optimal flow matching \cite{tong2023improving, Pooladian2023Multisample} trains the vector field by leveraging the optimal transport (OT) solution between $q_0$ and $q_1$. The resulting optimal vector field moves each sample along a straight-line trajectory with a constant velocity, corresponding to the Wasserstein geodesic between the distributions. In practice, these straight trajectories lead to stable training and faster inference, since generative sampling can then reach the target distribution with few integration steps.

The guidance of flow matching refers to directing the probability flows toward outcomes with desired properties \cite{dhariwal2021diffusion,du2023reduce,graikos2022diffusion,ho2022classifier,song2020score}. 
In this context, sample generation aims not only to approximate the data distribution but also to satisfy additional properties, such as conditioning on auxiliary information or optimizing an energy-based objective. 
Training-based guidance methods \cite{ho2022classifier,song2020score} address this by training a specialized model for a given conditioning scenario. 
While effective, these methods require retraining for every new conditioning scenario, which incurs significant cost and therefore limits their flexibility. Thus, a variety of training-free approaches have emerged for both diffusion models \cite{chung2022diffusion,song2023loss,ye2024tfg,uehara2024fine,tang2024fine} and flow matching models \cite{ben2024d,wang2024training,liu2023flowgrad,domingo2024adjoint,feng2025guidance}.

Among these methods, exact guidance is achieved by \cite{uehara2024fine,tang2024fine,domingo2024adjoint,feng2025guidance}. Specifically, \cite{uehara2024fine,tang2024fine} reformulate guidance as a stochastic optimal control (SOC) problem and achieve exactness by modifying both the source distribution and the vector field. Additionally, \cite{domingo2024adjoint} shows that exact guidance is possible by modifying only the vector field, given a suitable noise schedule. However, these methods require solving an SOC problem for every new conditioning scenario, which is computationally expensive. 
Recently, \cite{feng2025guidance} proposed a framework for exact guidance including various adjustments of the vector field. 
However, its exactness is only proved for a specific class of pre-trained vector fields, thereby having limited generality. 
Moreover, a shared feature of all of these methods is that the vector field is transformed. For optimal flow matching models, this implies that the desirable property of straight transport maps is not preserved when guidance is applied.

In this work, we show that exact guidance can instead be achieved by appropriately modifying the source distribution while keeping the original vector field unchanged. We introduce the Source-Guided Flow Matching (SGFM) framework, which reduces guidance to the well-defined task of sampling from the modified source distribution. We prove that sampling from the modified source distribution and driving the flow along the exact vector field precisely recovers the desired target distribution. Furthermore, we provide bounds on the Wasserstein error of the target distribution when using an approximate sampler of the source distribution and an approximate vector field. 

The key to effective implementation of SGFM is accurate and efficient sampling of the modified source distribution. SGFM gives the user the flexibility to tune the procedure by customizing the sampling method according to their specific problem. Such methods include importance sampling, Hamiltonian Monte Carlo, and optimization-based sampling. We discuss the asymptotic exactness of SGFM with these methods. Interestingly, SGFM with optimization-based sampling method coincides with the heuristic formulation in \cite{ben2024d}. In the context of our framework, this method is equivalent to recovering the mode of the modified source distribution. In this way, we offer a new view of \cite{ben2024d} with theoretical justification that also naturally extends to other sampling methods. Experiments on synthetic 2D datasets, physics-informed generative tasks, and imaging inverse problems demonstrate the effectiveness and flexibility that SGFM offers compared to other methods.

%% file: 2_prelim.tex
\section{Background}
Throughout this paper, we consider a generative modeling framework defined on a data space in $\mathbb{R}^d$. The generative model is characterized by a source distribution $q_0$ and a target distribution $q_1$. The source distribution is an arbitrary distribution from which samples can be drawn, while the target distribution represents an empirical data distribution given by a finite set of samples.
\subsection{Probability flow and flow matching}
The goal of a flow-based generative model is to sample from the target distribution $q_1$ by transforming samples from the source distribution $q_0$. Specifically, the model is defined by a vector field $u_t(x): [0,1] \times \mathbb{R}^d \rightarrow \mathbb{R}^d$ which transports particles according to the ordinary differential equation (ODE)
\begin{equation}
\label{eq:ODE}
\begin{split}
    dx =& u_t(x) dt.
\end{split}
\end{equation}
Associated with \eqref{eq:ODE} is the transport map $\phi_t(x_0)$, which maps the initial point $x_0$ to the solution $x_t$ at time $t$. Applying $\phi_t$ to a distribution of particles $p_0$ induces a probability flow where the density at time $t$ is given by the pushforward measure $p_t \triangleq  [\phi_t]_{\#}(p_0)$, where $[g ]_\#$ is defined by the property $\int f(x) \, d([g]_{\#}(p))(x) = \int f\circ g(x) \, dp(x)$ for every integrable function $f$ \cite{figalli2021invitation}. Equivalently, $p_t$ can be characterized as the probability flow arising from the continuity equation $\partial_t p_t + \nabla \cdot (p_t u_t) = 0$ \cite{villani2008optimal}.

In this view, the flow matching problem is to find a vector field $u_t$ that induces a probability flow $p_t$ such that $p_0=q_0$ and $p_1=q_1$. While the exact vector field $u_t$ is often inaccessible, it can be approximated by a neural network $v_t^{\theta}$ and trained using the conditional flow matching objective
\begin{align}\label{eq:CFM:loss}
    \mathcal{L}_{\text{FM}}(\theta) = \EE_{t \in \mathcal{U}[0,1], (x_0,x_1)\sim \pi} \left\|v_t^{\theta} ( (1-t)x_0 + t x_1 ) - (x_1 - x_0)\right\|^2,
\end{align} where the joint distribution $\pi \in \Gamma(q_0,q_1)$ with $\Gamma(q_0,q_1)$ being the set of all joint distributions having marginal distributions $q_0$ and $q_1$ \cite{lipman2022flow}. For example, we can select $\pi(x_0,x_1) = q_0(x_0) \times q_1(x_1)$.

\subsection{Static and dynamic optimal transport}
 Among the many possible transport plans between $p_0$ and $p_1$, the use of optimal transport (OT) allows to find the one that minimizes the total cost of the transportation. This is quantified by the 2-Wasserstein distance, which can be formulated according to Kantorovich or Monge respectively as the following optimization problems:
\begin{align*}
    W_2^2(q_0, q_1) = \min_{\pi \in \Gamma(q_0,q_1)} \int_{\mathbb{R}^d \times \mathbb{R}^d} \left\| x - y \right\|^2 \,\mathrm{d}\pi(x,y) = \min_{T: T_{\#} q_0 = q_1} \int_{\mathbb{R}^d} \lVert x - T(x)\rVert^2 \,q_0(x)\,\mathrm{d}x.
\end{align*} 
As shown in \cite{villani2021topics,figalli2021invitation}, these problems admit unique minimizers $\pi^*$ and $T^*$, which are related by $\pi^* = [\textrm{Id}, T^*]_{\#} q_0$.
Of particular interest to our case is the dynamic OT formulation, which is defined by the optimization problem: 
\begin{align}\label{eq:OT:dynamic}
    W_2^2(q_0, q_1) &= \inf_{(p_t, u_t)}  \left\{ \int_0^1 \int_{\mathbb{R}^d} \|u_t(x)\|^2 \, p_t(x) \, dx \, dt \;\middle|\; \begin{aligned} &\partial_t p_t + \nabla \cdot (p_t u_t) = 0, \\ &p_0 = q_0, \; p_1 = q_1 \end{aligned} \right\} ,
\end{align} which seeks the vector field $u_t^*$ that induces a probability flow $p_t$ that transports the source distribution $p_0=q_0$ to the target distribution $p_1=q_1$ with minimal total kinetic energy. The relation between the static and dynamic OT solutions is simply given as $u_t^*( (1-t)x_0 + t T^*(x_0)) = T^*(x_0) - x_0$. Thus, the vector field $u_t^*$ gives rise to a linear trajectory $x_t = t T^*(x_0) + (1-t) x_0$ for every initial point $x_0$.

\subsection{Optimal flow matching}
There are infinitely many choices of vector fields that solve the flow-matching problem. However, the unique solution $u^*_t$ to the dynamic OT formulation \eqref{eq:OT:dynamic} is associated with particularly efficient inference and fast generation. This is because $u_t^*$ is independent of $t$, so ODE integration along this field simply yields straight-line paths, which lead to lower time-discretization errors and improved computational efficiency  \cite{kornilov2024optimal,liu2022flow}. To approximate $u_t^*$ via \eqref{eq:CFM:loss}, it is necessary to choose $\pi=\pi^*$. However, computing $\pi^*$ has cubic computational complexity in the number of samples, which is challenging for large datasets. A solution is to instead approximate $\pi^*$ using mini-batch data \cite{tong2023improving}, or alternatively to use entropic OT solvers \cite{Pooladian2023Multisample}.


\subsection{Flow matching guidance}
Given a pre-trained flow matching model that transforms the source distribution $q_0$ to the target distribution $q_1$, consider the conditional generation problem where the task is to generate samples that satisfy additional constraints. When the constraints are encoded by an energy function $J$ which attains its minimum when the constraints are satisfied, the likelihood of constraint satisfaction can be expressed in canonical form $\propto e^{-J(\cdot)}$. In this case, the new target distribution becomes $q_1'(x_1) \propto  q_1(x_1) \times e^{-J(x_1)}$. It can be shown that $q_1'$ is the solution of the variational problem 
\begin{align*}
    q_1' = \arg \min_{q} \EE_{x_1 \sim q} [J(x_1)] + {\rm{KL}}(q||q_1),
\end{align*} where $KL(q||q_1)$ denotes the Kullback-Leibler divergence between $q$ and $q_1$ \cite{uehara2024fine}. In this view, conditional generation is a fine-tuning problem: the distribution is shifted to reduce the task-specific loss $J$ while staying close to the original data distribution $q_1$.

%% file: 3_result.tex
\section{Source-Guided Flow Matching}\label{sec:result}

Suppose that we have a pre-trained flow matching model $v_t$ that transports the source distribution $q_0$ to the target distribution $q_1$. 
Consider the conditional generation task in which the new target distribution is of the form $q_1'(x_1)\propto q_1(x_1) \times e^{-J(x_1)}$, where $J$ is a given loss function.
The problem considered in this paper is how samples from $q_1'$ can be generated.
To that end, one could, in principle, modify the source distribution and/or the vector field.
In what follows, we explore how to guide the probability flow to arrive at $q_1'$ by retaining the pre-trained vector field while modifying only the source distribution. 

\subsection{Exact guidance under an exact transportation map}
Consider the ideal case where the pre-trained vector field $v_t$ exactly transports $q_0$ to $q_1$. 
In this case, we derive a closed-form expression for the modified source distribution. We show that this modified source distribution arrives precisely at the desired target distribution under the same vector field.
A formal statement of this result is presented in the following theorem, which is proved in Appendix~\ref{sec:app:pf:thm1}.


\begin{algorithm}[t]
\caption{Source-Guided Flow Matching} \label{alg:algorithm1}
\begin{algorithmic}[1]
    \STATE \textbf{Input}: Source samples $x_0\sim q_0$, target data samples $x_1\sim q_1$, loss function $J$
    \STATE Train the vector field $v_t^{\theta}(\cdot)$ that transforms $q_0$ to $q_1$
    \STATE Sample $x_0 \sim q_0'$
    \STATE Integrate over ODE $\frac{d}{dt} x_t = v_t^{\theta}(x_t)$
    \STATE \textbf{Output}: samples $x_1$
\end{algorithmic}
\end{algorithm}

\begin{theorem}\label{thm:exact}
Let $q_0$ and $q_1$ be the source and target distributions, respectively. Let $v_t\colon\mathbb{R}^d\to\mathbb{R}^d$ be a vector field whose flow map $\phi_t$ satisfies $(\phi_1)_\# q_0 = q_1$. 
For any measurable function $J\colon\mathbb{R}^d\to\mathbb{R}$, define the new target distribution $q_1'(x_1) = \frac{1}{Z_1}   q_1(x_1)\,e^{-J(x_1)}$ and new source distribution $q_0'(x_0)=\frac{1}{Z_0}q_0(x_0)\,e^{-J\circ T(x_0)}$, where $T = \phi_1$, and $Z_0,Z_1$ are normalizing constants. Then, the same flow $\phi_t$ transports $q_0'$ to $q_1'$, i.e., 
$(\phi_1)_\# q_0' = q_1'.$
\end{theorem}

Theorem~\ref{thm:exact} indicates that, if $x_0\sim q_0'$ and $x_t$ evolves as $ d x_t=v_t(x_t) dt$, then $x_1=T(x_0)\sim q_1'$. In other words, exact guidance is achieved. Inspired by this theorem, we propose our SGFM framework, which is presented in Algorithm~\ref{alg:algorithm1}. First, we learn a vector field $v_t^{\theta}$ by minimizing the flow matching loss in \eqref{eq:CFM:loss}. 
Next, we draw samples $x_0 \sim q_0'$, employing the sampling strategies that will be discussed in Section~\ref{sec:sampling}. Finally, each sample $x_0$ is transported along the learned vector field $v_t^{\theta}$ by integrating the associated ODE, yielding guided samples $x_1$.



\paragraph{Modification of the vector field over $\mathbb{R}^d \times [0,1]$ versus modification of the source distribution over $\mathbb{R}^d$:} In conditional generation, existing methods \cite{song2023loss,feng2025guidance} augment the original vector field $v_t$ with a guidance term, which is typically approximated via Monte Carlo sampling. Since generation involves evaluating this augmented vector field at numerous intermediate times $t\in[0,1]$, with each evaluation demanding many samples, the overall process requires extensive sampling. In contrast, our approach leaves the vector field unchanged, and transforms the task into sampling from a modified source distribution at a single time. 

\subsection{Error bounds under approximations in vector field and source distribution}
Learning an exact vector field is inherently difficult, particularly in high-dimensional spaces. In addition, samples may not be drawn exactly from the ideal source distribution $q_0'$, introducing further discrepancies in the generative process. In this section, we analyze how these errors jointly influence the quality of the generated samples. Specifically, we quantify how deviations in both the vector field and the source distribution contribute to the divergence between the target and generated distributions.

To that end, we denote by $v_t(x)$ the exact vector field and $\phi_t(x)$ its flow function, i.e., $\frac{d}{dt}\phi_t(x) = v_t( \phi_t(x))$.
Denote by $v_t^{\theta}(x)$ the learned vector field and $\phi_t^{\theta}(x)$ its corresponding flow function, i.e., $\frac{d}{dt}\phi_t^{\theta}(x) = v_t( \phi_t^{\theta}(x))$.
The formal result on the error bound is presented below, and proven in Appendix~\ref{sec:app:pf:thm2}.
\begin{theorem}\label{thm:error_bound}
Assume that $\left\| v_t(x) - v_t^{\theta}(x) \right\|_{\infty} \leq \epsilon$, and the learned flow $v_t^{\theta}(x)$ is $L_v$-Lipschitz continuous in $x$.
Suppose that the sampling method returns samples of distribution $\tilde{q}_0$.
Then, the generated samples of distribution $[\phi_1^{\theta}]_\# \tilde{q}_0$ satisfy $W_2(q_1',[\phi_1^{\theta}]_\# \tilde{q}_0) \leq e^{L_v} W_2(q_0', \tilde{q}_0) + \epsilon e^{L_v}.$
\end{theorem}
Theorem~\ref{thm:error_bound} provides an upper bound on the 2-Wasserstein distance between the generated and target distributions. The bound makes the contributions of the two error sources explicit. The first term measures the distributional discrepancy introduced by an approximate sampler of the source distribution, and is scaled by $e^{L_v}$, the Lipschitz constant of the flow map $\phi_1^{\theta}$. Intuitively, any deviation in the initial distribution can be amplified by at most a factor of $L_{\phi_1}$ during transport. 
The second term captures the accumulated effect of errors in the learned vector field over time. This contribution stems from integrating a bounded drift perturbation $\epsilon$ through an $L_v$-Lipschitz flow, with the constant $L_v$ controlling how local errors can grow along trajectories. The parameter $L_v$ characterizes the sensitivity of the generative process to errors in both source distribution and vector field. 

Theorem~\ref{thm:error_bound} indicates that our guidance method is particularly effective when the vector field is relatively flat (i.e., when $L_v$ is small), as errors from the source distribution propagate little in such cases. This scenario arises when target samples are preprocessed to be close to the source samples. 
In practice, as long as the source distribution is close to $q_0'$ and the vector field is properly trained, the generated distribution will remain close to the desired target distribution. In the ideal case when the vector field is perfectly learned ($\epsilon = 0$), exact guidance becomes feasible. We can then employ any advanced sampling methods, such as Hamiltonian Monte Carlo, to draw samples from $q_0'$.

\subsection{Improved guidance with the optimal vector field}
Sampling from the modified source distribution $q_0'(x_0)=\frac{1}{Z_0}q_0(x_0)\,e^{-J\circ T(x_0)}$ involves evaluating the flow map $T = \phi_1$ through integration of the vector field $v_t$ over $t\in [0,1]$. To make it efficient, inspired by \cite{tong2023improving}, we instead learn the optimal vector field $v_t^*$ that transforms $q_0$ to $q_1$. In this case, trajectories become straight lines with constant velocity, which greatly reduces the required number of evaluations and thus accelerates sampling. Moreover, according to Theorem 4.1.3 in \cite{figalli2021invitation}, the flow map $\phi_1$ in this case coincides with the optimal Monge map $T^*$.

\begin{figure}[t]
\vspace{-0.4cm}
    \centering
	\subfigure[Optimal vector field.]{
	\includegraphics[scale=0.2]{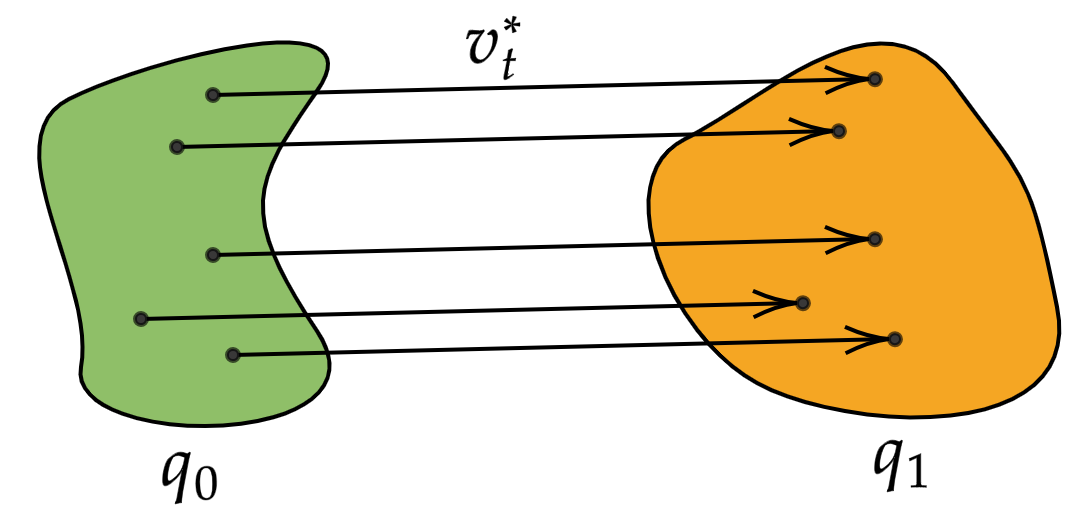}}
	\subfigure[Guidance with optimal vector field.]{
	\includegraphics[scale=0.2]{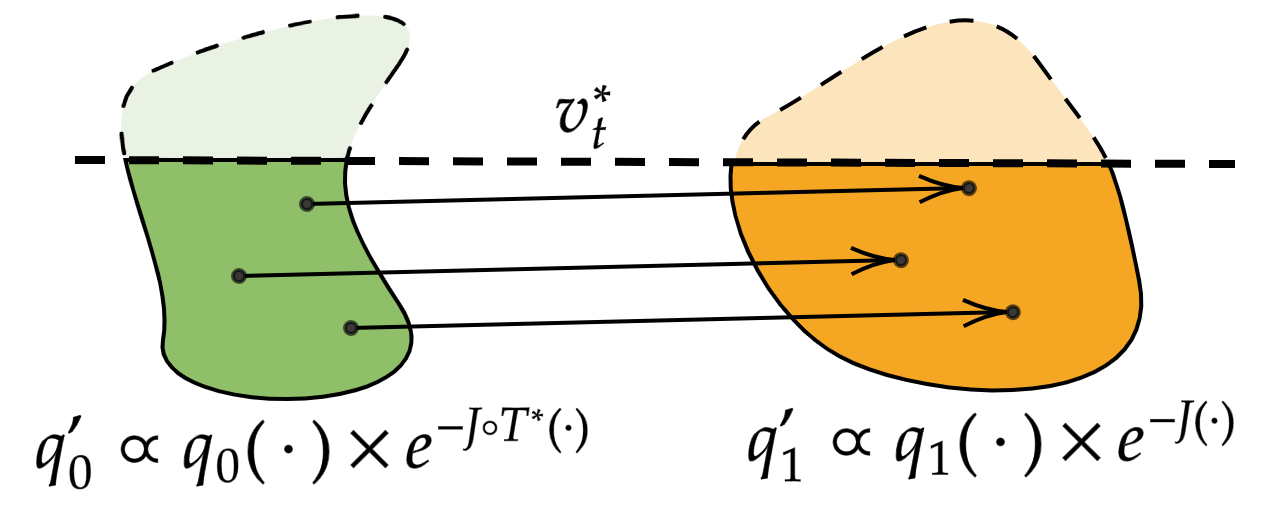}}
	\caption{Illustration of SGFM}
	\label{fig:illus}
\end{figure}

We present an illustration of our guidance method in the case of the optimal vector field.
As shown in Figure~\ref{fig:illus} (a), the optimal vector field $v_t^*$ establishes a point‐wise mapping from each source sample $x_0\sim q_0$ to its corresponding target sample $x_1 \sim q_1$. When we additionally minimize the energy $J$, as illustrated in Figure~\ref{fig:illus} (b), sampling from the new target $q_1'$ reduces to identifying the subset of the source samples that is mapped onto $q_1'$ under the flow $v_t^*$. Theorem~\ref{thm:exact} shows that this subset of source samples has the distribution $q_0'(x_0) \propto q_0(x_0) e^{-J\circ T^*(x_0) }$, where $T^*$ is the flow map at $t=1$.

\begin{figure}[t]
\vspace{-0.5cm}
    \centering
    \subfigure[ $q_1$]{
	\includegraphics[width=0.23\columnwidth]{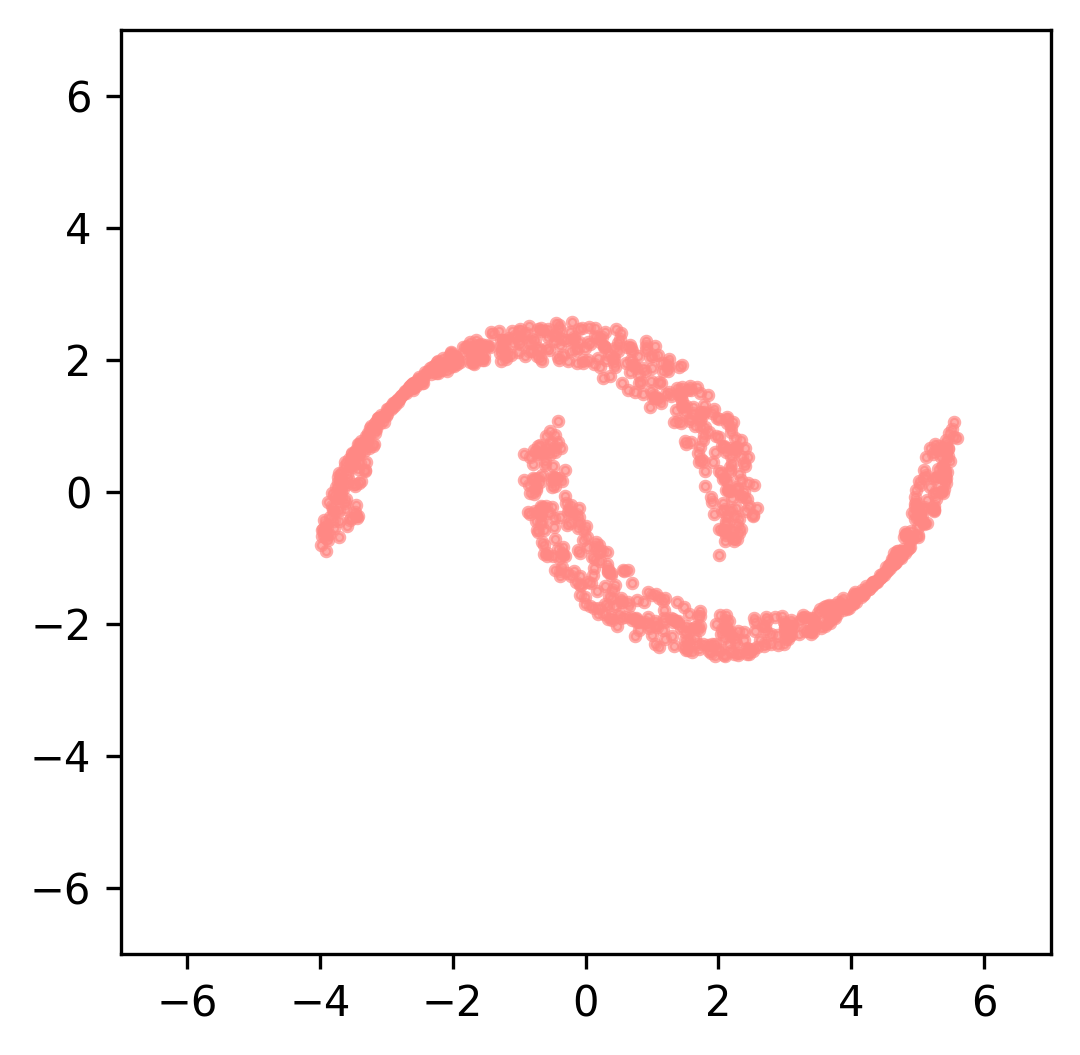}}
	\subfigure[$q_1'$]{
	\includegraphics[width=0.23\columnwidth]{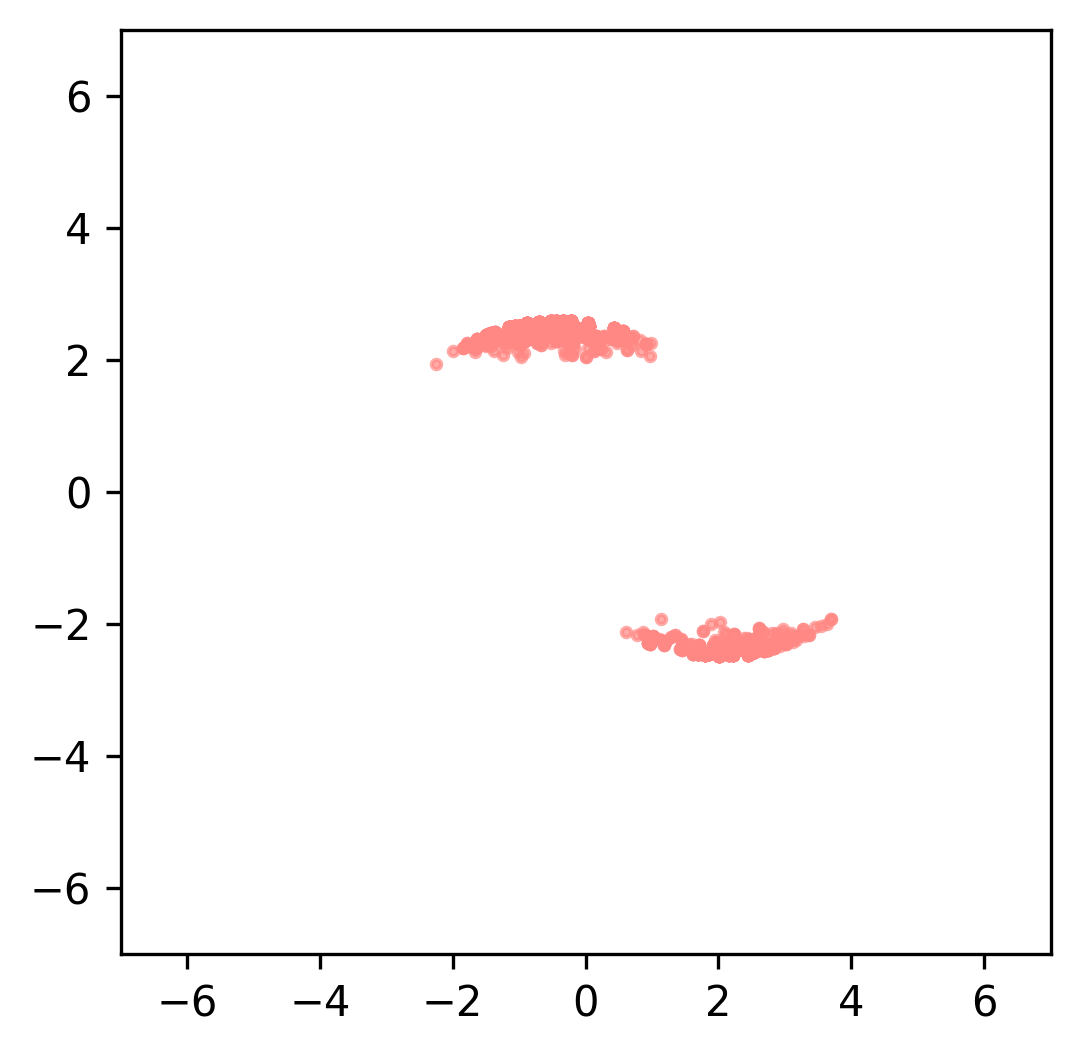}}
    \subfigure[$g^{MC}$ \cite{feng2025guidance}]{
	\includegraphics[width=0.23\columnwidth]{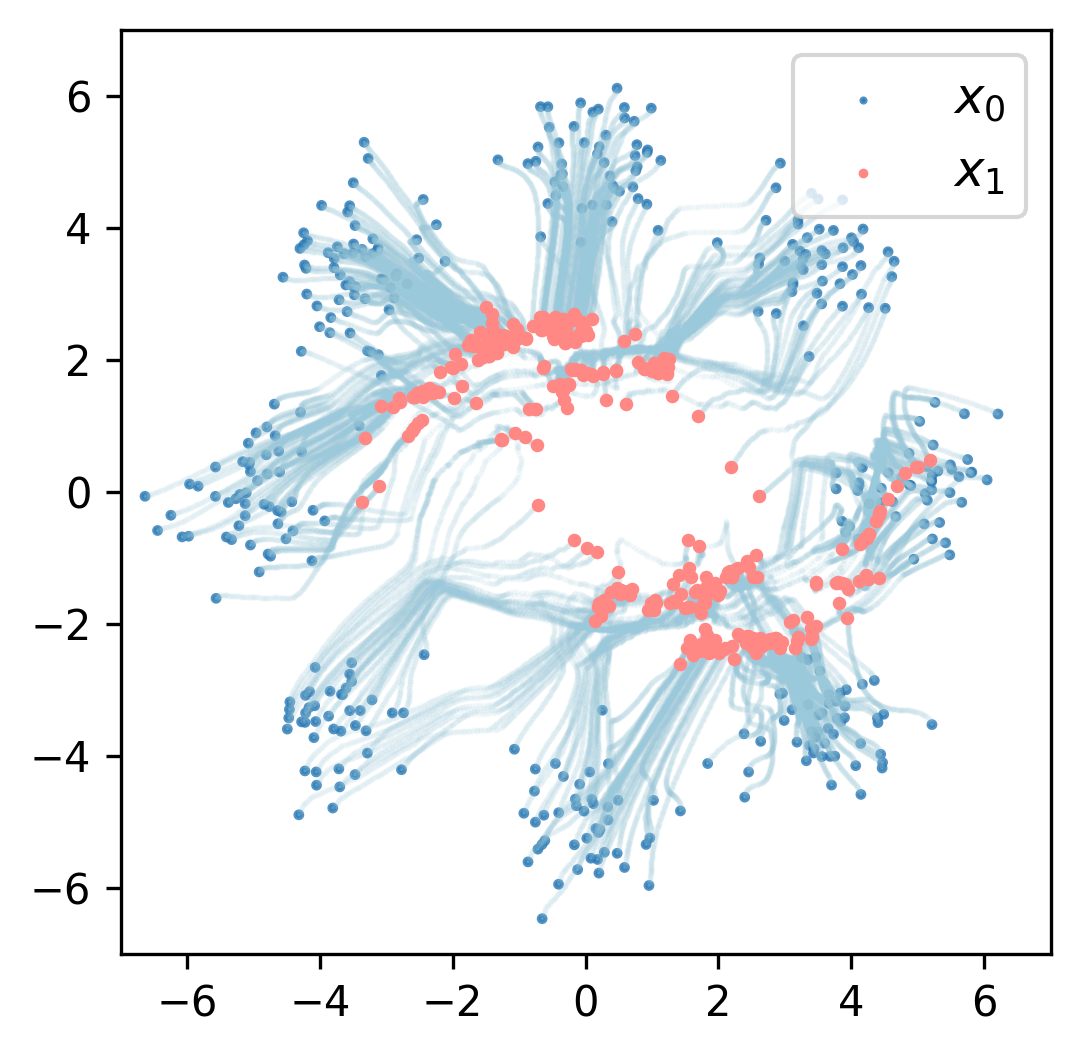}}
	\subfigure[Our method]{
	\includegraphics[width=0.23\columnwidth]{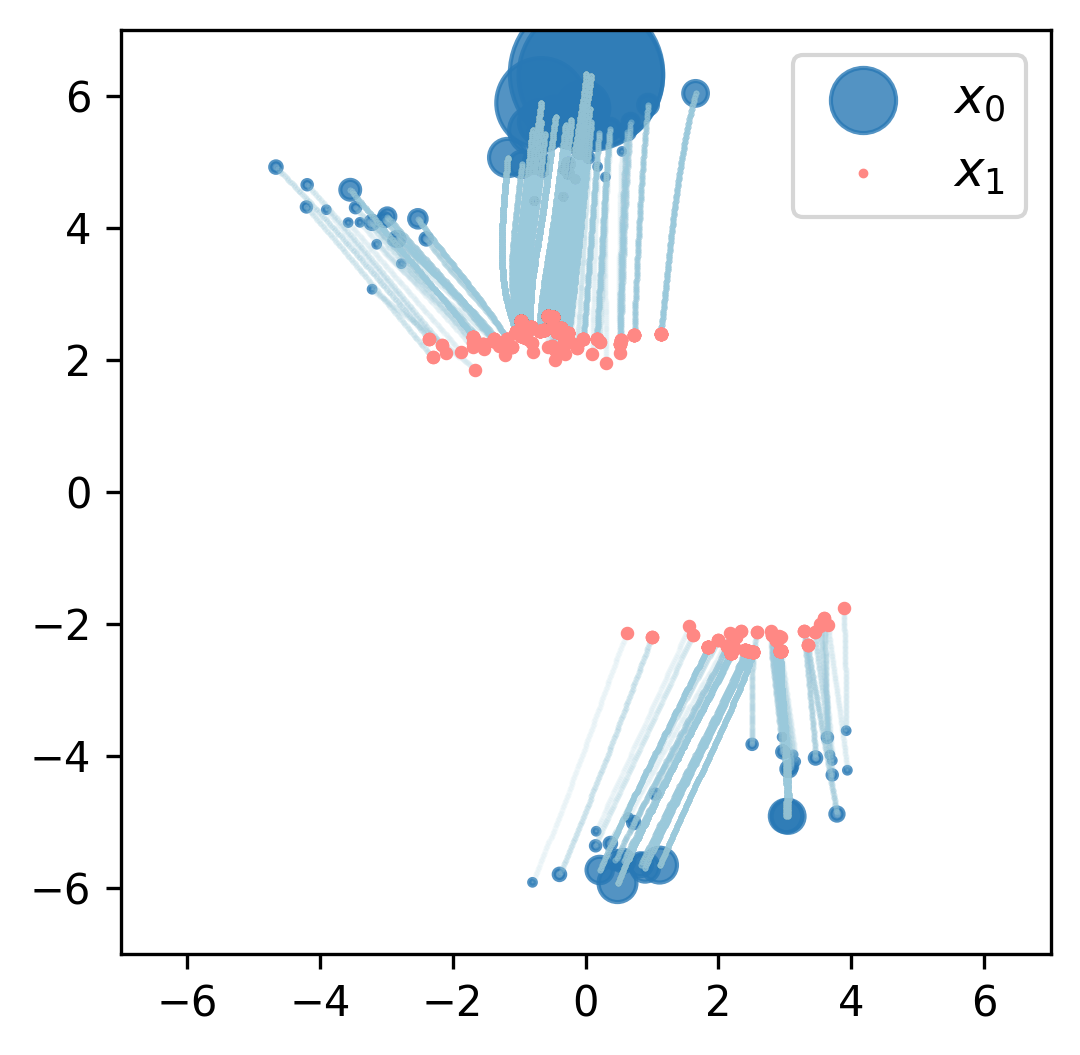}}
	\caption{Illustration of flow guidance in a 2D example. The method $g^{MC}$ \cite{feng2025guidance} modifies the vector field, resulting in curved trajectories. In contrast, our proposed method maintains the optimal vector field and modifies the source distribution, leading to straight trajectories. In (d), each ball's radius indicates the relative weight of the corresponding source sample.}
	\label{fig:2D:comp_guidance}
\end{figure}

Beyond speeding up sampling from the source distribution, straight trajectories from the optimal vector field can also accelerate the inference. Since our guidance method only modifies the source distribution, it inherits the straight-line behavior of the optimal vector field, as shown in Figure~\ref{fig:2D:comp_guidance}(d). This usually leads to faster inference and improved stability as the integration of ODE requires fewer discretization steps.
In contrast, the guidance method proposed in \cite{feng2025guidance} maintains the original source distribution and modifies the vector field with an additional guidance term. As a result, this approach produces curved sampling trajectories, as shown in Figure~\ref{fig:2D:comp_guidance} (c), even when the pre-trained vector field would produce straight paths. Such curvature requires more discretization steps and a higher computational cost to maintain integration accuracy.

%% file: 4_sampling.tex
\section{Sampling from the Modified Source Distribution}\label{sec:sampling}

Given the pre-trained optimal vector field $v_t^*$, the guidance problem is reduced to drawing samples from the modified source distribution. Thus, the key to effective implementation of our method is accurate and efficient sampling from $q_0'(x_0)\propto~q_0(x_0) e^{-J\circ T^*(x_0)}$. The choice of the sampling method depends on the properties of the cost function $J$ and the dimensionality of the sample space. Whenever the sampling method generates a sequence of approximate distributions $(\tilde{q}^k)_{k\geq 0}$ such that $W^2_2(\tilde{q}^k, q_0')\rightarrow 0$ as $k\rightarrow\infty$, our method of guided flow matching is asymptotically exact, as follows from Theorem \ref{thm:error_bound}. In this section, we discuss asymptotically exact samplers, optimization-based samplers, including their connection to D-Flow \cite{ben2024d}, and other methods.


\subsection{Asymptotically exact sampling methods}\label{sec:sampling:exact}
\subsubsection*{Importance sampling}
In low-dimensional spaces, importance sampling (IS) \cite{Chopin2020} offers a fast and gradient-free sampling method. 
Given an unnormalized target distribution $q$, an initial set of particles is generated using a proposal distribution $m$ such that $\text{supp}(m)\supset\text{supp}(q)$. Samples are then drawn from this set according to weights determined by their relative probability in the target versus proposal distribution $W^n=\frac{w(X^n)}{\sum\limits_{m} w(X^m)}$, where $w(x)\propto\frac{q(x)}{m(x)}$. 
The approximate distribution $\tilde{q}^N(x)\triangleq\sum_{n=1}^N W^n\delta_{X^n}(x),$ with $X^n\sim m,$ converges weakly to the target distribution $q$ when $N\rightarrow\infty$ \cite{Chopin2020}. 
Assuming $q$ is defined on a closed and bounded subset of $\mathbb{R}^d$, this implies that $W^2_2(\tilde{q}^N, q)\rightarrow 0$  as $N\rightarrow\infty$ \cite[Theorem 6.9]{Villani2009}. When $q=q_0'$ and $m=q_0$, we have $w(x_0)=e^{-J\circ T^*(x_0)}$. Note that this method does not require $J$ to be differentiable. For a detailed outline of the algorithm, see Appendix \ref{sec:app:is}.
\subsubsection*{Hamiltonian Monte Carlo}
Importance sampling suffers from the curse of dimensionality, making Hamiltonian Monte Carlo (HMC) \cite{neal2011mcmc} a popular alternative in high-dimensional sample spaces. HMC is a Markov chain Monte Carlo method for unnormalized, continuous densities in Euclidean space where partial derivatives of the log density exist. In particular, HMC generates proposal samples by propagating the target variable, representing position in space, and an auxiliary momentum variable using Hamiltonian dynamics, achieving extensive exploration while maintaining high acceptable probabilities. For more details, see Appendix \ref{sec:app:hmc}. 

HMC typically returns an ergodic Markov chain, which means that it converges asymptotically to the target distribution $q$ \cite{neal2011mcmc}. 
When the negative log-likelihood $-\ln{q}$ is twice differentiable, 
strongly convex and has Lipschitz-continuous gradients, and the integration of the dynamics is sufficiently accurate, the law of the Markov chain after $N$ steps $\tilde{q}^N$ approximates the target distribution $q$ up to arbitrarily small Wasserstein precision $W^2_2(\tilde{q}^N, q)$ for a sufficiently large $N$ \cite[Theorem 5]{chen2019optimal}.

In our case, we have the negative log likelihood $-\ln{q_0'(x_0)}=-\ln{q_0(x_0)} + J\circ T^*(x_0)$, where $T^*$ generally prevents the convexity requirement from being globally satisfied. However, since $T^*$ is learned to encourage straight-line transport, we might expect that the composition $J \circ T^*$ approximately preserves the convexity of $J$ locally. To escape local modes, various HMC strategies exist such as tempering \cite{neal2011mcmc}. 



\subsection{Optimization-based sampling}
Another approach to obtaining samples from the modified source distribution $q_0'$ is to search for the modes by optimizing the negative log-likelihood:
\begin{align}\label{eq:opt_problem:our}
    \min_{x_0} - \ln q_0'(x_0) \quad \Leftrightarrow  \quad \min_{x_0} - \ln q_0(x_0) + J \circ T^* (x_0).
\end{align}
In this formulation, the term $J\circ T^*(x_0)$ introduces task-specific loss via the OT map $T^*$, while the term $-\ln q_0(x_0)$ acts as a regularizer. This regularizer, however, attracts $x_0$ toward the most probable fixed points, rather than toward regions of high probability. For example, in the case of a standard Gaussian source distribution, the optimization problem \eqref{eq:opt_problem:our} becomes 
\begin{align}\label{eq:opt_problem1:our}
\min_{x_0} \frac{\left\| x_0\right\|^2}{2} + c + J \circ T^* (x_0),
\end{align} in which the regularizer $- \ln q_0(x_0)= \frac{\left\| x_0\right\|^2}{2} + c$ would guide the sample toward the unique mode at $x_0=0$ and lead to mode collapse.

To mitigate this issue, the regularizer can be replaced by an alternative that better promotes sample diversity. For a Gaussian source distribution $x_0\sim q_0=\mathcal{N}(0, I_d)$, we have $\|x_0\|^2\sim\chi^2$, where $\chi^2_d$ is the chi-square distribution with $d$ degrees of freedom. Instead of regularizing with the probability of the sample, one might instead use the probability of the norm of the sample, $-\ln p_{\chi^2_d}(\left\|x_0 \right\|^2)$. This means that the unique prior mode at $x_0=0$ is replaced by the sphere $\{x_0: \|x_0\|^2=\argmax_{x}p_{\chi^2_d}(x)=\max(d-2, 0)\}$. The resulting optimization problem is
\begin{align}\label{eq:opt_problem2:our}
    \min_{x_0}  - \ln p_{\mathcal{X}^2}(\|x_0\|^2) + J \circ T^* (x_0) \Leftrightarrow  \min_{x_0} -(d-2) \log\left\| x_0\right\| + \frac{\left\|x_0 \right\|^2}{2} + J \circ T^* (x_0),
\end{align} which coincides with heuristic formulations in \cite{ben2024d} (see discussion below).  

An extension of this idea is to explicitly target high-density regions of the prior. Since $\mathbb{E}[\|x_0\|^2]=d$, $\text{Var}[\|x_0\|^2]=2d$, and $p_{\chi^2}$ is unimodal, we observe that the prior density concentrates in hyperspherical shell $|\|x_0\|^2-d|\leq\sqrt{2d}$. Motivated by this, we propose a new method within the optimization-based sampling family for Gaussian source distributions given by
\begin{align}\label{eq:opt_problem3:our}
&\min_{x_0} J \circ T(x_0) \quad \text{s.t. } |\|x_0\|^2-d|\leq \sqrt{2d}.
\end{align}
In practice, the constrained problem can be addressed either by incorporating a regularizer of the form $( \left\|x_0\right\|^2 - d)^2$, or by applying projected gradient descent onto the feasible set defined by $|\|x_0\|^2-d|\leq \sqrt{2d}$.

The implementation details can be found in Appendix \ref{sec:app:optsamp}. 
From Theorem~\ref{thm:exact}, it follows that ensuring that $x_0$ lies in high-density regions of $q_0'$ implies that the corresponding sample $x_1$ will also lie in high-density regions of $q_1'$, which justifies \eqref{eq:opt_problem:our} - \eqref{eq:opt_problem3:our} as sampling methods in our framework. In practice, we can design specific optimization objectives to sample from $q_0'$ depending on the problem structure and the source distribution. 

\paragraph{Relation to D-Flow\cite{ben2024d}:}  The authors of
\cite{ben2024d} heuristically reformulate sampling as an optimization problem with various choices of regularization. In fact, their formulations coincide with \eqref{eq:opt_problem1:our} - \eqref{eq:opt_problem2:our}.
Since our formulation guarantees fidelity to the ground truth target distribution through Theorem~\ref{thm:exact}, the SGFM framework explicitly clarifies the role of the regularization term and thereby provides foundational support for D-Flow. Thus, D-Flow can be regarded as a special case of SGFM.

The optimization-based sampling method is typically suitable when the target distribution resembles a Dirac distribution, or when we are interested in obtaining a high-probability sample rather than a representative sample of the distribution. However, it risks leading to mode collapse as discussed below.

\paragraph{Example of mode collapse in optimization-based sampling:}
We illustrate the limited expressiveness of optimization-based sampling. Consider a set of particles at locations $(x^1,x^2)$ uniformly distributed over an 'S'-shaped structure.
To encourage $x_1$ and $x_2$ to stay close, we introduce a soft penalty $J=\|x^1-x^2\|$ to guide the generation. 
\begin{wrapfigure}{r}{0.3\columnwidth}
\vspace{-0.8cm}
\begin{center}
\centerline{\includegraphics[width=0.35\columnwidth]{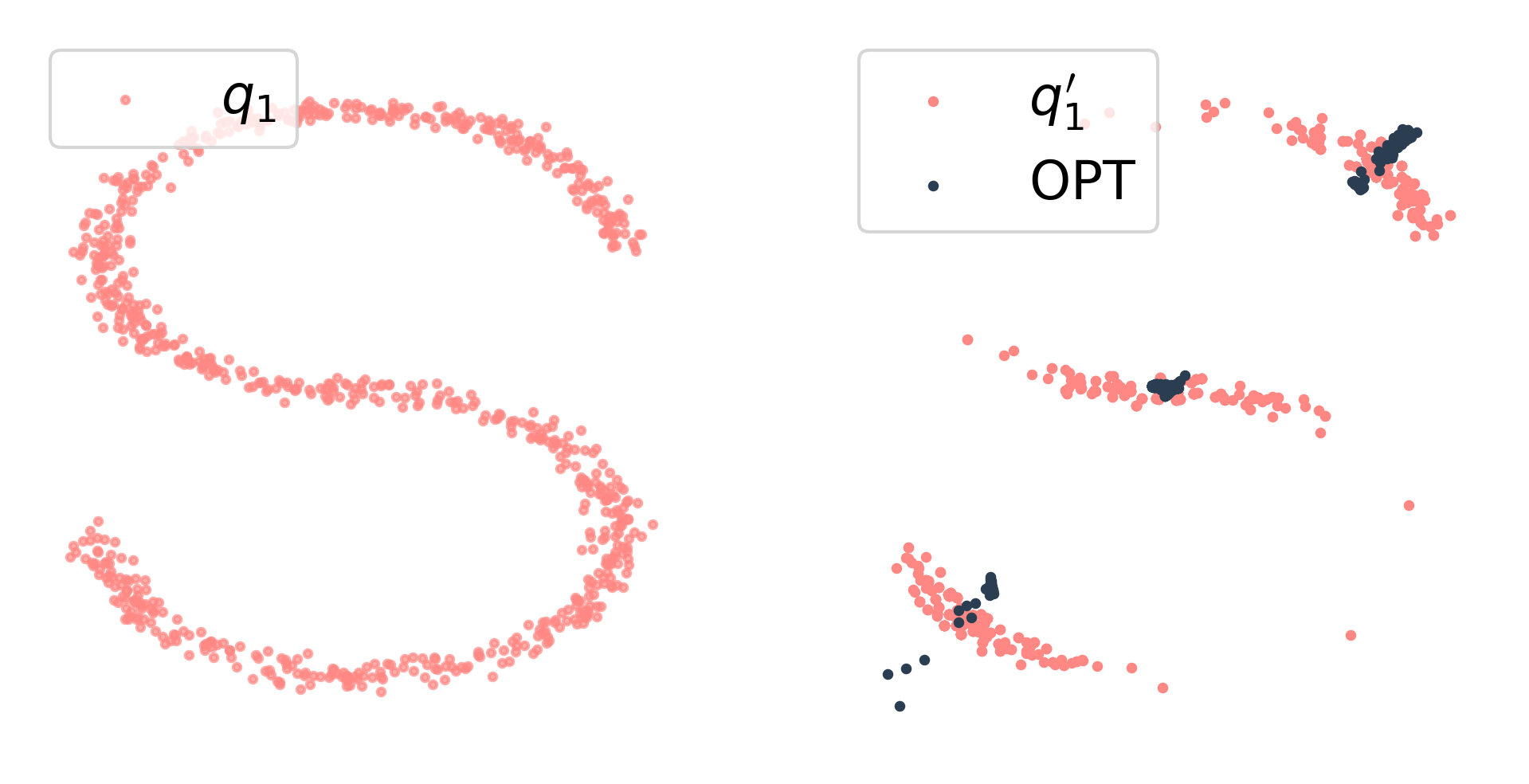}}
\caption{Mode collapse in optimization-based sampling}
\label{fig:2D:illustrate:diversity}
\end{center}
\end{wrapfigure}
As shown in Figure~\ref{fig:2D:illustrate:diversity}, applying \eqref{eq:opt_problem:our}, or equivalently D-Flow, leads to an excessive concentration of the particles around the line $x_1 = x_2$. Therefore, optimization-based sampling fails to capture the inherent diversity of the true conditional distribution.

\subsection{Other sampling methods}
We have presented a selection of methods to sample from $q_0'$. A key strength of SGFM is that it reduces guidance to a well-defined sampling problem, enabling users to flexibly choose the most suitable method for their specific problem. For example, when gradient information is available and a minor bias is acceptable, the Unadjusted Langevin Algorithm (ULA) \cite{robert1999monte} offers an efficient option. If asymptotic exactness is required and the computational budget is larger, a better alternative is Metropolis–Adjusted Langevin Algorithm (MALA) \cite{robert1999monte}, which corrects ULA’s bias. For extensive exploration of complex, high-dimensional distributions, HMC would be preferred.

%% file: 5_experiment.tex
\section{Experiments}\label{sec:experiments}
In this section, we evaluate the performance of SGFM on a toy 2D example, image generation, and physics-informed generative modeling. We benchmark our method against its closest counterparts, D-Flow \cite{ben2024d} and the top-performing methods in \cite{feng2025guidance}.

\subsection{Toy 2D example}\label{sec:experiment:2D}
We begin the evaluation on low-dimensional synthetic datasets. Specifically, we select a uniform source distribution and an 8-Gaussian target distribution. 
\begin{wrapfigure}{r}{0.4\columnwidth}
\vspace{-2cm}
\begin{center}
\centerline{\includegraphics[width=0.4\columnwidth]{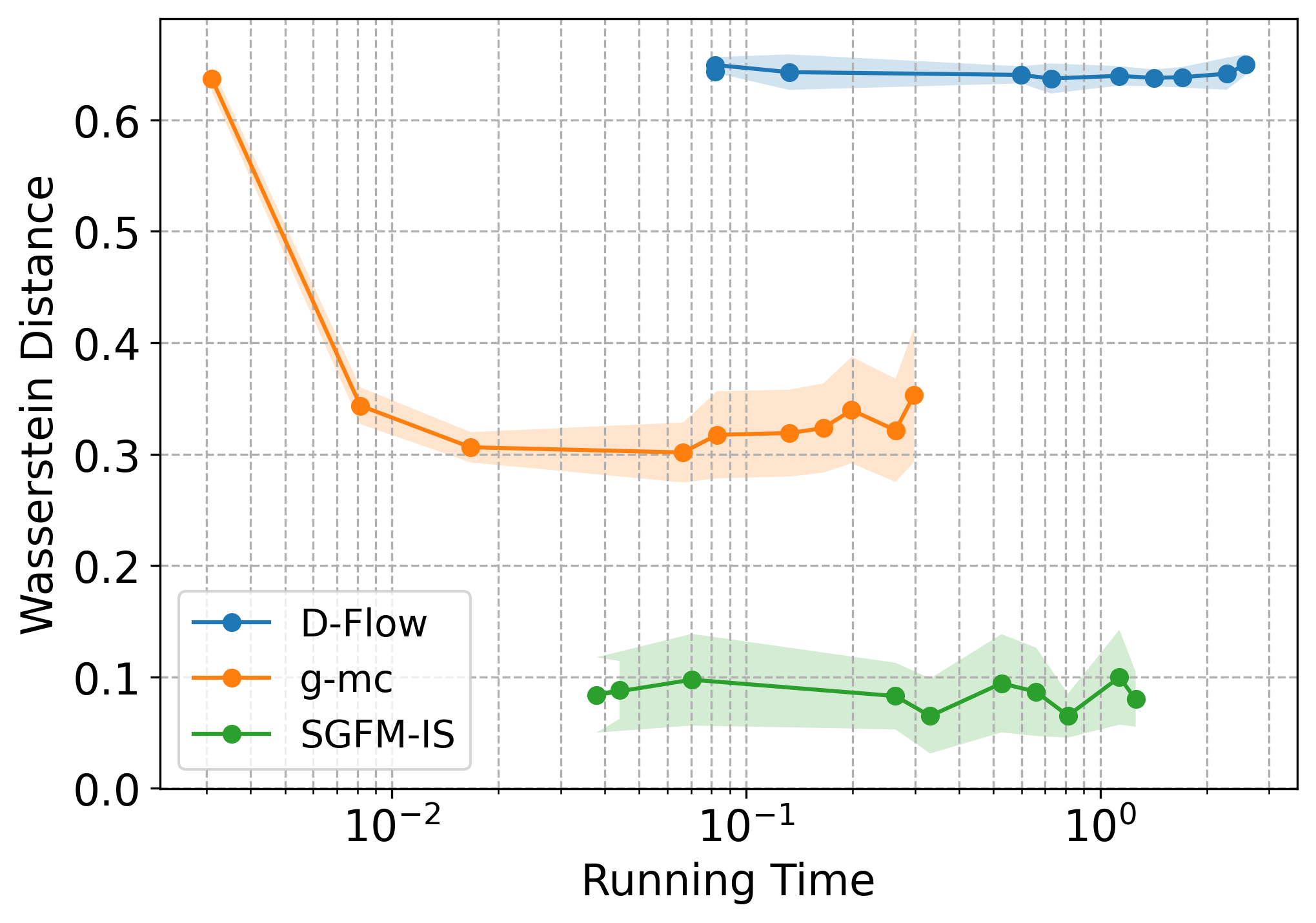}}
\caption{Comparison of sample quality and running time in 2D example.}
\label{fig:2D:comp_uniform}
\end{center}
\end{wrapfigure}
Since the source distribution is non-Gaussian, diffusion-based guidance methods cannot be applied. In this low-dimensional setting, we adopt importance sampling and refer to our method as SGFM-IS. We evaluate the sample quality and inference time by varying the number of function evaluations (NFEs). The sample quality is measured using the Wasserstein distance between the true guided distribution and the generated distribution.
\begin{wrapfigure}{r}{0.4\columnwidth}
\vspace{-0.7cm}
\begin{center}
\centerline{\includegraphics[width=0.4\columnwidth]{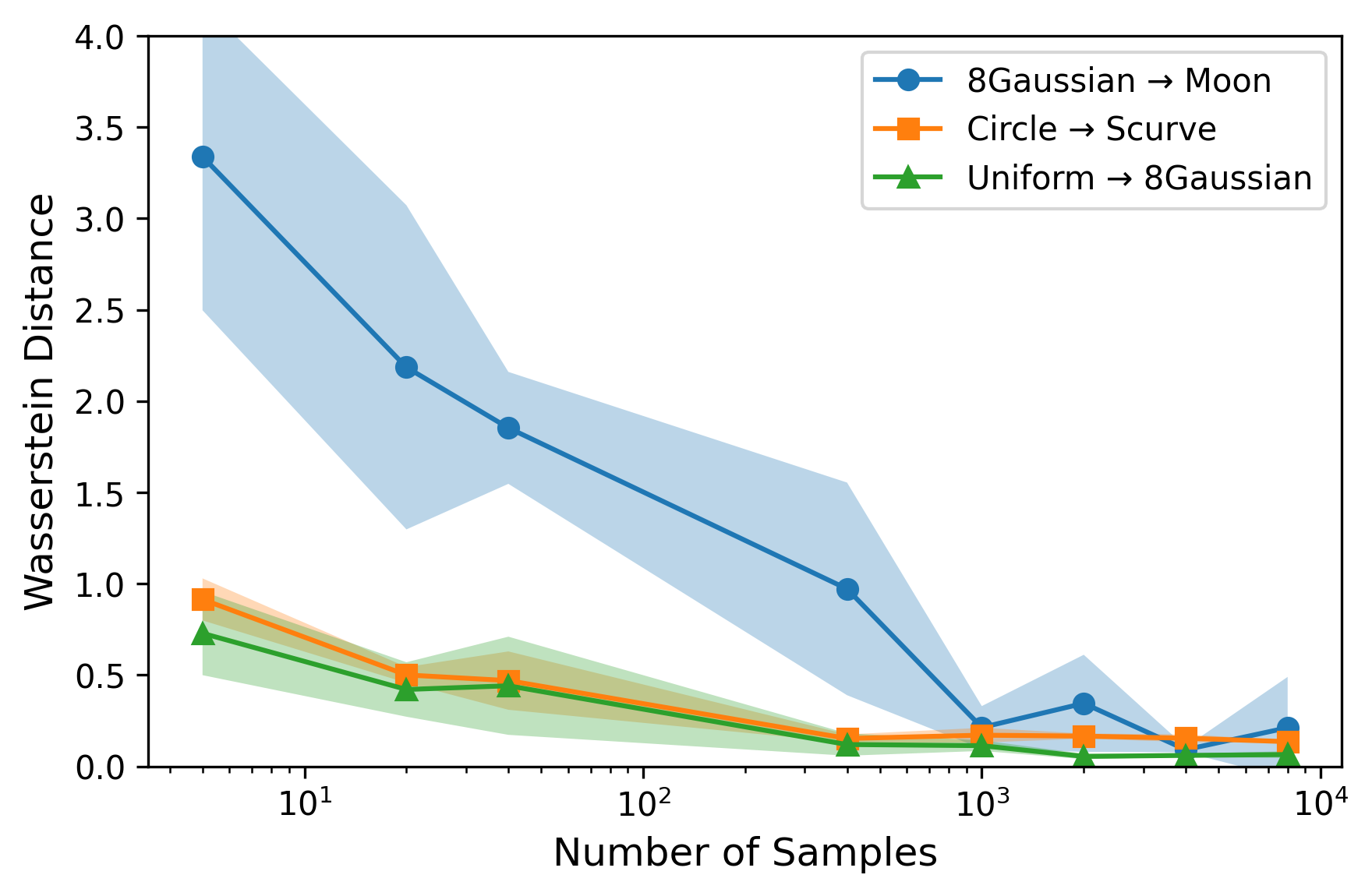}}
\caption{Asymptotic exactness of our guidance method. (Source $\rightarrow$ target) }
\label{fig:2D:illustration:exact}
\end{center}
\vspace{-0.6cm}
\end{wrapfigure}
As shown in Figure~\ref{fig:2D:comp_uniform}, our method consistently achieves lower Wasserstein distances. Moreover, reducing NFEs, which lowers runtime, has only a small effect on sample quality. This observation aligns with the prior findings that the optimal vector field produces straight trajectories and requires fewer evaluations.

Next, we evaluate the exactness of SGFM. Specifically, we investigate how the guidance precision evolves as the number of IS samples increases. The guidance precision is measured by the Wasserstein distance between the generated distribution and the ground-truth distribution, where the latter is approximated using up to $10^4$ samples. As shown in Figure~\ref{fig:2D:illustration:exact}, for each pair of source and target distributions, the Wasserstein error consistently decreases as the number of IS samples increases for every pair of source and target distributions. Therefore, in low-dimensional tasks, our guidance method achieves asymptotic exactness with an increasing number of samples.

\subsection{PDE solution operator}\label{sec:experiments:darcy}
We consider a high-dimensional inverse problem based on the Darcy flow equations \cite{bastek2024physics, jacobsen2025cocogen}. Darcy flow is an elliptic PDE describing fluid flow though a porous medium with permeability field $K$ and pressure field $p$. The flow matching model is trained to sample pairs of $K$ and $p$ occurring as discretized solutions on a square domain with resolution $64\times 64$. The dataset \cite{bastek2024physics} is obtained by solving the PDE using finite differences. For more details, see Appendix \ref{sec:app:darcyeq}-\ref{sec:app:darcyimp}.

The conditional sampling problem is to generate permeability fields that are consistent with a partially observed pressure field. As this problem has many solutions, we let the family of valid solutions be the target distribution. The validity of an inverse estimate $\hat{K}$ is measured by $J(p_{\hat{K}})$, where $J$ computes the target reconstruction error and $p_{\hat{K}}$ is the true pressure field corresponding to $\hat{K}$. Since $p_{\hat{K}}$ is inaccessible, the sampling guidance cost is $J\big(\hat{p})$ where $\hat{p}$ is the pressure field sampled jointly with $\hat{K}$. We evaluate SGFM-HMC, SGFM-OPT \eqref{eq:opt_problem1:our} and SGFM-OPT $\chi^2$ \eqref{eq:opt_problem2:our}, where the latter corresponds to D-Flow with the preferred regularization option, and benchmark these methods against $g^{\text{cov-A}}$.

Figure~\ref{fig:darcy} shows the target pressure and the true pressure solution corresponding to a single outcome of the permeability field for each method. SGFM-OPT $\chi^2$ obtains the best target reconstruction, closely followed by SGFM-HMC. In comparison, SGFM-OPT and particularly $g^{\text{cov-A}}$ suffer from large biases. Additional samples are shown in Appendix \ref{sec:app:darcy_multioutput}. To further analyse the performance, the validity $J(p_{\hat{K}})$, guidance cost $J(\hat{p})$, and physical consistency $\|p_{\hat{K}} - \hat{p}\|$ of 25 samples are shown in Table~\ref{tab:darcy_agg}, reported as the median and interquartile range. SGFM-OPT $\chi^2$ achieves the best validity, followed by SGFM-HMC. In comparison, SGFM-OPT and $g^{\text{cov-A}}$ do not perform significantly better than unconditional sampling. Although $g^{\text{cov-A}}$ achieves the lowest guidance cost, it compromises physical consistency, leading to poor validity. Similarly, SGFM-OPT has worse validity than SGFM-HMC despite achieving lower guidance cost. In contrast, both SGFM-HMC and SGFM-OPT $\chi^2$ maintain physical consistency while SGFM-OPT $\chi^2$ achieves lower guidance cost resulting in the best validity. 

Due to the high-dimensional and complex source distribution, the SGFM methods require longer runtimes compared to $g^{\text{cov-A}}$, which limits the number of samples that can be evaluated. While this constrains our ability to assess how well the methods capture the whole family of solutions, we show that SGFM-HMC performs best in this regard for an example in lower dimension in Appendix \ref{sec:app:ode}. 
\begin{figure}
    \centering
    \includegraphics[width=\linewidth]{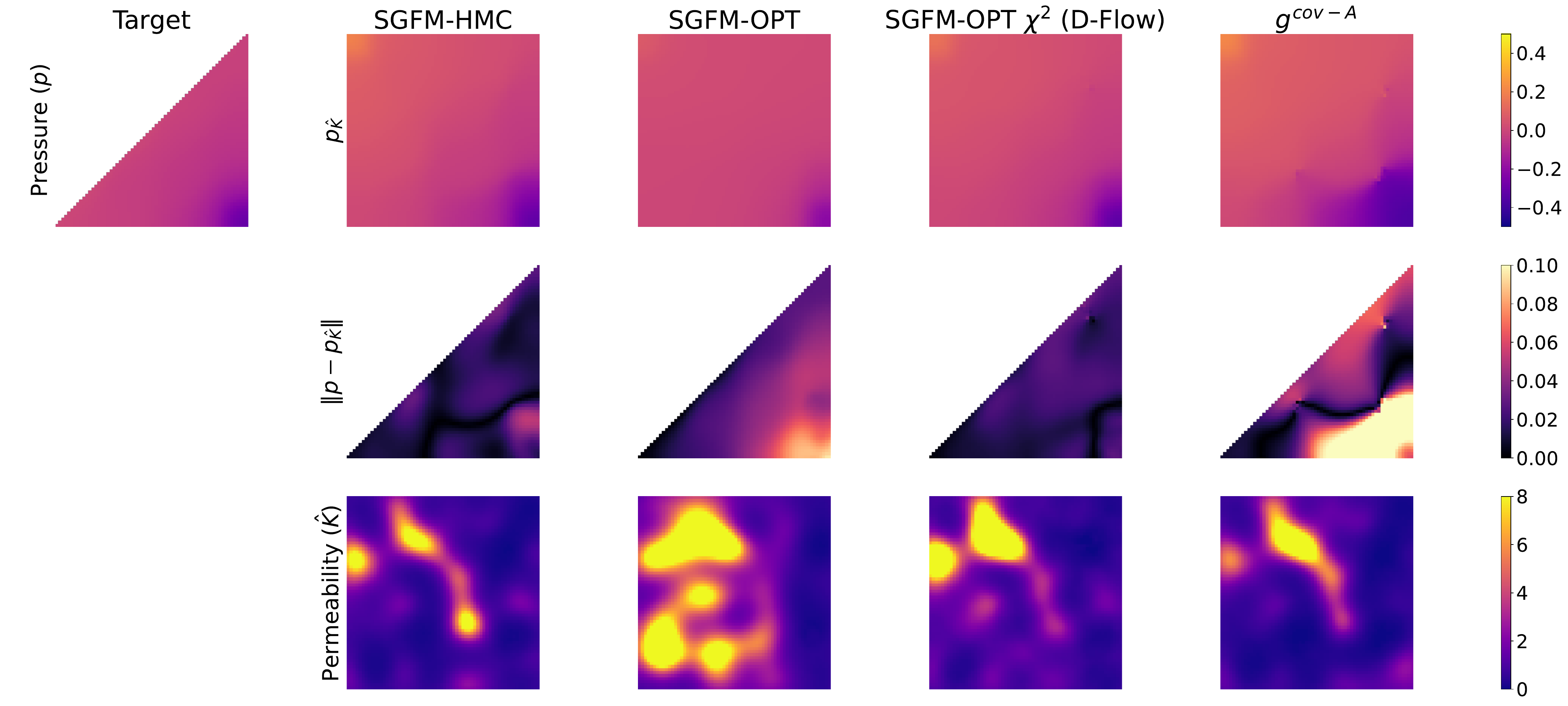}
    \caption{Solutions to the inverse problem of the Darcy flow equations. Top: target pressure and true solution $p_{\hat{K}}$ corresponding to inverse estimate $\hat{K}$; middle: target reconstruction error; bottom: inverse estimate $\hat{K}$ of the permeability field generated by conditional sampling.}
    \label{fig:darcy}
\end{figure}
\begin{table}[ht]
\centering
\caption{Performance of guidance methods in the Darcy flow inverse problem}
\resizebox{\textwidth}{!}{%
\begin{tabular}{lccc}
\toprule
\textbf{Method} & \textbf{Validity of Inverse Estimate} & \textbf{Guidance Cost} & \textbf{Physical Consistency} \\
\midrule
SGFM-HMC                   & 0.591  [0.532, 0.654] & 0.277  [0.239, 0.335] & \textbf{0.189}  [0.169, 0.229] \\
SGFM-OPT                   & 0.907  [0.503, 1.875] & 0.212  [0.164, 0.297] & 0.430  [0.177, 0.771] \\
SGFM-OPT $\chi^2$ (D-Flow) & \textbf{0.474}  [0.416, 0.562] & 0.177  [0.121, 0.203] & 0.191  [0.157, 0.213] \\
$g^{\text{cov-A}}$          & 0.993  [0.857, 1.304] & \textbf{0.033}  [0.030, 0.054] & 0.289  [0.247, 0.351] \\
\midrule
\textit{Unconditional sampling} & 1.006  [0.860, 1.269] & 1.051  [0.905, 1.289] & 0.214  [0.170, 0.274] \\
\bottomrule
\end{tabular}}
\label{tab:darcy_agg}
\end{table}

\subsection{Imaging inverse problem on CelebA}
We conduct experiments on various imaging inverse problems using the CelebA dataset, evaluating five distinct tasks: denoising, deblurring, super-resolution, random inpainting, and box inpainting. Since the target distribution for these inverse problems is typically Dirac or highly concentrated, we apply optimization-based sampling within our SGFM framework. We label the SGFM-OPT variants by 1-6, where SGFM-1 corresponds to \eqref{eq:opt_problem1:our}, SGFM-2 corresponds to \eqref{eq:opt_problem2:our} (\cite{ben2024d}), and SGFM 3-5, SGFM-6 corresponds to \eqref{eq:opt_problem3:our} implemented with different regularizers and projected gradient descent, respectively. Our methods are benchmarked against strong baselines, including the top-performing g-covA and g-covG from \cite{feng2025guidance}, and the PnP method from \cite{martin2024pnp}. For details of the SGFM variants, implementation, and visualizations of the generated images, see Appendix~\ref{sec:app:celeba}.

The experimental results in Table~\ref{tab:celeba} demonstrate that SGFM-OPT variants achieve state-of-the-art performance. Specifically, they outperform $g-covA$ and $g-covG$ across all tasks. Our method is competitive with PnP-flow in most tasks and ranks one class below in deblurring; however, we note that PnP is specifically designed for imaging inverse problems, while our method is more general than that. The results confirm that using the SGFM framework with optimization-based samplers is a highly effective and flexible strategy for imaging inverse problems.


\begin{table}[h]
\centering
\caption{PSNR ($\uparrow$) comparison of methods for inverse problems on CelebA.}
\begin{tabular}{lccccc}
\hline
Method & Denoising & Deblurring & Super-resolution & Rand inpaint & Box inpainting \\
\hline
g-covA            & 26.73 & 29.72 & 18.45 & 19.61 & 24.88 \\
g-covG            & 30.35 & 29.50 & 24.18 & 25.49 & 26.12 \\
PnP               & \textbf{32.14} & \textbf{38.74} & 31.33 & \textbf{33.87} & \textbf{29.92} \\
\hline
SGFM-OPT-1          & 28.51 & 35.12 & 33.30 & 34.02 & 28.51 \\
SGFM-OPT-2  & 28.95 & 35.23 & 33.32 & 34.01 & 28.43 \\
SGFM-OPT-3        & 31.51 & 35.21 & 33.28 & 34.05 & 30.09 \\
SGFM-OPT-4        & \textbf{31.60} & \textbf{35.27} & 33.31 & 34.03 & \textbf{30.12} \\
SGFM-OPT-5        & 28.94 & 35.22 & \textbf{33.33} & \textbf{34.06} & 28.55 \\
SGFM-OPT-6        & 31.54 & 32.60 & 32.10 & 32.36 & 29.19 \\
\hline
\end{tabular}\label{tab:celeba}
\end{table}


%% file: 6_conclusion.tex
\section{Conclusion}
We presented a framework for guided flow matching with theoretical guarantees. The framework reduces the guidance problem to a problem of sampling from a modified source distribution.
Examples on 2D benchmarks, physics-informed generative tasks, and imaging inverse problems demonstrated the effectiveness and flexibility of the framework.
We acknowledge that sampling from the source distribution may present its own challenges, especially for complex, high-dimensional distributions.
Nevertheless, the proposed method offers users the flexibility to select a sampling strategy that balances their desired trade-offs between accuracy and computational cost.

%% file: Appendix.tex
\appendix
\section{Proofs}
\subsection{Proof of Theorem \ref{thm:exact}}\label{sec:app:pf:thm1}
Recall that $T=\phi_1$. Since the transport map $T$ pushes $q_0$ to $q_1$, according to \cite{Villani2009} we have
\begin{align}\label{eq:pf:thm1:1}
    q_1(x_1) = T_{\#} q_0 (x_1) = \frac{q_0(T^{-1}( x_1))} { |\mathrm{det} \nabla T^{-1}(x_1)|}.
\end{align}

The resulting pushforward distribution of $q_0'$  under the transport map $T$ is then
\begin{align}
    T_{\#} q_0' (x_1) = \frac{q_0'(T^{-1}( x_1))} { |\mathrm{det} \nabla T^{-1}(x_1)|} = \frac{q_0(T^{-1}( x_1)) e^{-J(x_1)}}{ Z |\mathrm{det} \nabla T^{-1}(x_1) |} = \frac{q_1(x_1) e^{-J(x_1)}}{Z} =q_1'(x_1),
\end{align}
where the second equality follows from the definition of $q_0'$ and third equality follows from \eqref{eq:pf:thm1:1}. The proof is complete.

\subsection{Proof of Theorem \ref{thm:error_bound}}\label{sec:app:pf:thm2}
Before the proof of Theorem~\ref{thm:error_bound}, we give a useful lemma.
\begin{lemma}\label{lemma:1}
Suppose that $f(x)$ is $L$-Lipschitz continuous in $x$. Then, we have
\begin{align}
    W_2(f_\# \mu, f_\# \nu) \leq L W_2 (\mu,\nu).
\end{align}
\end{lemma}
\begin{proof}
The 2-Wasserstein definition gives
\begin{align*}
    W_2^2(f_\# \mu, f_\# \nu) = \inf_{\pi' \in \Gamma(f_\# \mu, f_\# \nu) } \EE_{(f(x),f(y))\sim \pi'} \left\| f(x) - f(y) \right\|^2.
\end{align*}
Denote $(f \times f)(x,y) = ((f(x), f(y))$. For every $\pi' \in \Gamma (f_\# \mu, f_\# \nu)$, we have a corresponding $\pi \in \Gamma(\mu,\nu)$ that satisfies $\pi' = (f \times f)\# \pi$. Then, we have
\begin{align}\label{eq:pf:lemma}
    W_2^2(f_\# \mu, f_\# \nu) &= \inf_{\pi \in \Gamma(\mu, \nu) } \EE_{(x,y)\sim \pi'} \left\| f(x) - f(y) \right\|^2 \nonumber \\
    &\leq L^2 \inf_{\pi \in \Gamma(\mu, \nu) } \EE_{(x,y)\sim \pi'}  \left\| x - y \right\|^2 \nonumber \\
    & = L^2 W_2^2(\mu,\nu),
\end{align}
where the inequality follows from the Lipschitz property of $f$. Taking the square root on both sides of \eqref{eq:pf:lemma} completes the proof.
\end{proof}

Now we are ready to prove Theorem~\ref{thm:error_bound}.
By virtue of the triangle inequality for the Wasserstein distance \citep[Lemma~5.3]{santambrogio2015optimal}, we have
\begin{align}\label{eq:pf:thm2:1}
    W_2(q_1',[\phi_1^{\theta}]_\# \tilde{q}_0) &= W_2([\phi_1]_\# q_0', [\phi_1^{\theta}]_\# \tilde{q}_0) \\
    & \leq W_2([\phi_1]_\# q_0', [\phi_1^{\theta}]_\# {q}_0') + W_2([\phi_1^{\theta}]_\# q_0', [\phi_1^{\theta}]_\# \tilde{q}_0 ).
\end{align}

Since the learned vector field has uniform error bound $\epsilon$ and is $L_v$-Lipschitz continuous, by virtue of Theorem 1 in \cite{benton2023error}, the first term can be bounded by 
\begin{align}\label{eq:pf:thm2:2}
    W_2([\phi_1]_\# q_0', [\phi_1^{\theta}]_\# {q}_0') \leq \epsilon e^{L_v}.
\end{align}

In what follows, we analyze the Lipschitz property of $\phi_t^{\theta}$. Recall that $\phi_t^{\theta}$ is the flow of the learned vector field $v_t^{\theta}$. Let $x_t$ and $y_t$ be the solutions of the ODEs
\begin{align*}
    &d{x}_t = v_t^{\theta}(x_t) dt, \quad x_0 = x_0 \nonumber \\
    &d{y}_t = v_t^{\theta}(y_t) dt, \quad y_0 = y_0,
\end{align*}
respectively. Define $\Delta_t = \left\| x_t - y_t\right\|^2$. Then, we have
\begin{align*}
    \frac{d{\Delta}_t}{dt}  = 2 \langle x_t - y_t, \frac{d{x}_t}{dt} - \frac{d{y}_t}{dt} \rangle  = 2 \langle x_t - y_t, v_t^{\theta}(x_t) -v_t^{\theta}(y_t)\rangle \leq 2L_v \left\| x_t - y_t\right\|^2 = 2L_v \Delta_t.
\end{align*}
Integrating from $0$ to $t$ gives
\begin{align*}
    \Delta_t \leq \Delta_0 + 2L_v \int_{0}^t \Delta_s ds.
\end{align*}
By virtue of Grönwall’s inequality, we have
\begin{align*}
    \Delta_t \leq \Delta_0 e^{2L_v t} = \left\| x_0 - y_0\right\|^2 e^{2L_v t}.
\end{align*}
By taking the square root, we have that $\phi_t^{\theta}(x)$ is $e^{L_v t}$-Lipschitz continuous in $x$. In particular, at $t=1$, $\phi_1^{\theta}(x)$ is $e^{L_v}$ Lipschitz continuous.
Then, it follows from Lemma~\ref{lemma:1} that the second term is bounded by 
\begin{align}\label{eq:pf:thm2:3}
    W_2([\phi_1^{\theta}]_\# q_0', [\phi_1^{\theta}]_\# \tilde{q}_0 ) \leq e^{L_v} W_2(q_0', \tilde{q}_0).
\end{align}

Substituting \eqref{eq:pf:thm2:2} and \eqref{eq:pf:thm2:3} into \eqref{eq:pf:thm2:1}, we have the desired result. The proof is complete.

\section{Additional Discussions}

\subsection{Related Works}
\paragraph{Diffusion guidance:} Conditional sampling has been widely studied in diffusion models \cite{chung2022diffusion,song2023loss,ye2024tfg,guo2024gradient,wu2023practical,xu2024provably}. However, the diffusion model requires the source distribution to be Gaussian, and cannot handle general source distributions. Therefore, these guidance methods cannot be applied here.


\paragraph{Flow matching guidance:} The flow guidance methods can be divided into two groups:
training-based guidance and
training-free guidance. Training-based guidance \cite{zheng2023guided} requires retraining when we have a different conditioning. Therefore, this paper focuses on training-free guidance \cite{ben2024d,feng2025guidance}. One closely related training-free guidance method is D-Flow \cite{ben2024d}, which proposes to optimize the source samples via a regularized optimization problem. 
However, its optimization objective is heuristic, whereas our framework provides the missing theoretical foundation.
Besides, \cite{feng2025guidance} proposed a training-free guidance method that keeps the original source distribution and modifies the vector field. 
Such an approach generates curved vector fields and, therefore, requires a large number of discretization steps to integrate the ODE. 
Moreover, the exactness of this guidance method applies to a limited class of pre-trained vector fields and lacks generality.


\paragraph{Guidance via stochastic optimal control (SOC):} Optimal control methods have been used to guide generative models \cite{uehara2024fine,tang2024fine,wang2024training,domingo2024adjoint}. Specifically, \cite{wang2024training} augments the vector field with an additional control term, obtained by solving a SOC problem. However, \cite{wang2024training} does not connect the generated distribution with the target distribution, and there is a bias between these two distributions. 
The works \cite{uehara2024fine,tang2024fine} cancel out this bias by both modifying the vector field and shifting the initial distribution. More recently, \cite{domingo2024adjoint} showed that solely adjusting the vector field is able to remove the bias if the noise schedule is appropriately selected.
Our method is orthogonal to the guidance methodology of \cite{domingo2024adjoint}: 
we remove the bias by solely shifting the source distribution. Moreover, whenever we have a new guidance energy function, 
SOC-based guidance methods have to re-solve the SOC problem, which is computationally expensive.

\paragraph{Guidance by optimizing the source distribution: } Conditional generation by optimizing the source distribution has been explored in \cite{ben2024d,wallace2023end,tang2024inference,novack2024ditto,karunratanakul2024optimizing}. These works propose to propagate the gradient from the target criteria through the whole generation
process to update the initial noise. However, their optimization objectives are heuristically designed without theoretical justification. Besides, these optimization-based approaches are easy to get trapped in local minima and lose diversity. In this paper, we propose a unified framework with theoretical justification and flexible choices of sampling methods.

\subsection{Sampling algorithms}
\subsubsection{Importance sampling}\label{sec:app:is}
Following the discussion on importance sampling (IS) in Section \ref{sec:sampling:exact}, a detailed outline of the method with target density $q=q_0'$ and proposal density $m=q_0$ is given in Algorithm \ref{alg:importance_sampling} \cite{Chopin2020}.
\begin{algorithm}[t]
\caption{Importance Sampling} \label{alg:importance_sampling}
\begin{algorithmic}[1]
    \STATE \textbf{Input}: samples from $x_0\sim q_0$
    \STATE Set $w(\cdot)\triangleq\frac{q_0'(\cdot)}{q_0(\cdot)}=e^{-J\circ T^*(\cdot)}$
    \STATE Compute weights $W^n=\frac{w(x_0^n)}{\sum\limits_{m}w(x_0^m)}$
    \STATE Sample $x_0'$ from $\{x_0^n\}$ with probabilities $\{W^n\}$
    \STATE \textbf{Output}: sample $x_0'$ from $q_0'$
\end{algorithmic}
\end{algorithm}

\subsubsection{Unadjusted and Metropolis Adjusted Langevin algorithms}
The Unadjusted Langevin Algorithm (ULA) \cite{robert1999monte} generates approximate samples from a target distribution with density $q(x) \propto \exp(-U(x))$ by discretizing the Langevin stochastic differential equation (SDE). Specifically, given a step-size $\eta_k > 0$, the ULA update is
\begin{align}
x_{k+1} = x_k - \eta \nabla U(x_k) + \sqrt{2\eta_k}, \xi_k,
\end{align}
where $\xi_k \sim \mathcal{N}(0, I)$ are independent Gaussian noise. Due to discretization errors, ULA introduces sampling bias. 

The Metropolis Adjusted Langevin Algorithm (MALA) \cite{robert1999monte} improves upon ULA by incorporating a Metropolis-Hastings correction step to ensure exact sampling from the target distribution $q(x)$. Given a current state $x_k$, MALA proposes a candidate $x'$ via
\begin{align}
x' = x_k - \eta \nabla U(x_k) + \sqrt{2\eta}, \xi_k,
\end{align}
and accepts it with probability:
$\alpha(x_k, x') = \min\left\{1, \frac{\pi(x')q(x_k|x')}{q(x_k)q(x'|x_k)}\right\}$,
where $q(\cdot|\cdot)$ denotes the transition density induced by the proposal step. If rejected, the chain remains at $x_k$. This correction guarantees that the stationary distribution matches exactly the target distribution $\pi(x)$.

\subsubsection{Hamiltonian Monte Carlo}\label{sec:app:hmc}
Hamiltonian Monte Carlo (HMC) \cite{neal2011mcmc} is an accept--reject MCMC method for unnormalized continuous densities on $\mathbb{R}^d$ where partial derivatives of the log density exist. By associating the target variable with the position of a particle in space and the density with its potential energy, the method introduces an auxiliary momentum variable and implements Hamiltonian dynamics to achieve extensive exploration while maintaining a high acceptance probability.

Specifically, the unnormalized target density $q$ is expressed in canonical form $q(x)\propto e^{-U(x)}$, where $U(x)\triangleq -\ln q(x)$ represents the potential energy. The momentum variable $v$ gives the kinetic energy $K(v)\triangleq\frac{\|v\|^2}{2}$. This forms the Hamiltonian $H(x,v)=U(x)+K(v)$ with the associated joint distribution $\pi(x,v)\propto e^{-\big(U(x) + K(v)\big)}$, where $x$ and $v$ are considered independent with marginals $q$ and the standard Gaussian distribution respectively. HMC generates samples from $\pi$ with MCMC, where each chain iteration starts by resampling the momentum, $v'\sim \mathcal{N}(0, I)$, while keeping the position unchanged, $x'=x$. Then, a Metropolis update step is implemented by generating proposals using Hamiltonian dynamics
\begin{equation}\label{eq:hamiltonian}
    \frac{dx}{dt}=v, \quad \frac{dv}{dt}=-\nabla_xU
\end{equation} to propagate $(x', v')$ along trajectories of constant energy to $(x^*, v^*)$ and accepting the new state with probability $\alpha=\frac{\pi(x^*, v^*)}{\pi(x', v')}$.

Integrating \eqref{eq:hamiltonian} with the leapfrog method (Algorithm \ref{alg:leapfrog}) ensures $\alpha\approx 1$, as $H$ is approximately constant and the transformation is volume-preserving. Still, the integration may move $x$ to positions with very different marginal density $U(x)$. The resampling step prevents the marginal $U$ from being constrained by the initial value of $H$. Thus, the momentum variable is critical for efficient exploration of the space. Algorithm \ref{alg:hmc} implements HMC when $q=q_0'$, initializing the process by $q_0$.

HMC can be tuned by appropriately choosing the step size and the number of leapfrog steps \cite{neal2011mcmc}. It is generally advised to choose the parameters such that the empirical acceptance rate is around the optimal value of 65\%. One may also randomly select these parameters from fairly small intervals at each Markov chain iteration to ensure that both big steps and fine-tuning steps can be taken at various points in the chain.

\begin{algorithm}[t]
\caption{Leapfrog integrator $\eta_{\epsilon, L}$} \label{alg:leapfrog}
\begin{algorithmic}[1]
    \STATE \textbf{Input}: initial state $(x_0, v_0)$\;
    $v_{0}=v_0-\frac{\epsilon}{2}\nabla U(x_0)$\;
    \For{$m=0$ \KwTo $L-1$}{
    $x_{m + 1}=x_m + \epsilon v_{m}$\;
    $v_{m + 1} = v_m - \epsilon\nabla U(x_{m+1})$\;}
    $v_{L}=v_L+\frac{\epsilon}{2}\nabla U(x_L)$
    \STATE \textbf{Output}: $(x_{L}, v_{L})$
\end{algorithmic}
\end{algorithm}

\begin{algorithm}[t]
\caption{Hamiltonian Monte Carlo} \label{alg:hmc}
\begin{algorithmic}[1]
    \STATE \textbf{Input}: samples from $x_0\sim q_0$\;
    \For{$n=0$ \KwTo $N-1$}{
    $v_n'\sim N(0, I)$\;
    $(x^*, v^*)=\eta_{\epsilon, L}(x_n, v_n')$\;
    $\alpha = e^{-\big(U(x^*)+K(v^*)\big)+U(x_n)+K(v_n')}$\;
    Draw $u \sim \mathcal{U}(0,1)$\;
    \If{$u < \alpha$}{
    $x_{n+1} = x^*$}
    \Else{$x_{n+1} = x_n$}
    }
    \STATE Set $x_0'=x_N$
    \STATE \textbf{Output}: sample $x_0'$ from $q_0'$
\end{algorithmic}
\end{algorithm}
\subsubsection{Optimization-based sampling}\label{sec:app:optsamp}
To solve \eqref{eq:opt_problem:our} or \eqref{eq:opt_problem2:our}, we can use any preferred optimization algorithm such as stochastic gradient descent (SGD) or Limited-memory Broyden–Fletcher–Goldfarb–Shanno (L-BFGS). Using the \text{torchdiffeq} package, the gradient of the objective can be computed via automatic differentiation. With access to the gradient, we iteratively refine the initial sample using a standard update rule. Starting from an initial $x_0^{(0)}$, the update takes the form 
$$x_0^{(k+1)}=\texttt{OPT\_Alg}(x_0^{(k)} ),$$
where \texttt{OPT\_Alg} denotes the chosen optimization routine. We can feed the final $x_0^{K}$ into $T^*$ to generate the sample $x_1 = T^*(x_0^{K})$.
\section{Additional Experimental Details}

\subsection{2D experiments}\label{sec:app:exp:2D}
In this section, we present more details on the 2D example in Section~\ref{sec:experiment:2D}, including implementation details and additional experimental results. All experiments were run on a single NVIDIA A100 GPU.

\subsubsection{Implementation details}

\paragraph{Flow matching model:} The vector field is approximated using a time-varying multilayer perceptron (MLP) adopted from \cite{tong2023improving}. We train a standard vector field model using an independent coupling distribution $\pi= q_0 \times q_1$, and an optimal vector field model using the optimal joint distribution $\pi^*$ in \eqref{eq:CFM:loss}. Each model is trained for 20,000 epochs with a batch size of 256, employing the Adam optimizer.

\paragraph{Conditional sampling:} We consider three pairs of source and target distributions: (i) 8-Gaussian to Moon, (ii) Uniform to 8-Gaussian, and (iii) Circle to S-Curve. For these three tasks, we respectively select loss functions $J(x)=((x[2])^2)/0.4$, $J(x) = 4|x[1]+x[2]|$, and $J(x) = 5|x[1]-x[2]|$, where $x := (x[1], x[2])$.

\paragraph{Implementation of D-Flow in \cite{ben2024d}:}
Among several choices of regularization terms in D-Flow, we employ $- \ln q_1$, which ensures the generated samples stay close to the target distribution $q_1$. Although $q_1$ generally lacks an explicit form, in this 2D experiment, we approximate it using kernel density estimation. 
For the pre-trained vector fields used in D-Flow, we evaluate two variants: a standard model trained with an independent $\pi= q_0 \times q_1$ and an optimal vector field model trained with $\pi^*$ in \eqref{eq:CFM:loss}. 
We refer to these variants as D-Flow and D-Flow-OT, respectively. In the optimization process, we use 60 optimization steps and employ SGD as the optimizer. 

\paragraph{Implementation of methods in \cite{feng2025guidance}:}
Among several training-free methods in \cite{feng2025guidance}, we select two of the best methods $g^{sim-MC}$ and $g^{MC}$, which perform well in low-dimensional settings. We use 100 and 50 Monte Carlo samples for $g^{sim-MC}$ and $g^{MC}$, respectively. For both methods, the pre-trained model is selected as the optimal vector field.

\paragraph{Implementation of SGFM:} We evaluate SGFM with four sampling methods: IS, ULA, MALA, and HMC. For ULA, MALA, and HMC, we run 100 sampling iterations. In both MALA and HMC, the proposal step‐size is tuned to target an acceptance rate of 60\%. Besides, each HMC iteration employs $L=5$ leapfrog steps. 

\paragraph{Evaluation metric:} The sample quality is measured using the 1-Wasserstein distance between the generated and ground truth distributions. The generated distribution is measured using $10^3$ samples, while the ground truth distribution is estimated using $10^4$ samples. All experiments were conducted ten times, with solid lines and shaded areas representing the mean and standard deviation.

\subsubsection{Additional results}

We conduct an extensive comparison across different source and target distributions. The generated distributions are visualized in Figure~\ref{fig:my-10row-comparison}. 
In the first experiment (8-Gaussian to moon), we observe that all the baseline methods and most Langevin-based algorithms struggle. The key reason is the highly multimodal landscape of this task, which makes the sampler easy to get trapped in local minima. However, SGFM-IS successfully navigates the posterior distribution.
In the second experiment (circular to S-curve), many guidance methods perform well, but D-Flow tends to overemphasize minimizing the loss $J$ and loses sample diversity.
In the third experiment (uniform to S-curve), we observe that D-Flow fails to generate a satisfactory conditional distribution, and $g^{MC}$ exhibits slight deterioration in sample quality. 
Across every experiment, SGFM-IS consistently delivers high-quality, diverse samples, with SGFM-MALA and SGFM-HMC providing strong alternatives in the latter two tasks.

\begin{figure}[htbp]
  \centering
  \begin{tabular}{ cccc}
    & & & \\
    \midrule
    \textbf{ $q_1$}
      & \includegraphics[width=0.15\linewidth,valign=t]{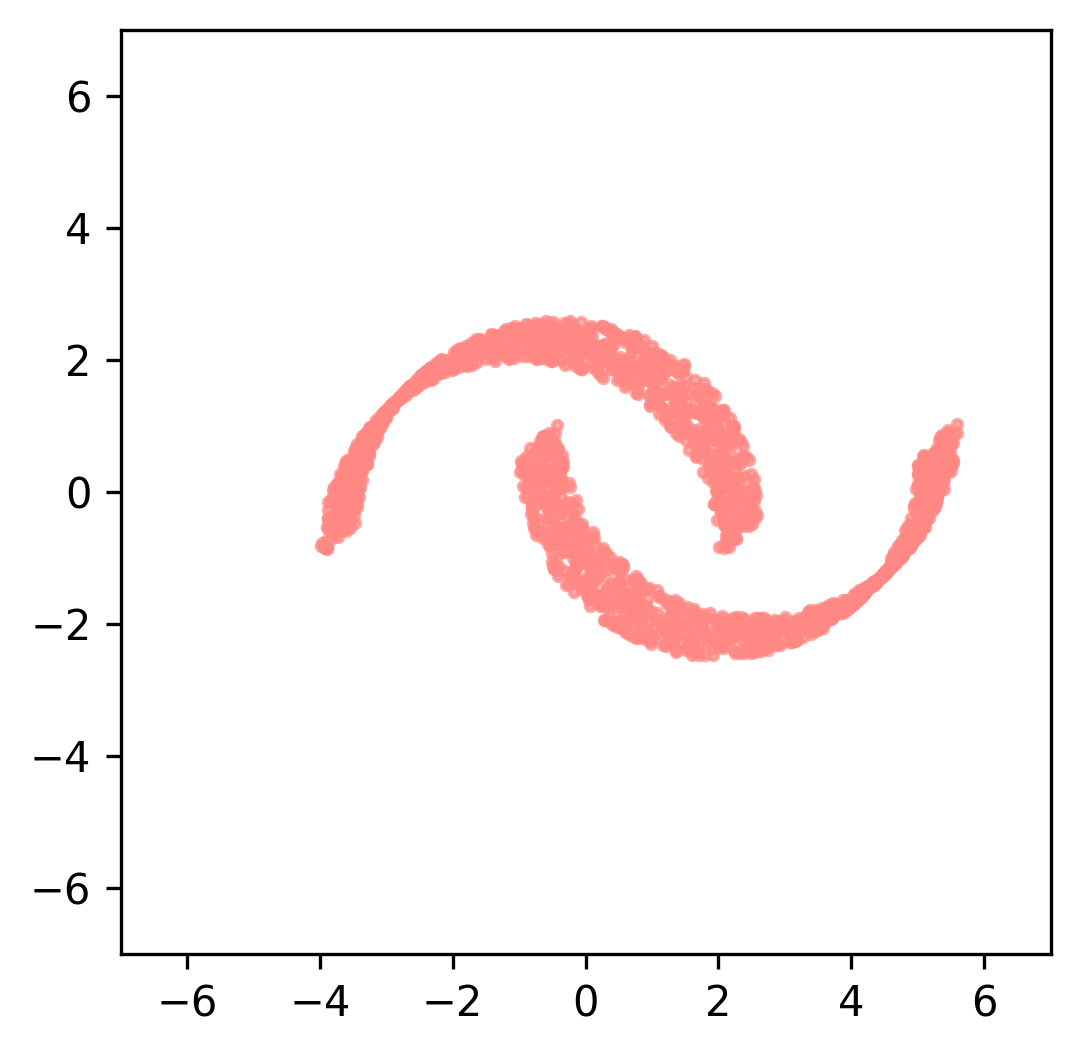}
      & \includegraphics[width=0.15\linewidth,valign=t]{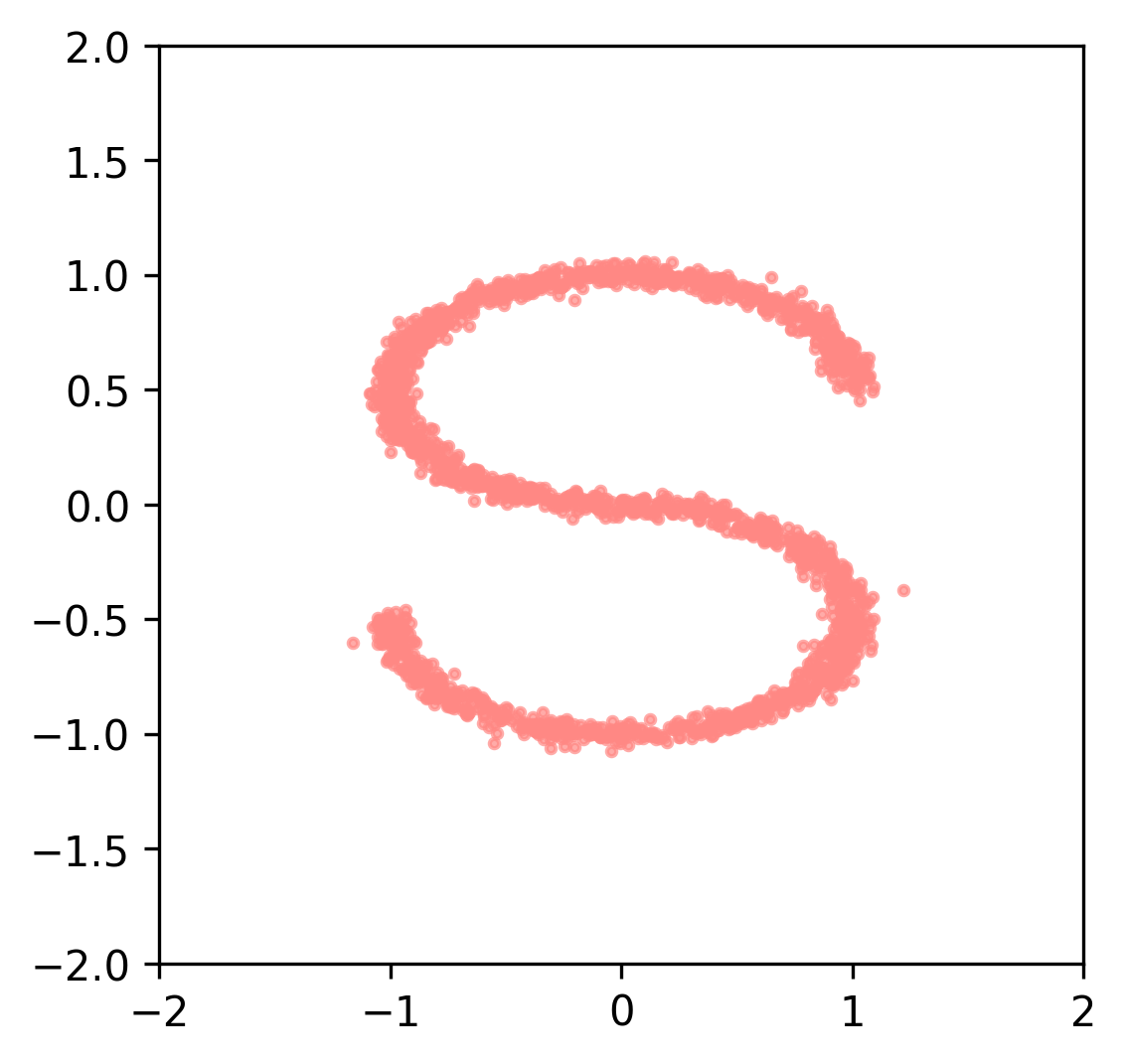} 
      & \includegraphics[width=0.15\linewidth,valign=t]{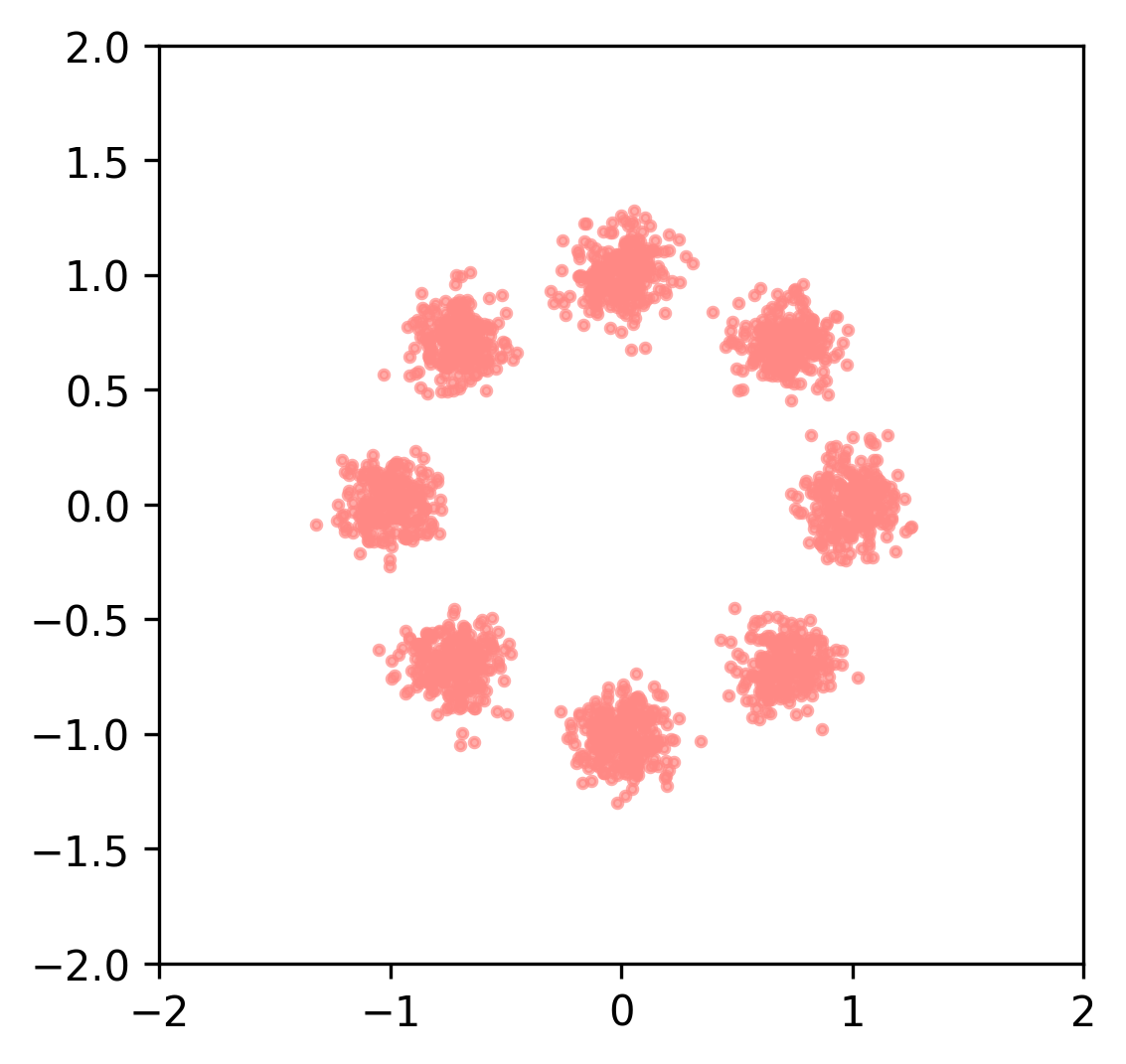} \\
    \textbf{ $q_1'$ (ground truth)}
      & \includegraphics[width=0.15\linewidth,valign=t]{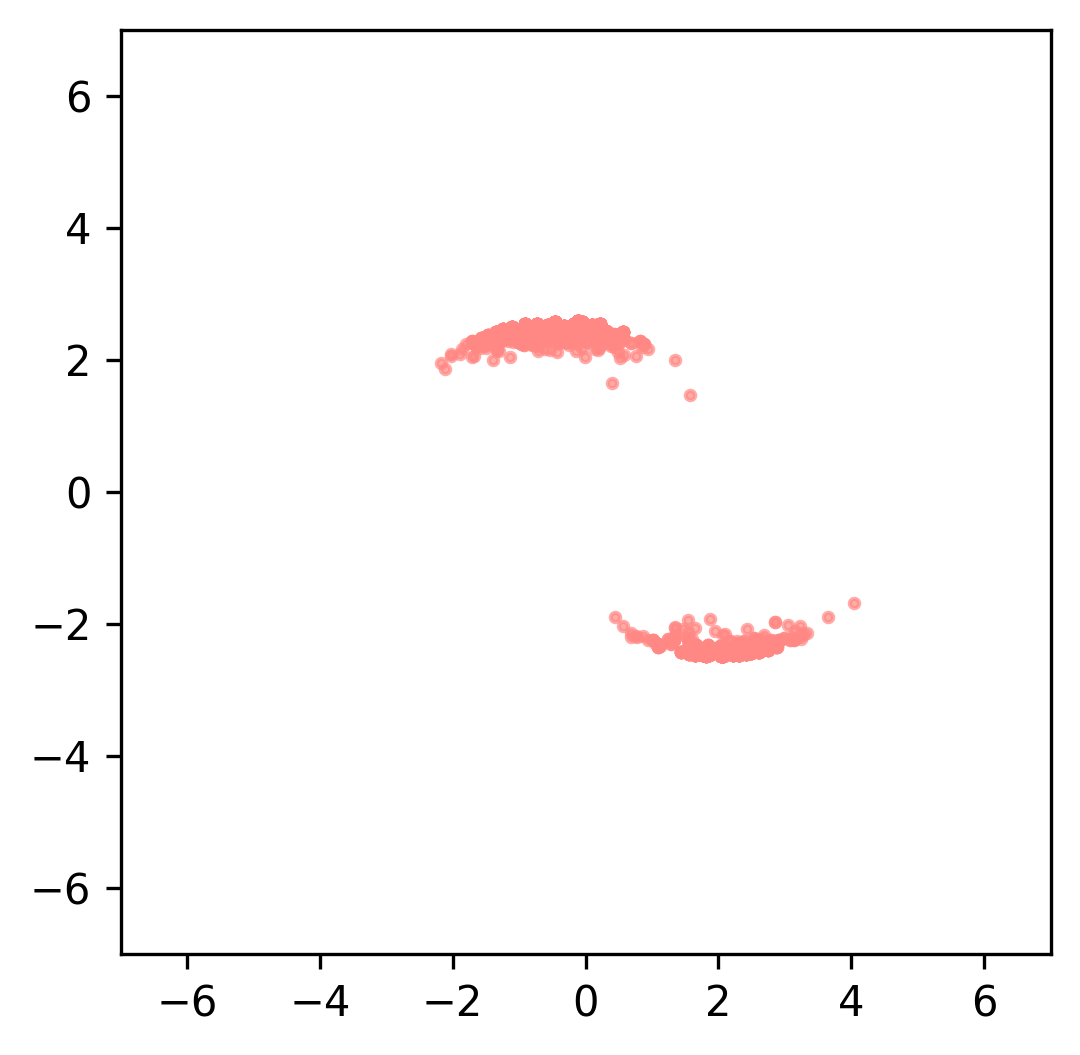}
      & \includegraphics[width=0.15\linewidth,valign=t]{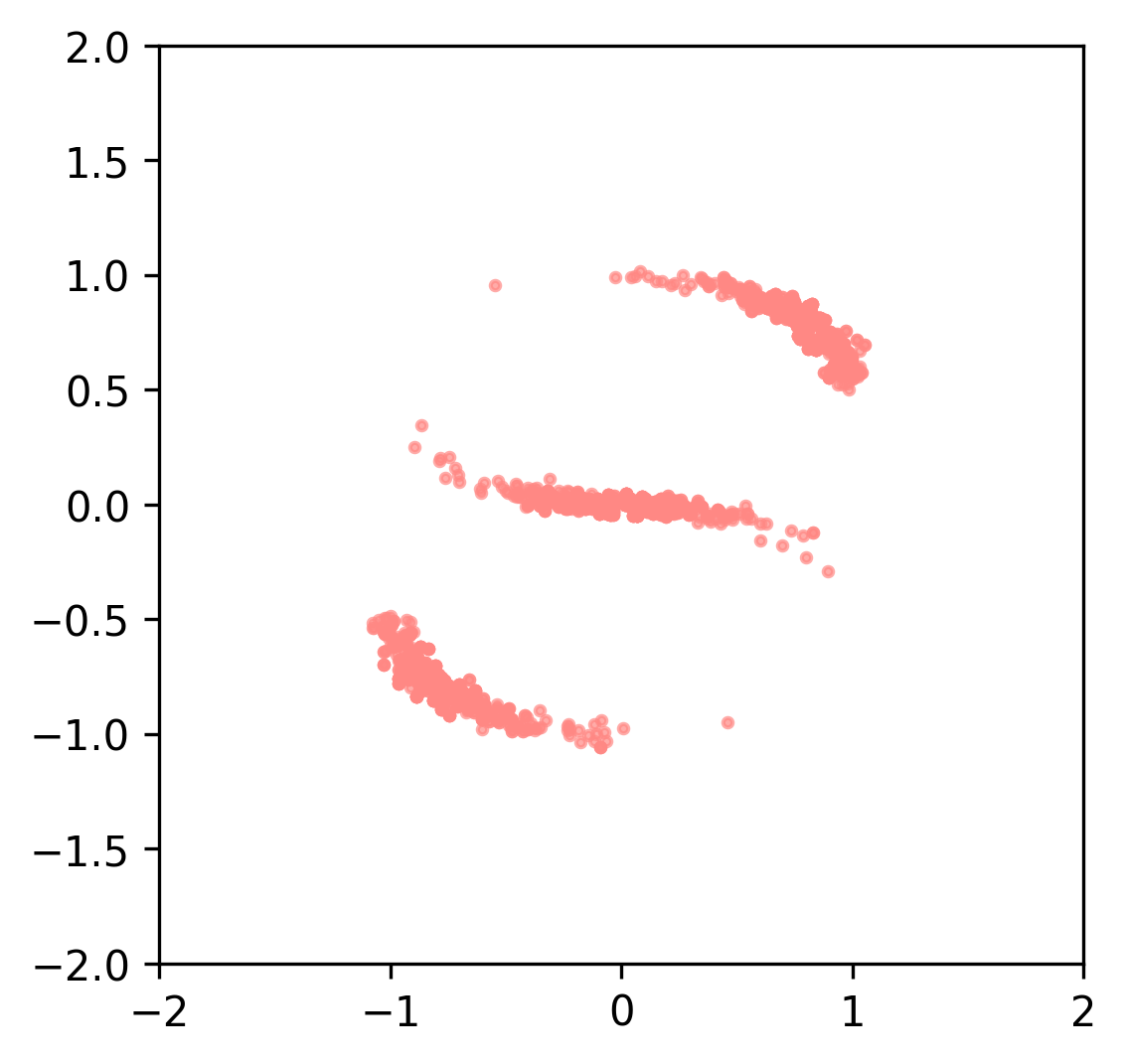} 
      & \includegraphics[width=0.15\linewidth,valign=t]{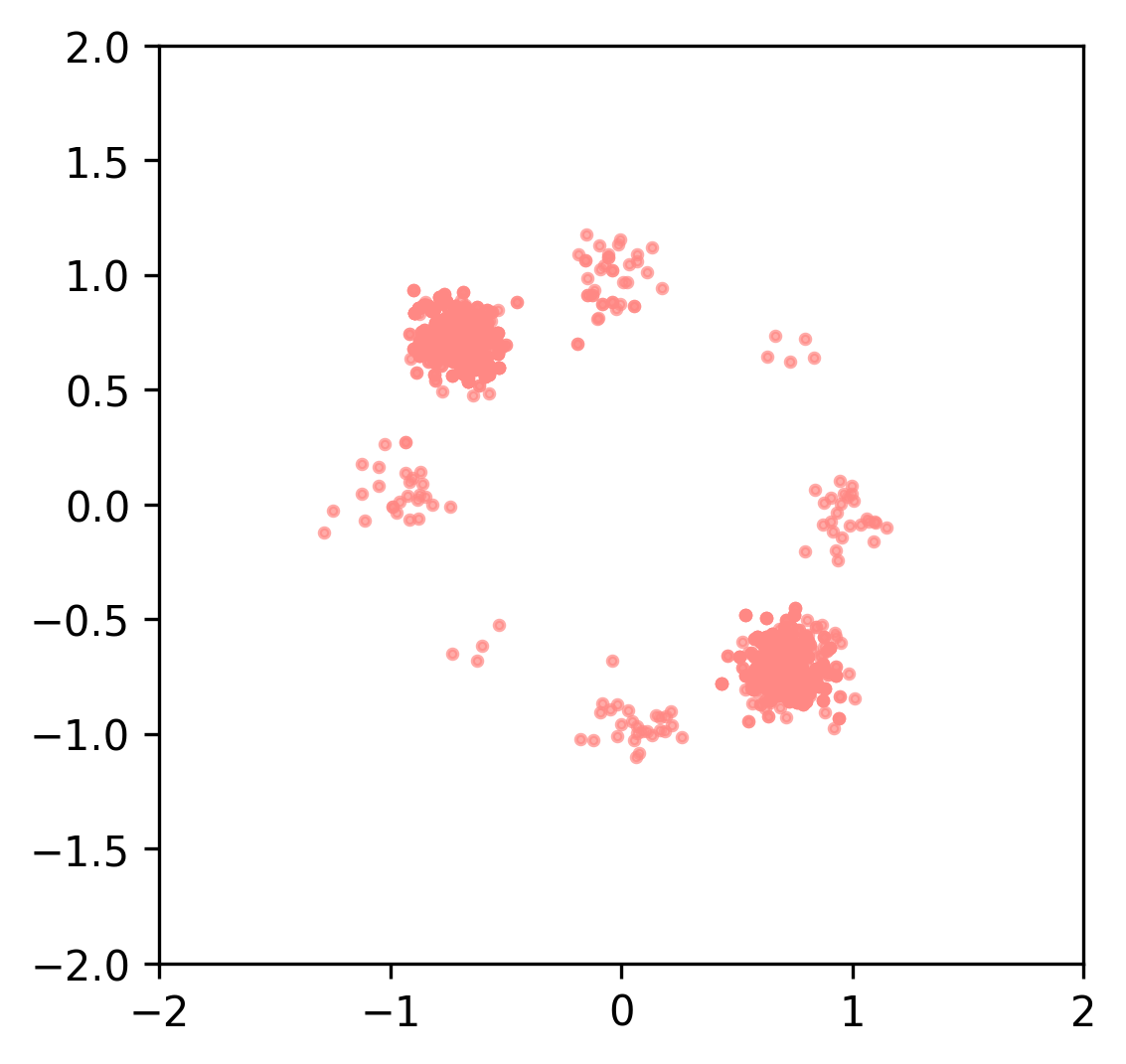} \\
    \textbf{D-Flow \cite{ben2024d}}
      & \includegraphics[width=0.15\linewidth,valign=t]{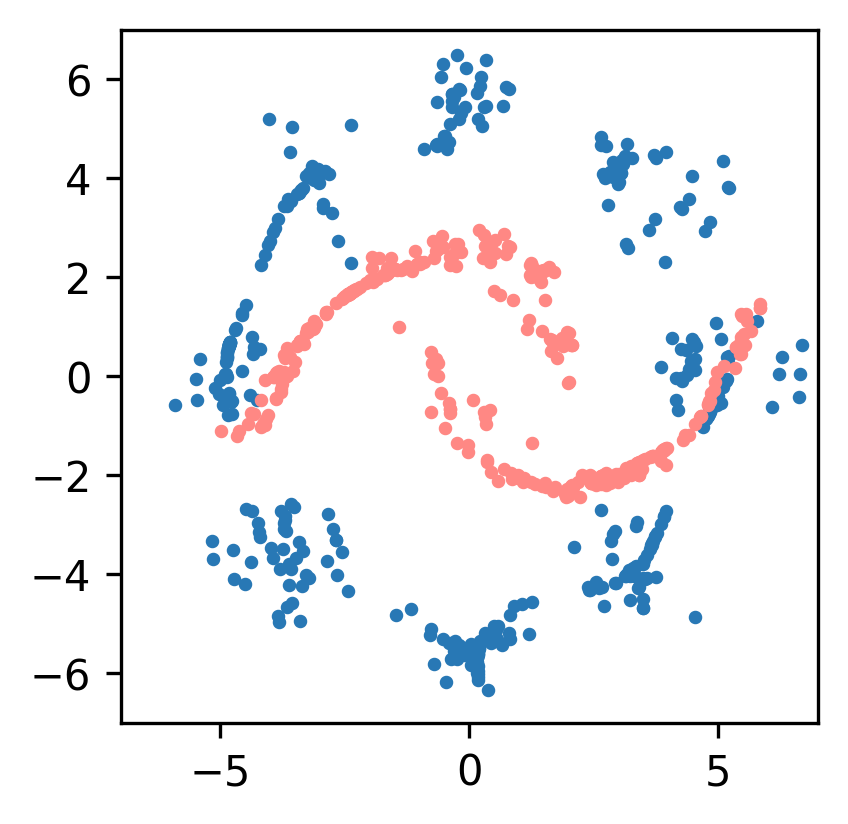}
      & \includegraphics[width=0.15\linewidth,valign=t]{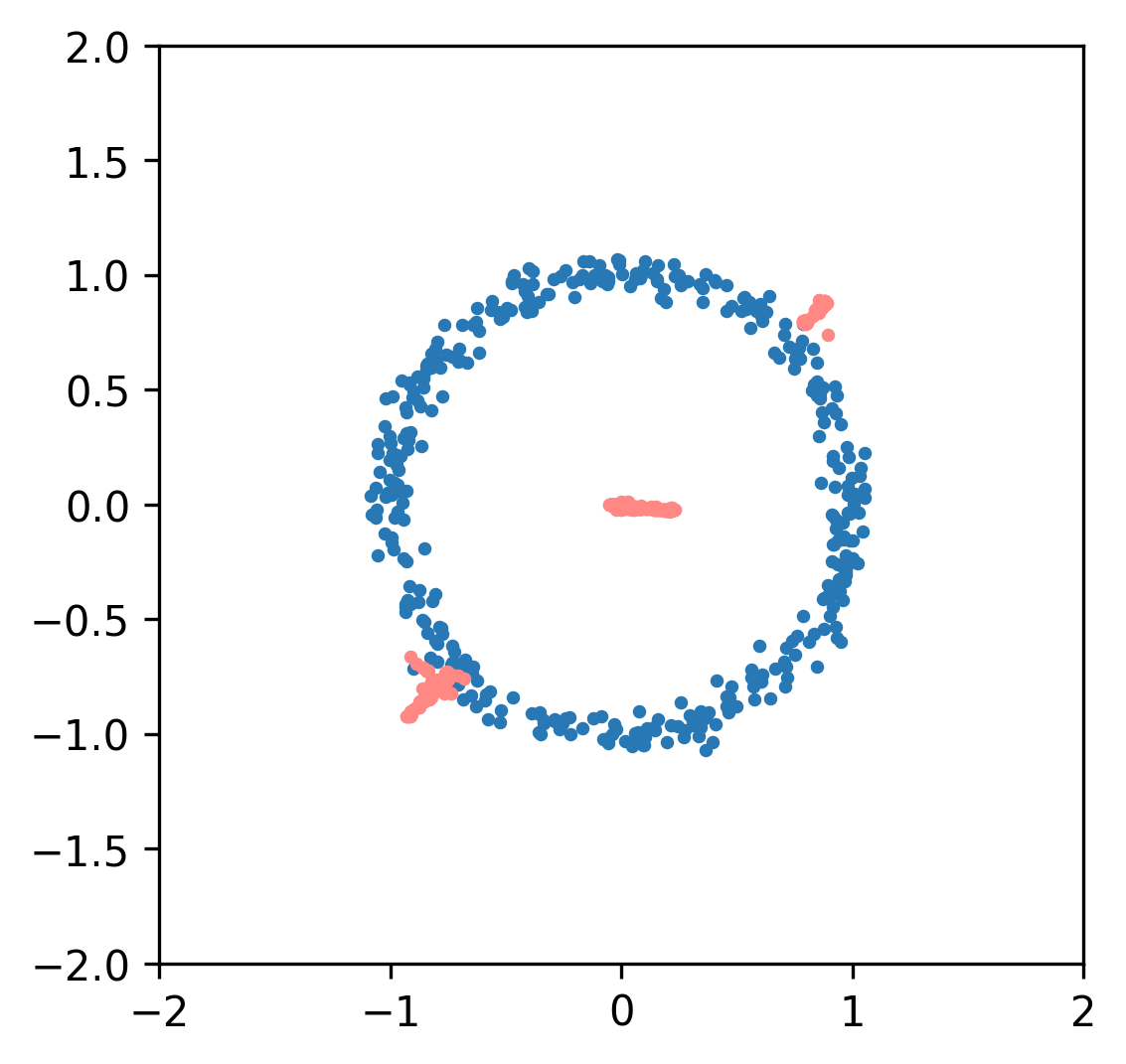} 
      & \includegraphics[width=0.15\linewidth,valign=t]{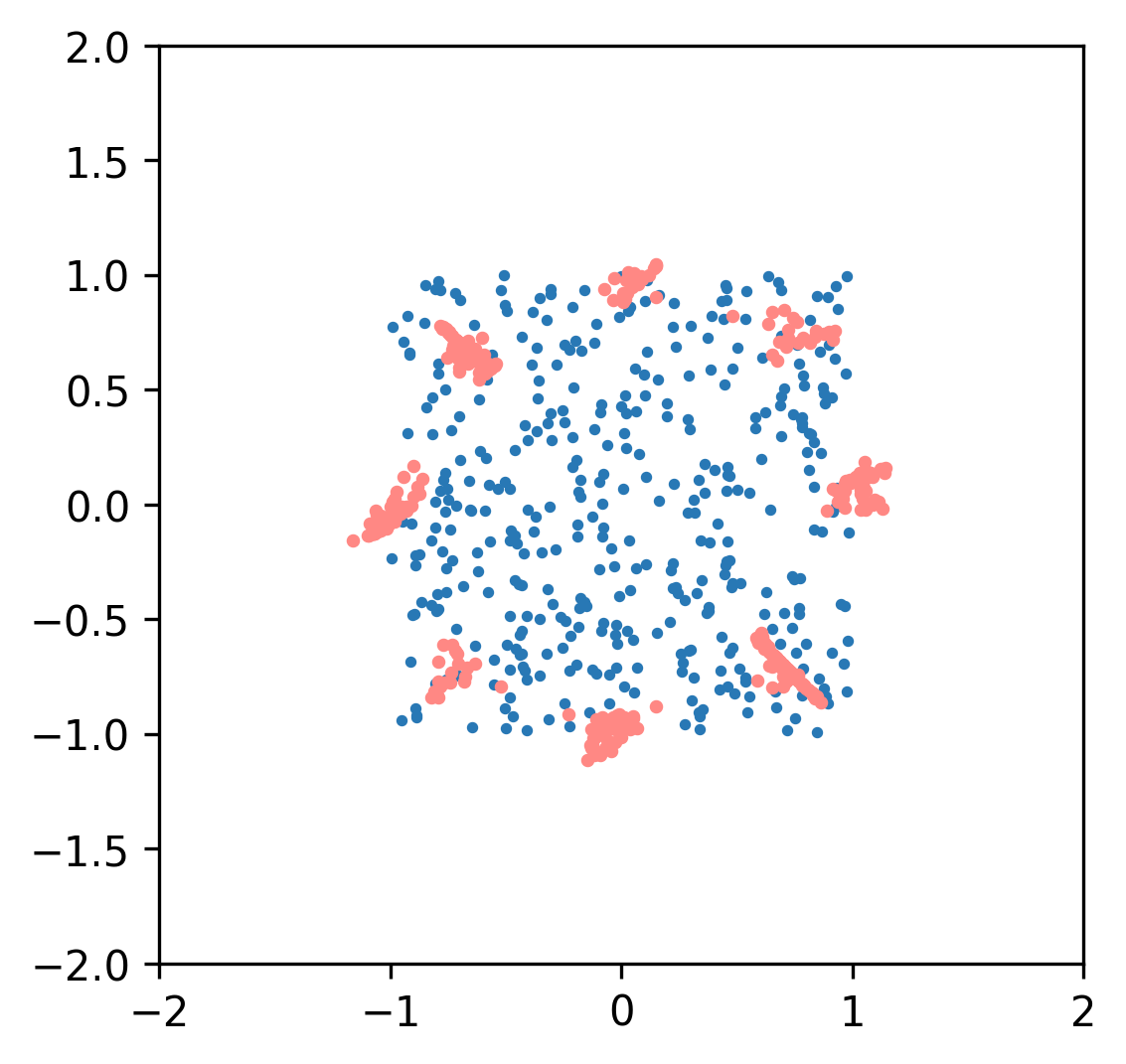} \\
    \textbf{D-Flow-OT \cite{ben2024d}}
      & \includegraphics[width=0.15\linewidth,valign=t]{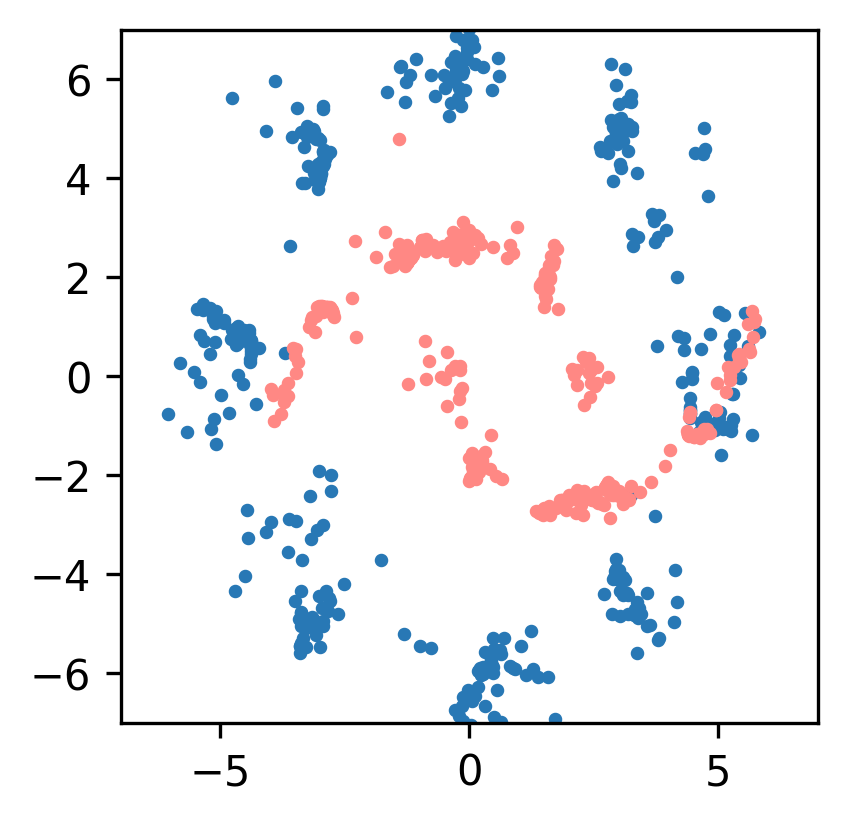}
      & \includegraphics[width=0.15\linewidth,valign=t]{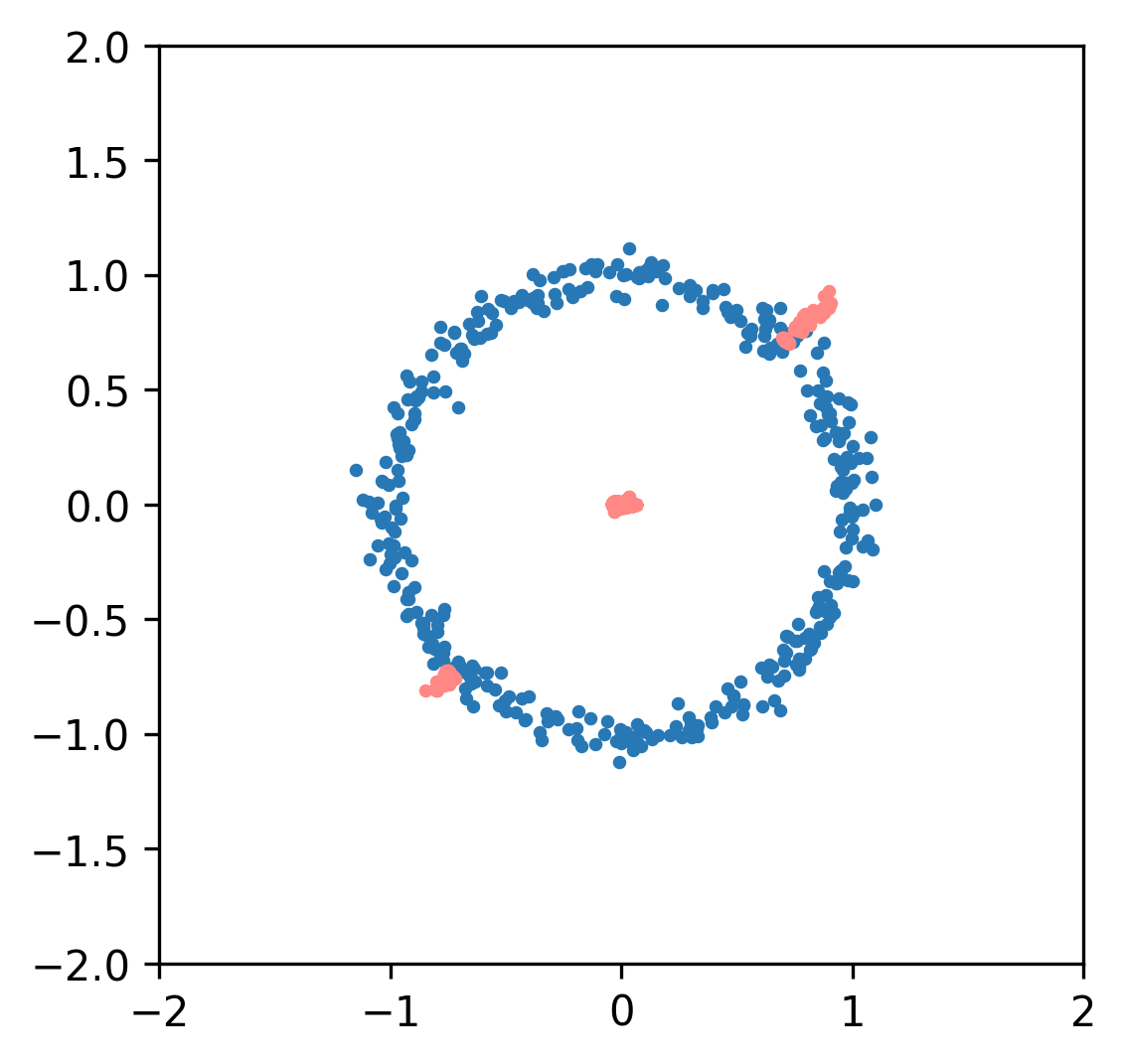} 
      & \includegraphics[width=0.15\linewidth,valign=t]{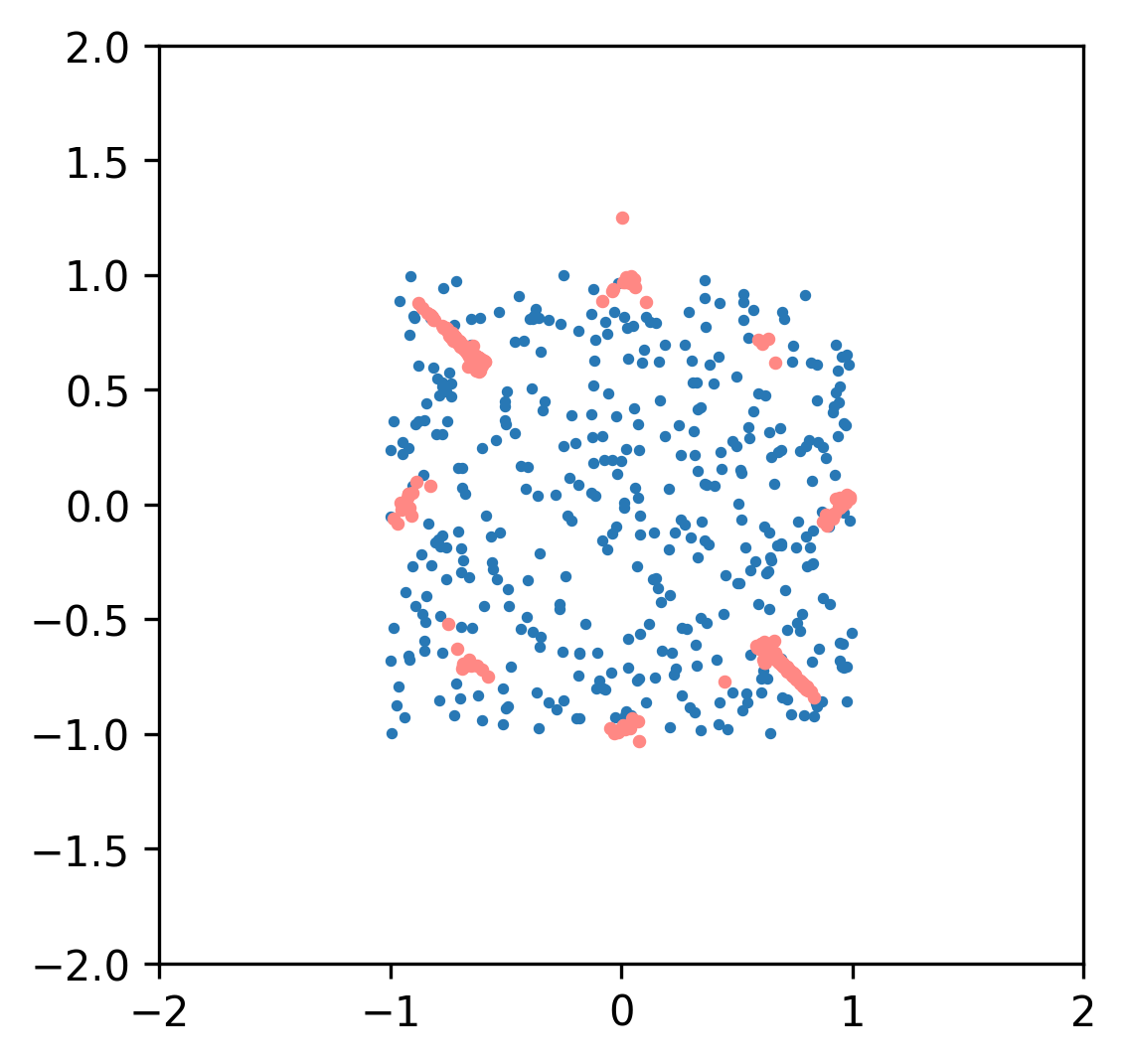} \\
    \textbf{$g^{sim-MC}$ \cite{feng2025guidance}}
      & \includegraphics[width=0.15\linewidth,valign=t]{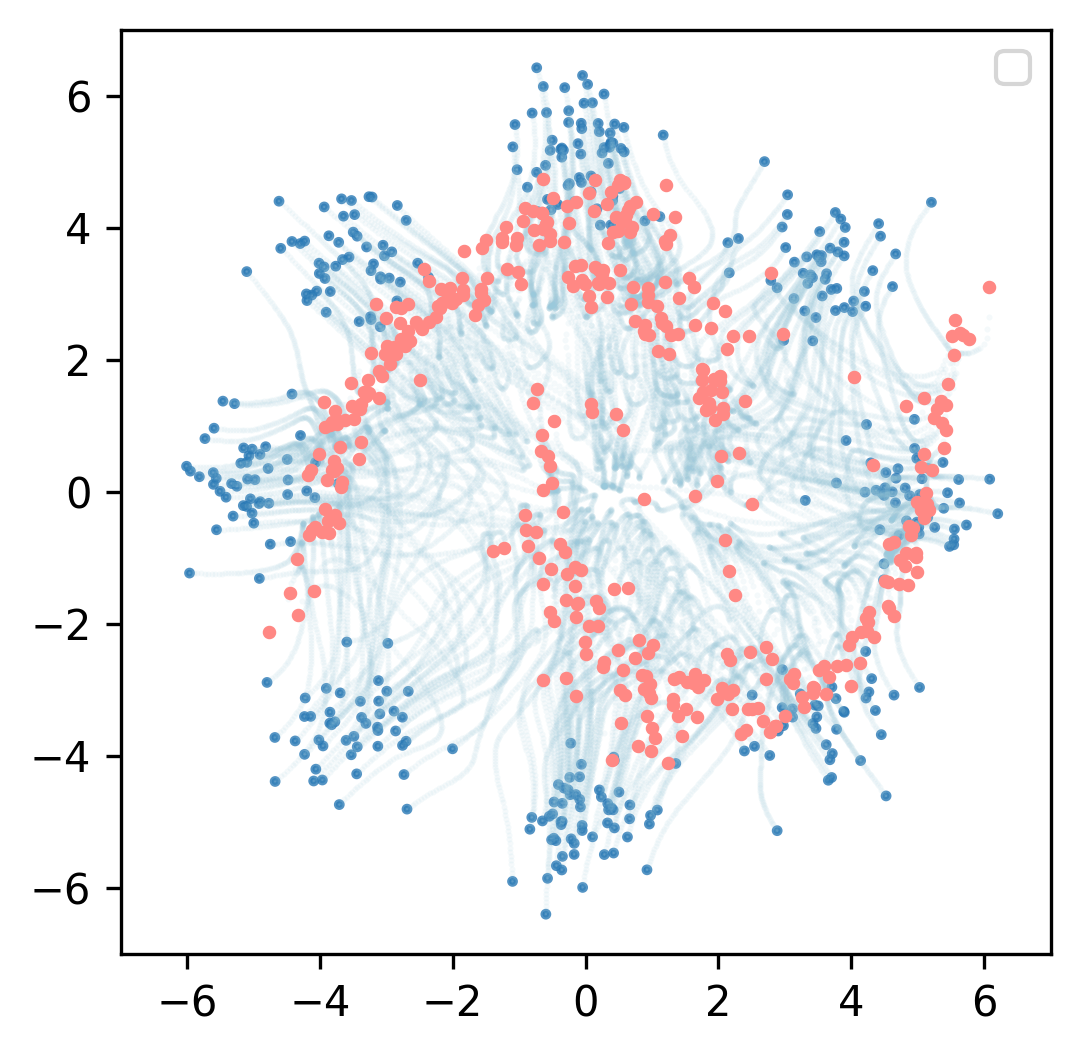}
      & \includegraphics[width=0.15\linewidth,valign=t]{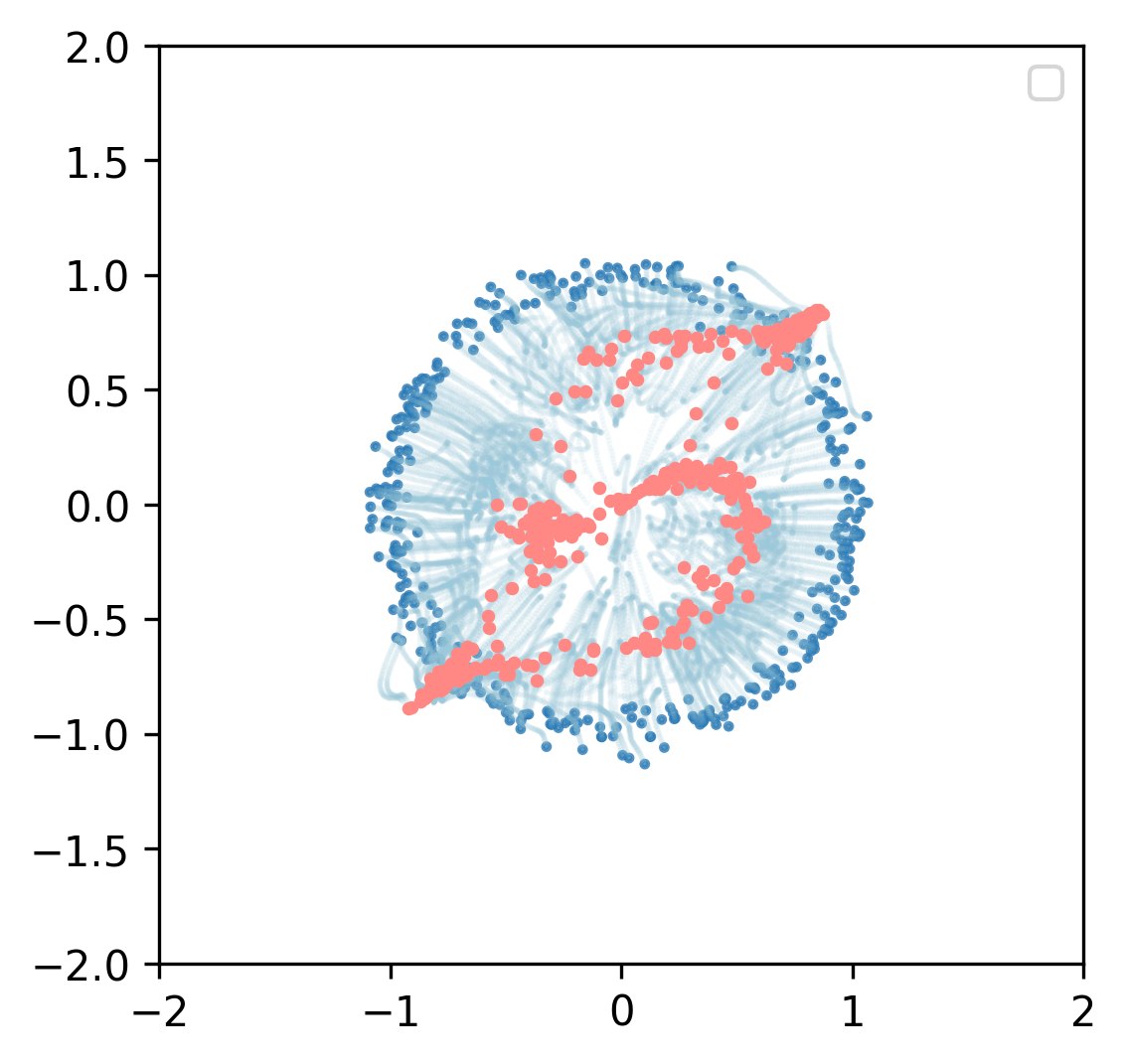} 
      & \includegraphics[width=0.15\linewidth,valign=t]{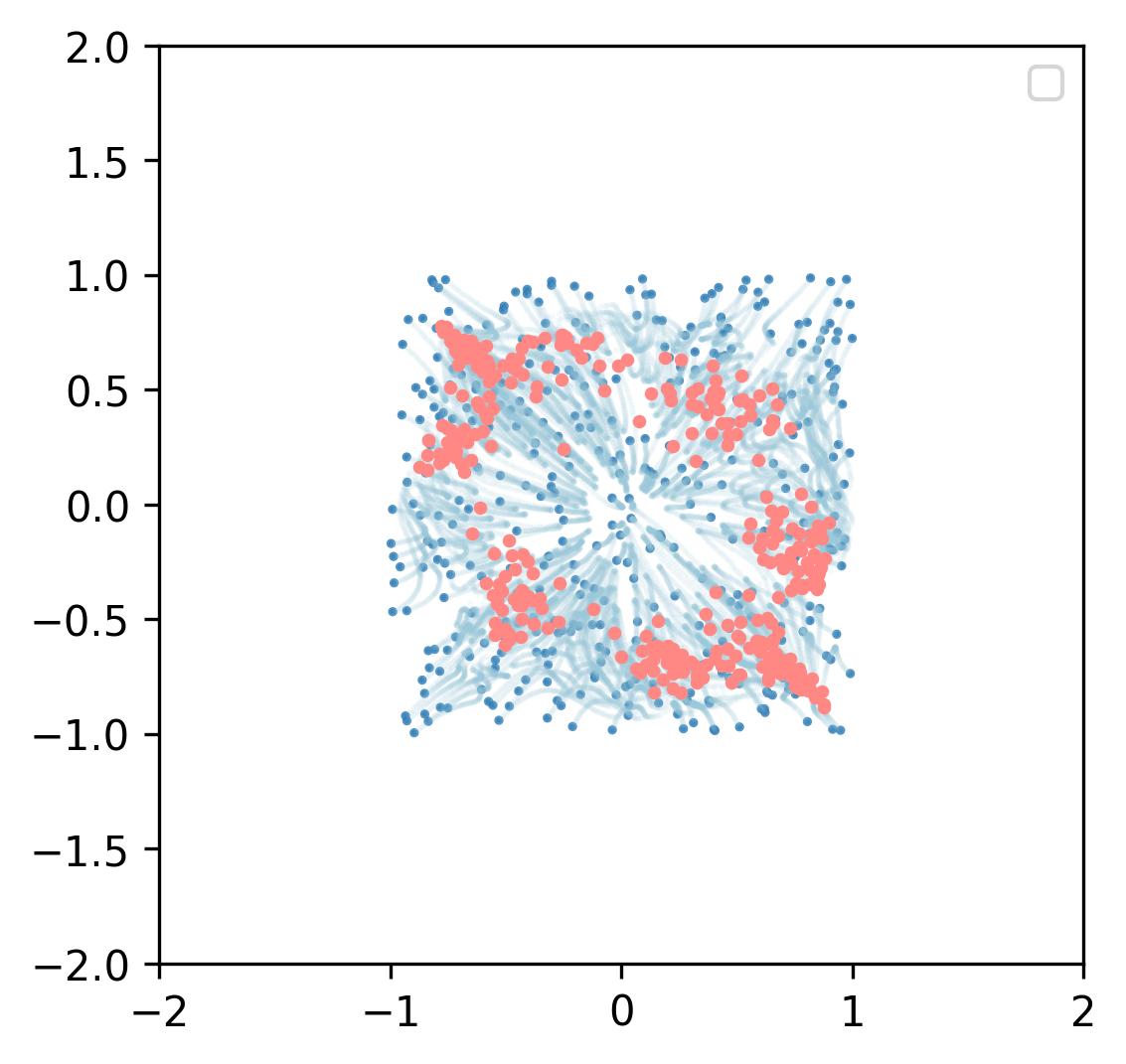} \\
    \textbf{$g^{MC}$ \cite{feng2025guidance}}
      & \includegraphics[width=0.15\linewidth,valign=t]{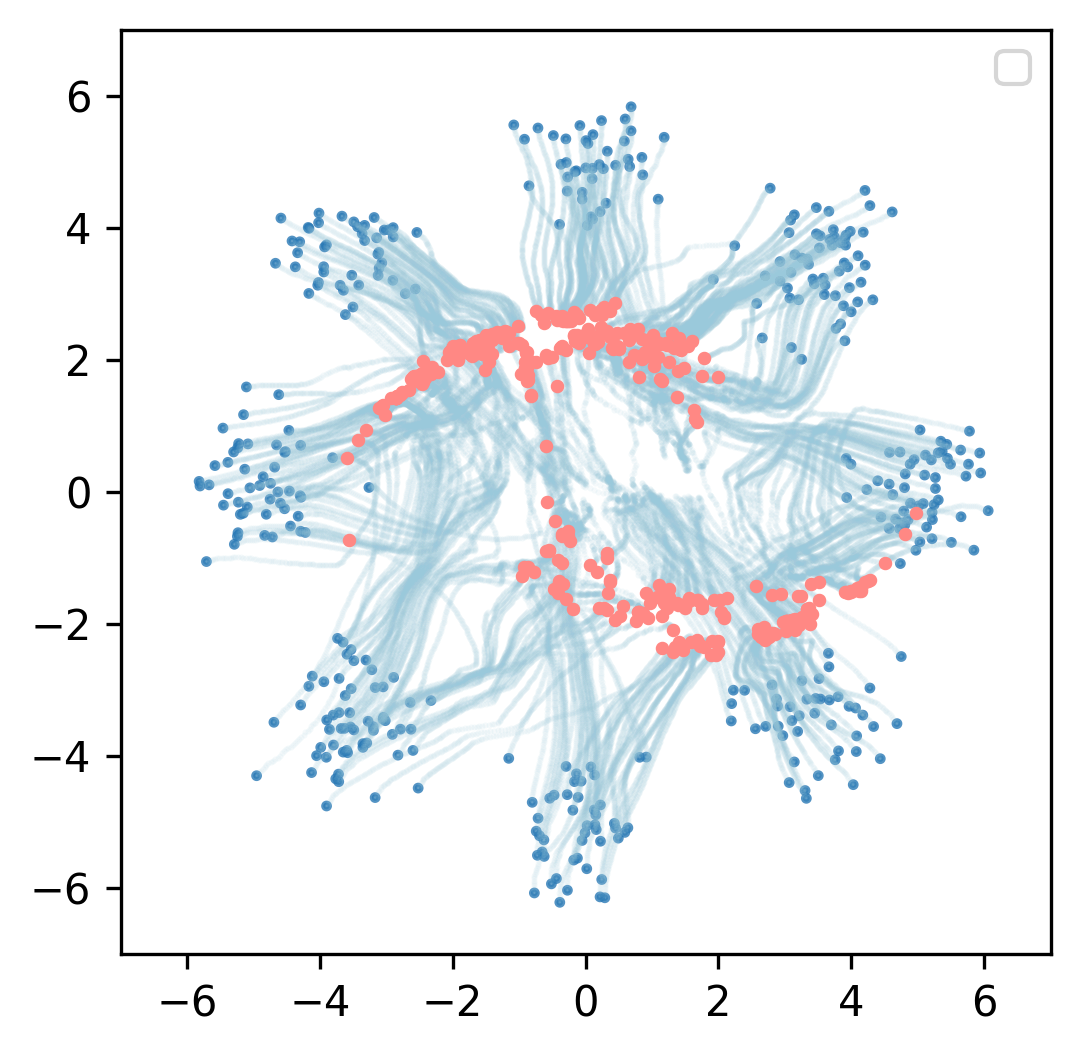}
      & \includegraphics[width=0.15\linewidth,valign=t]{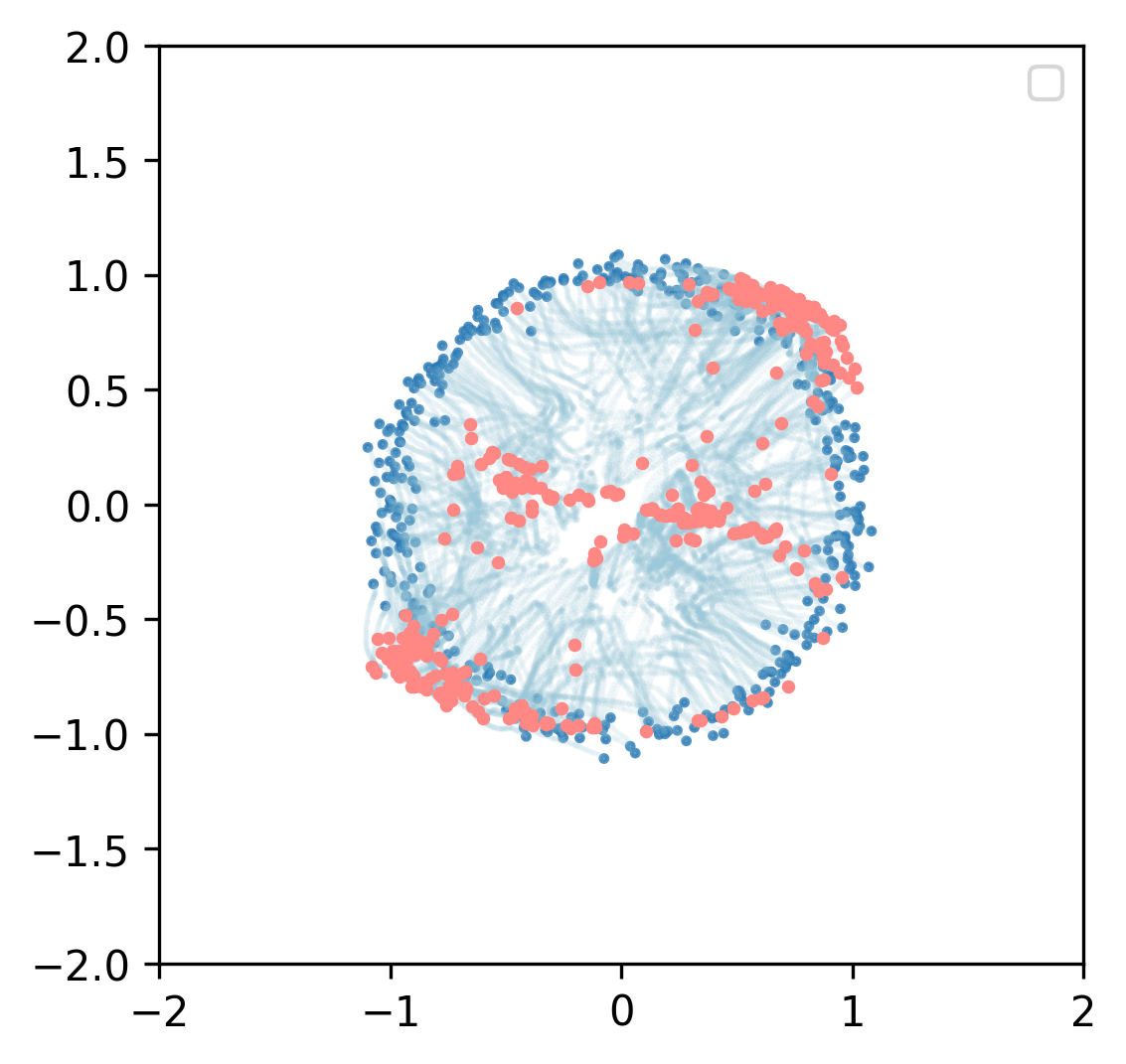} 
      & \includegraphics[width=0.15\linewidth,valign=t]{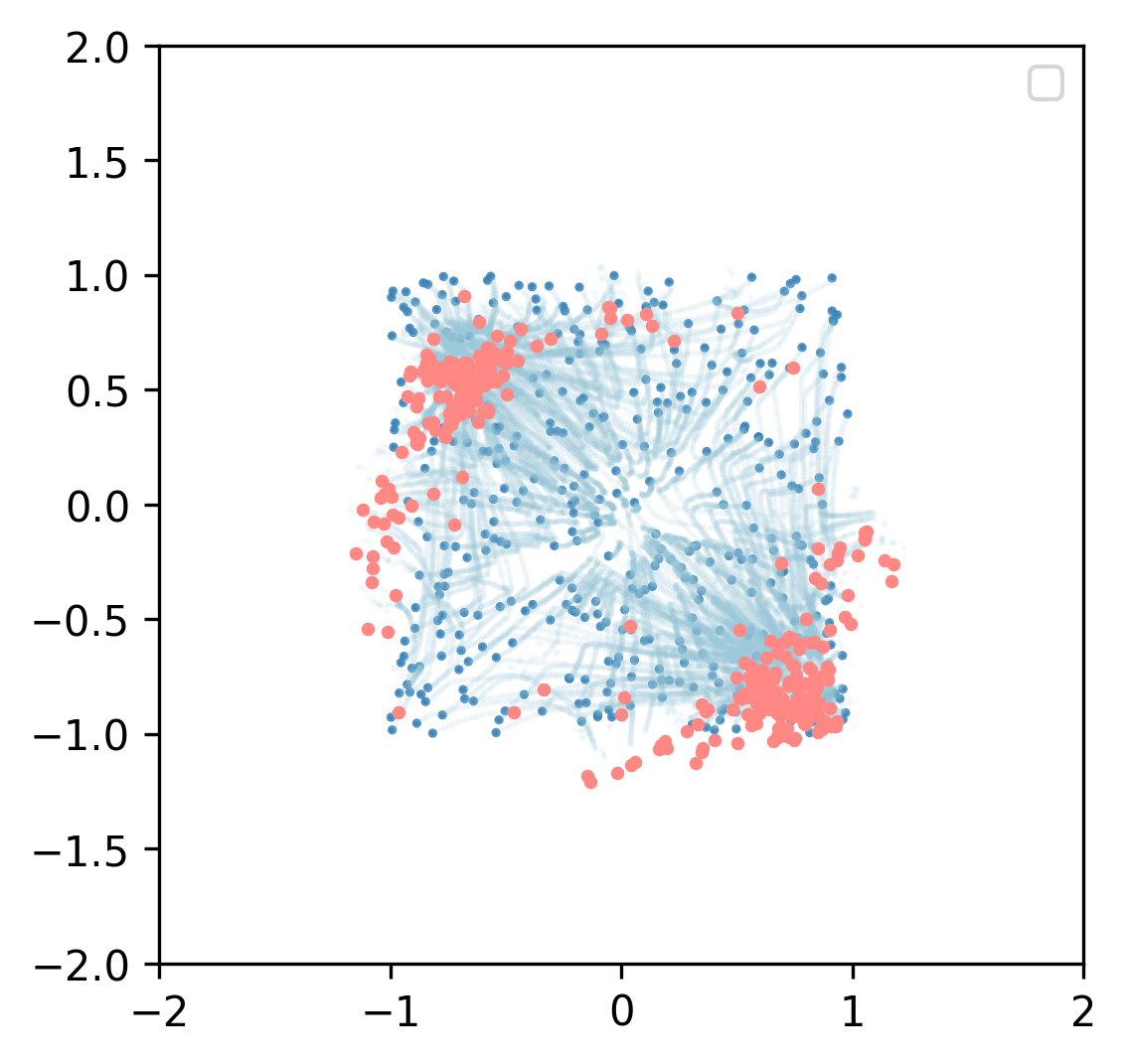} \\
    \textbf{SGFM-IS}
      & \includegraphics[width=0.15\linewidth,valign=t]{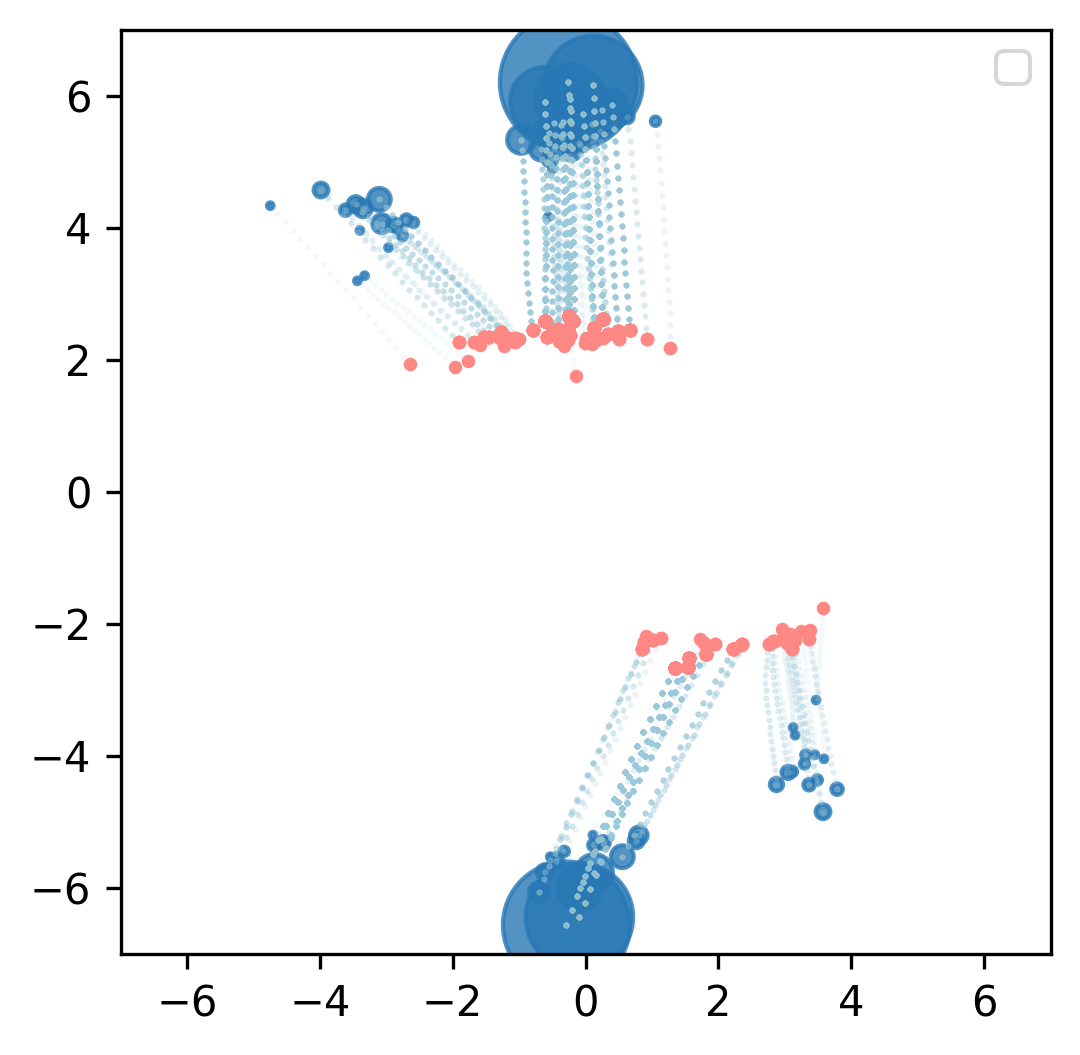}
      & \includegraphics[width=0.15\linewidth,valign=t]{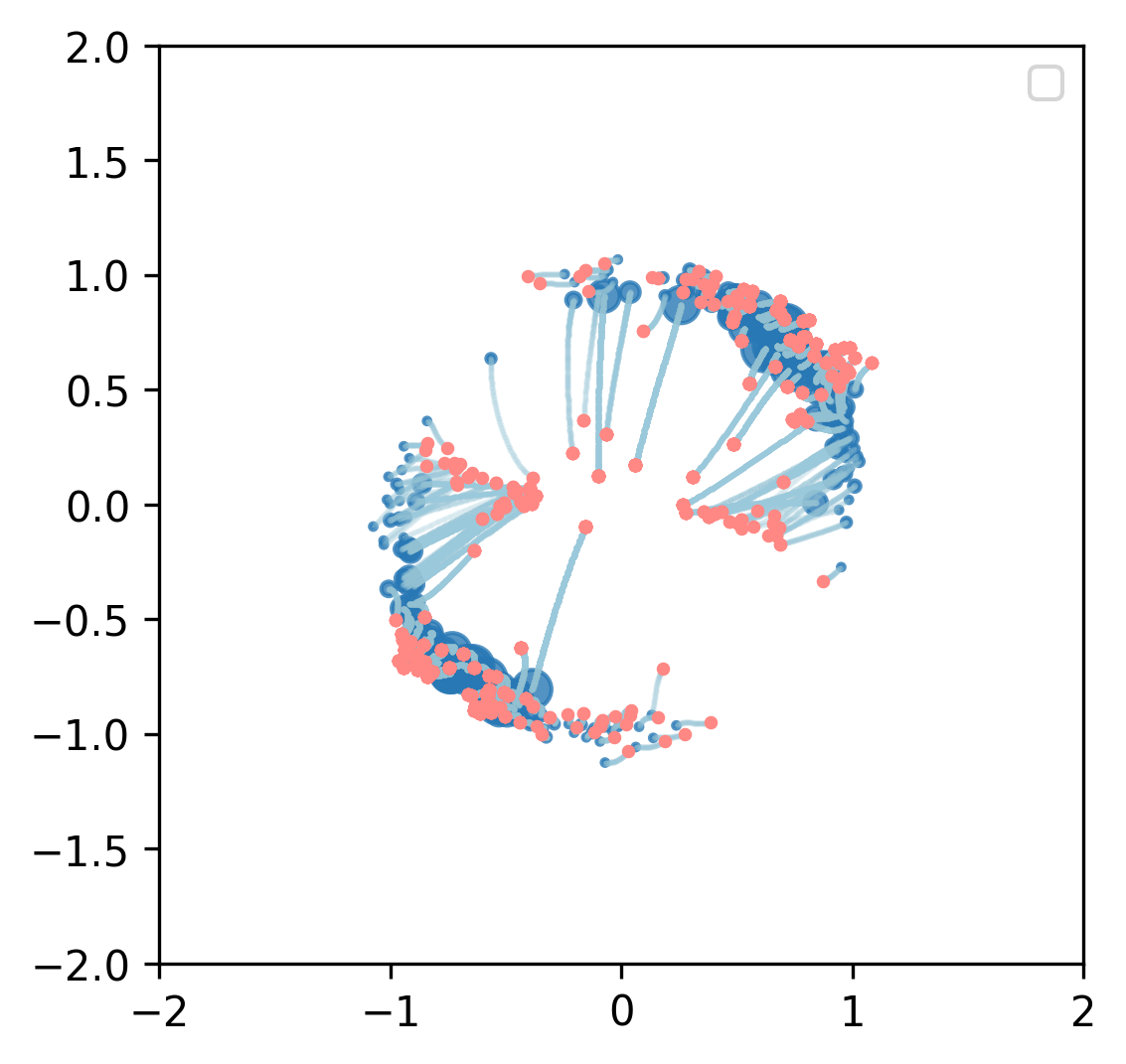} 
      & \includegraphics[width=0.15\linewidth,valign=t]{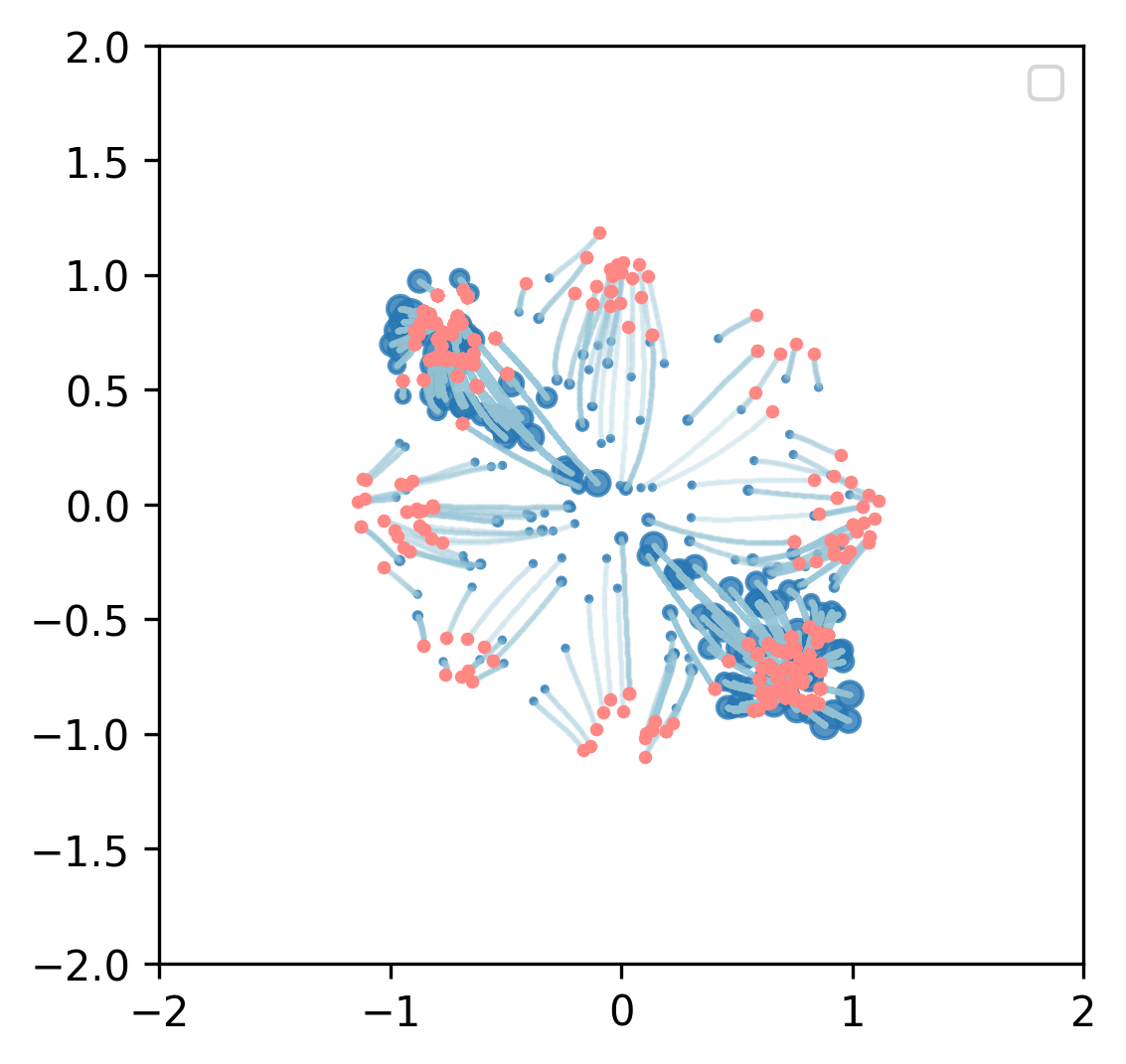} \\
    \textbf{SGFM-ULA}
      & \includegraphics[width=0.15\linewidth,valign=t]{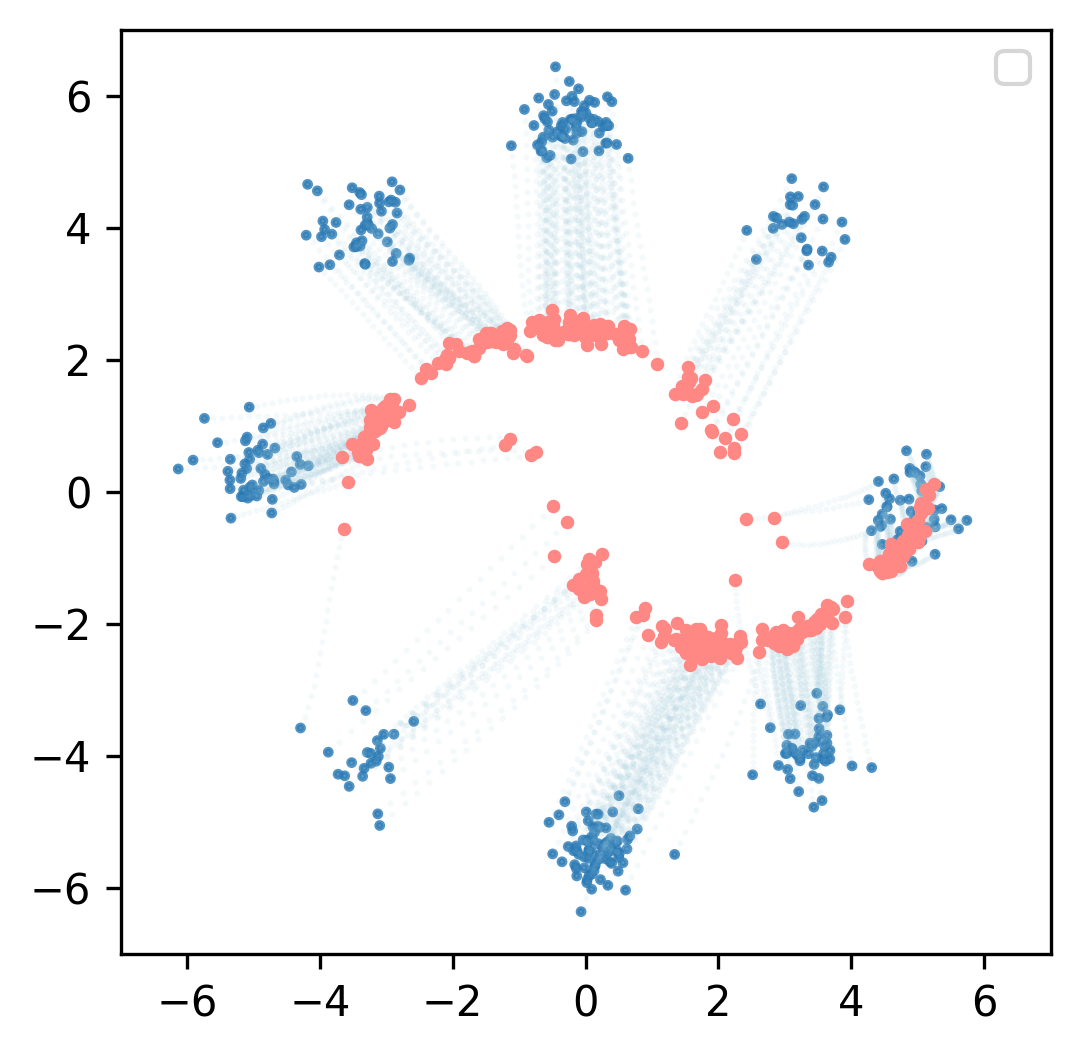}
      & \includegraphics[width=0.15\linewidth,valign=t]{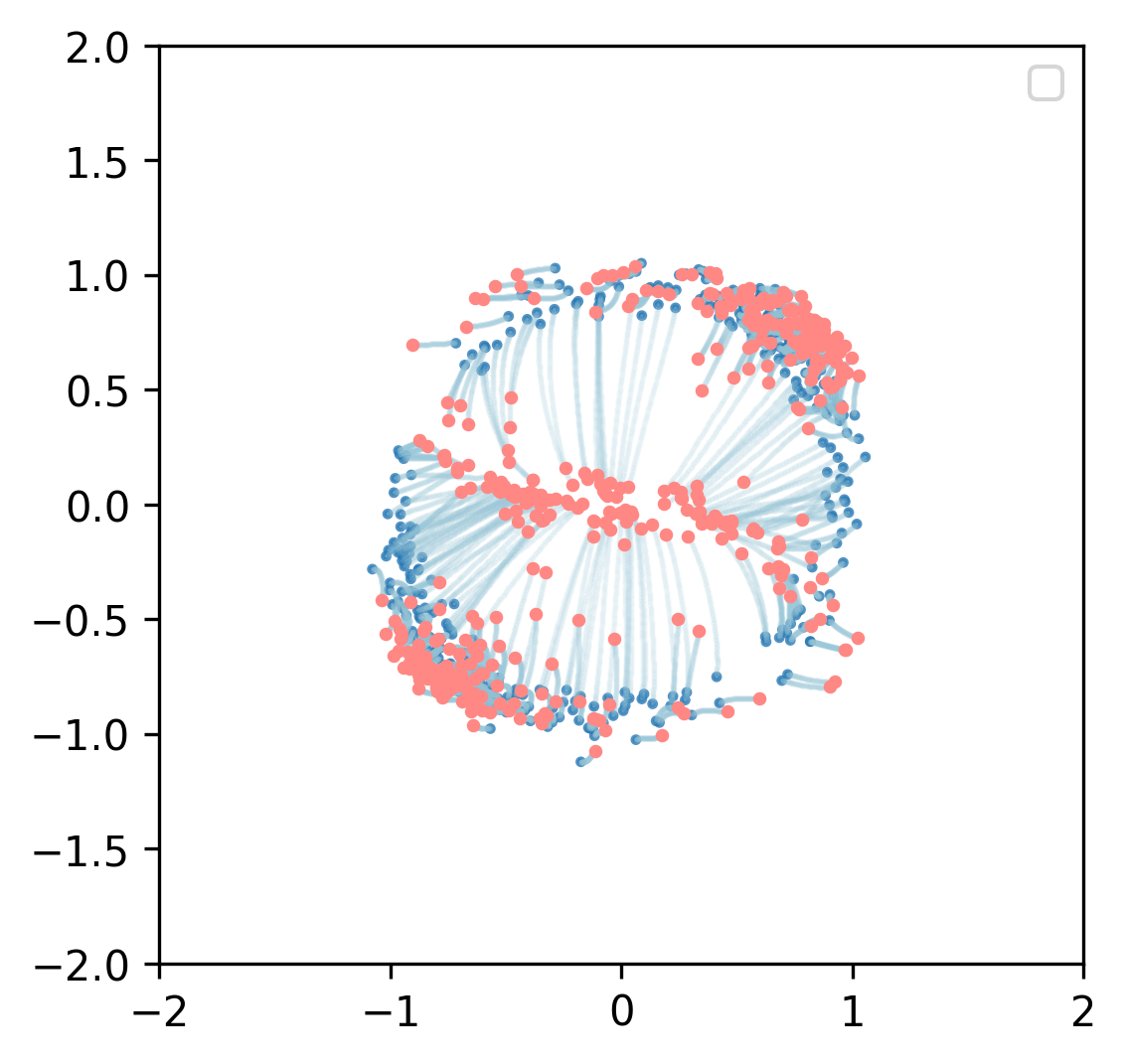} 
      & \includegraphics[width=0.15\linewidth,valign=t]{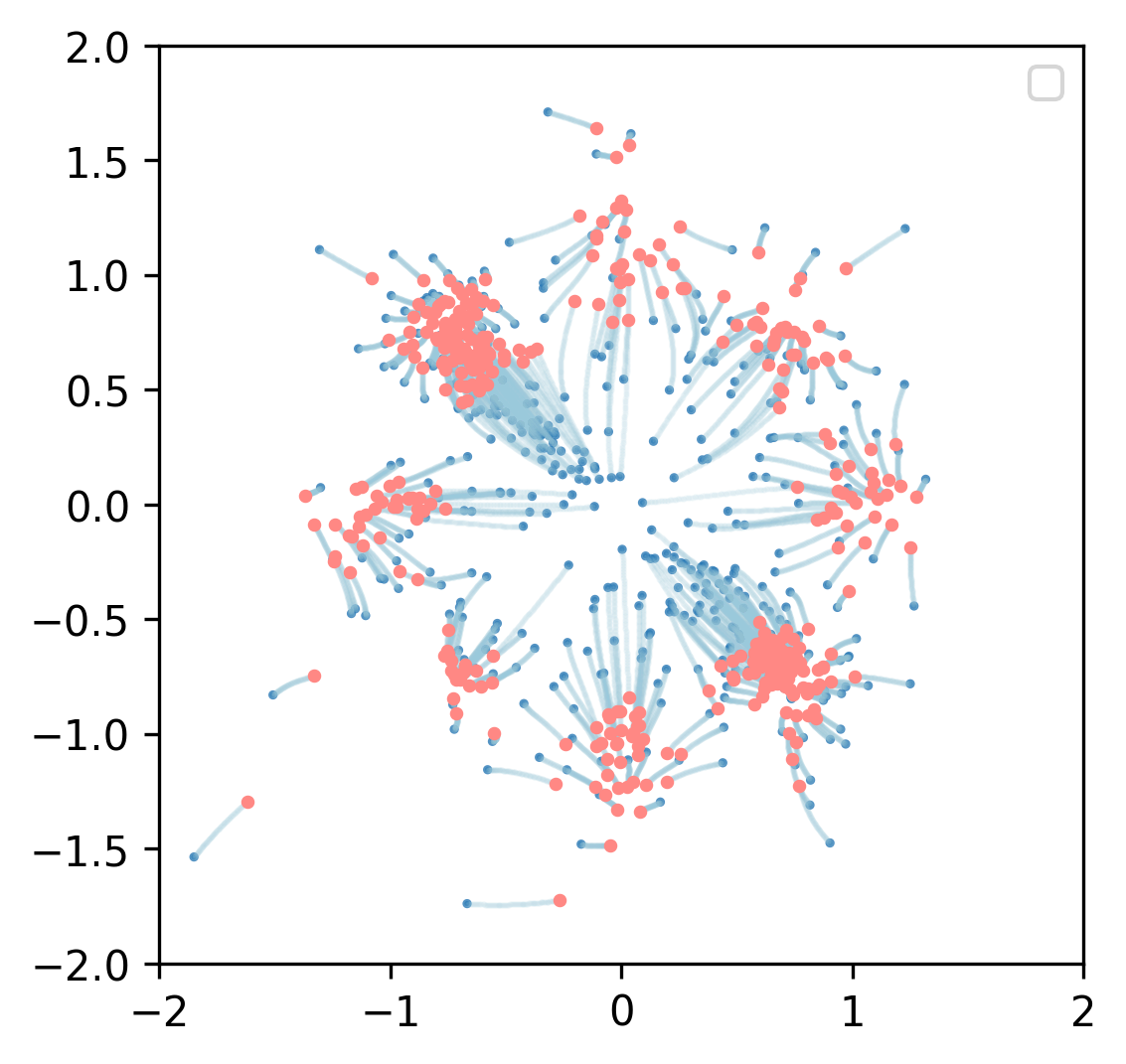} \\
    \textbf{SGFM-MALA}
      & \includegraphics[width=0.15\linewidth,valign=t]{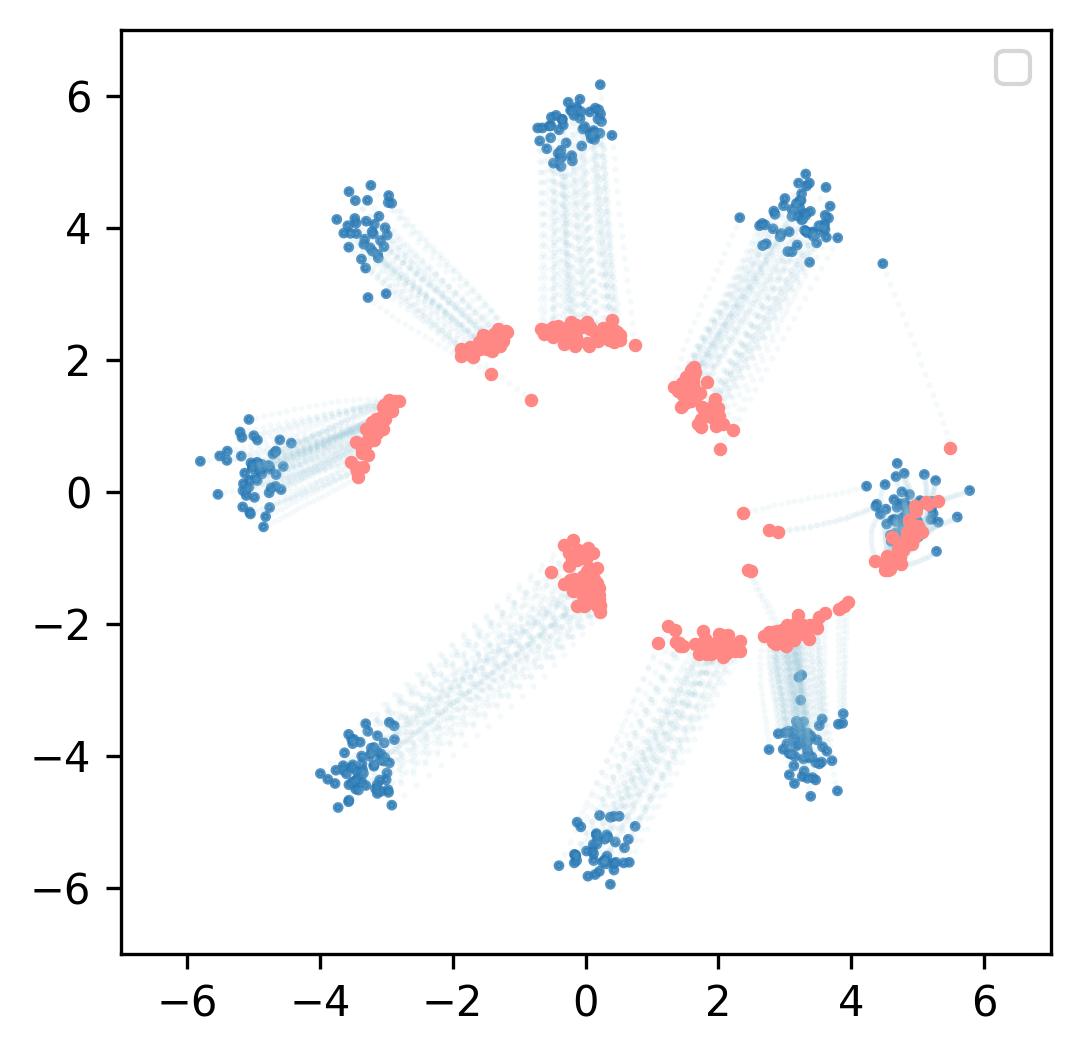}
      & \includegraphics[width=0.15\linewidth,valign=t]{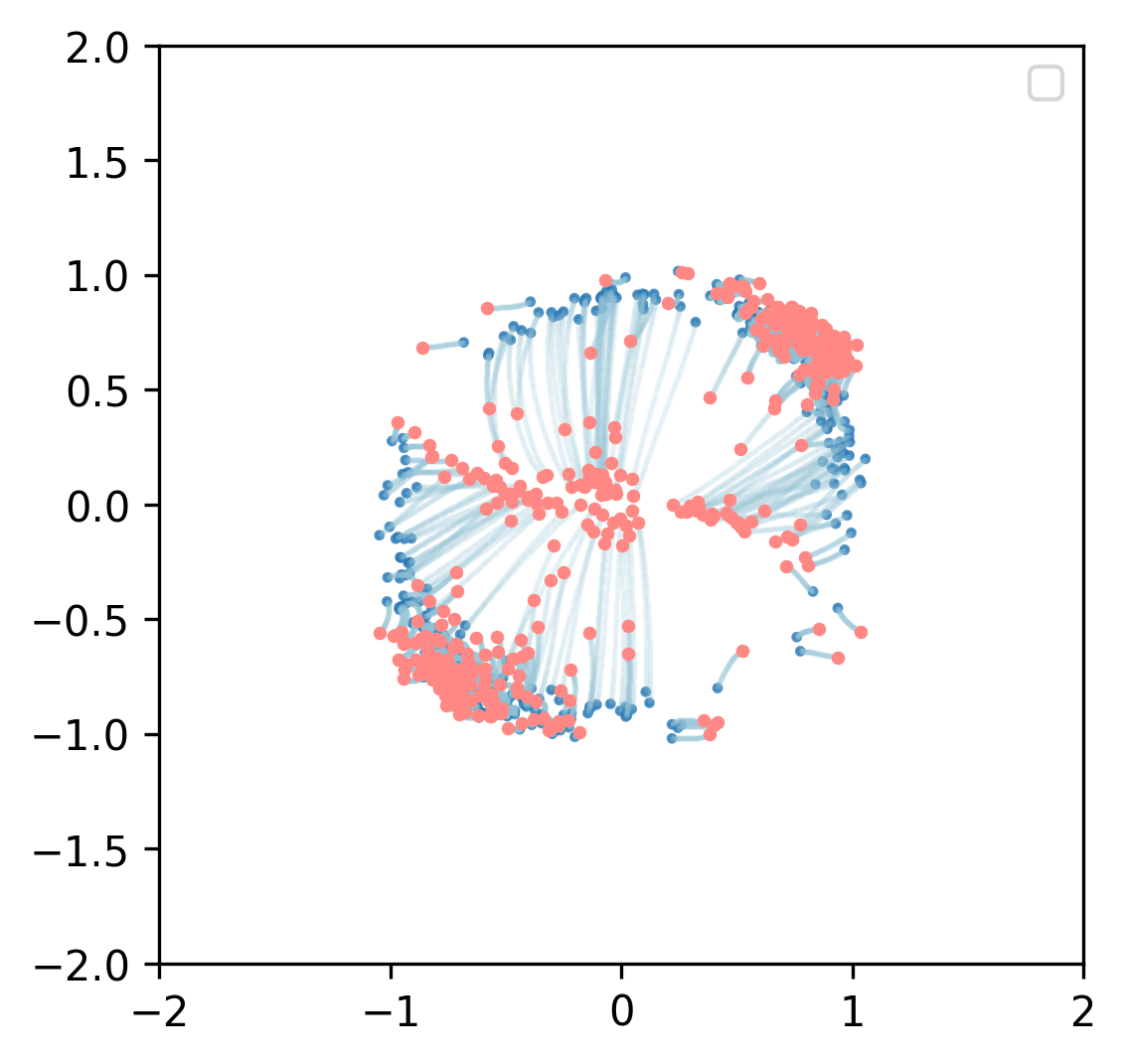} 
      & \includegraphics[width=0.15\linewidth,valign=t]{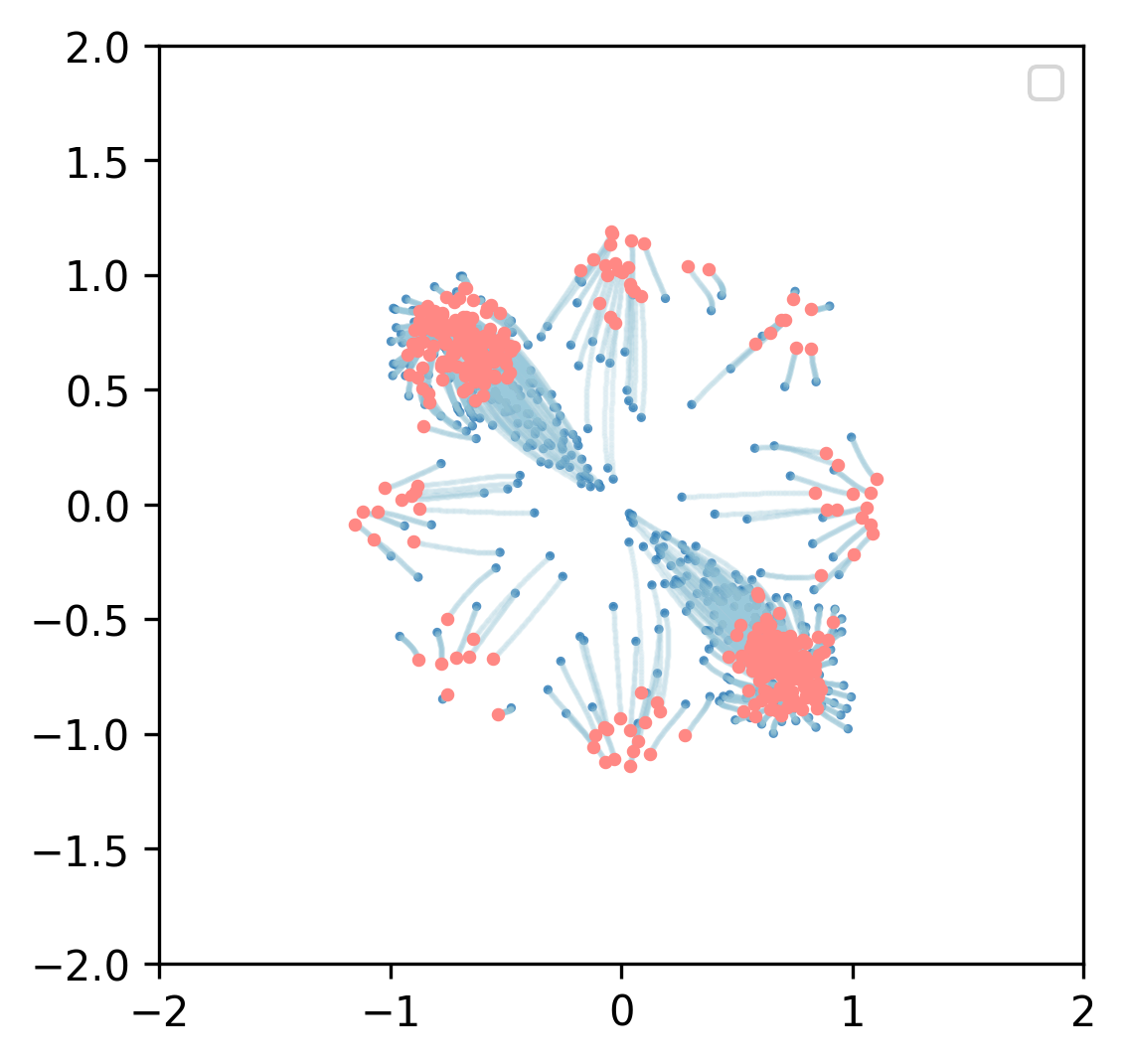} \\
    \textbf{SGFM-HMC}
      & \includegraphics[width=0.15\linewidth,valign=t]{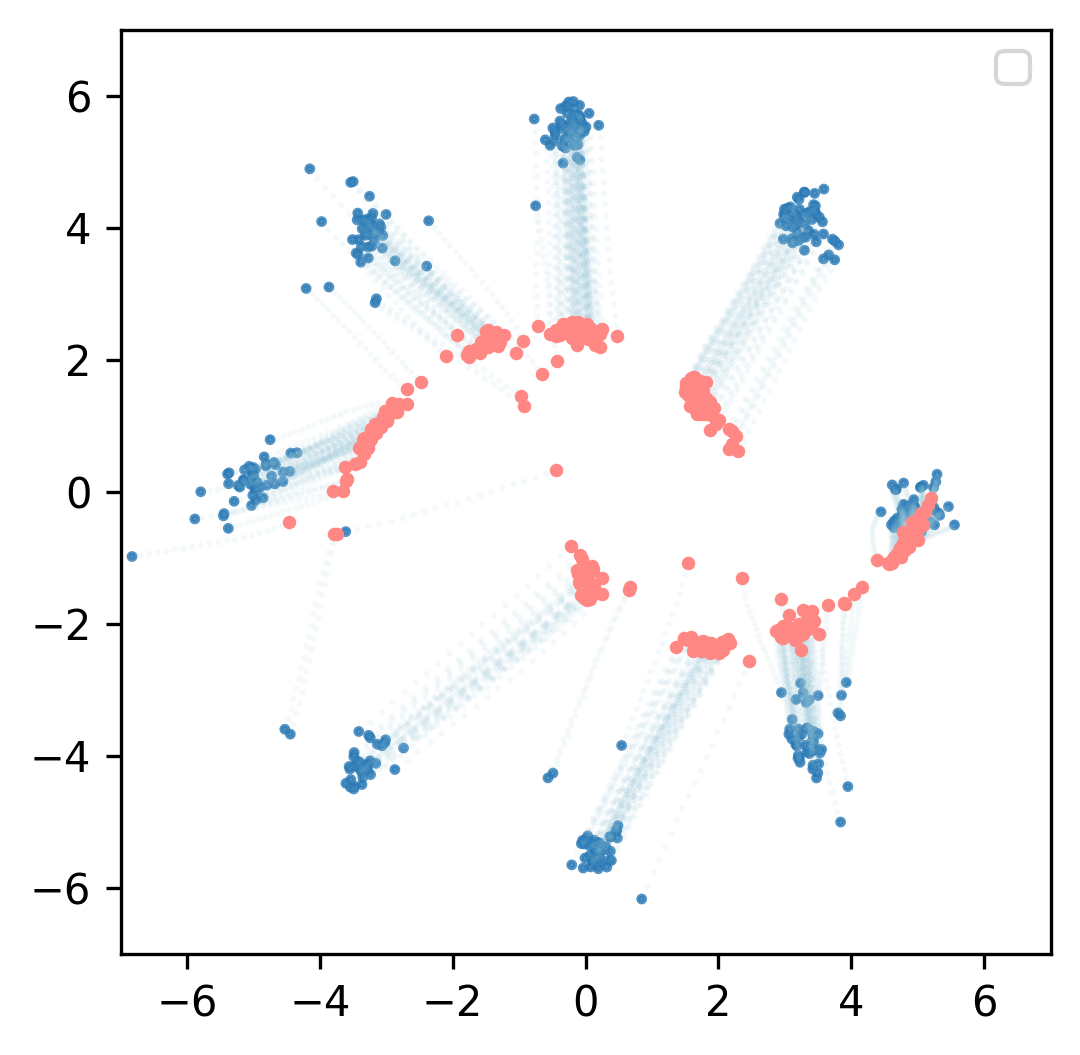} 
      & \includegraphics[width=0.15\linewidth,valign=t]{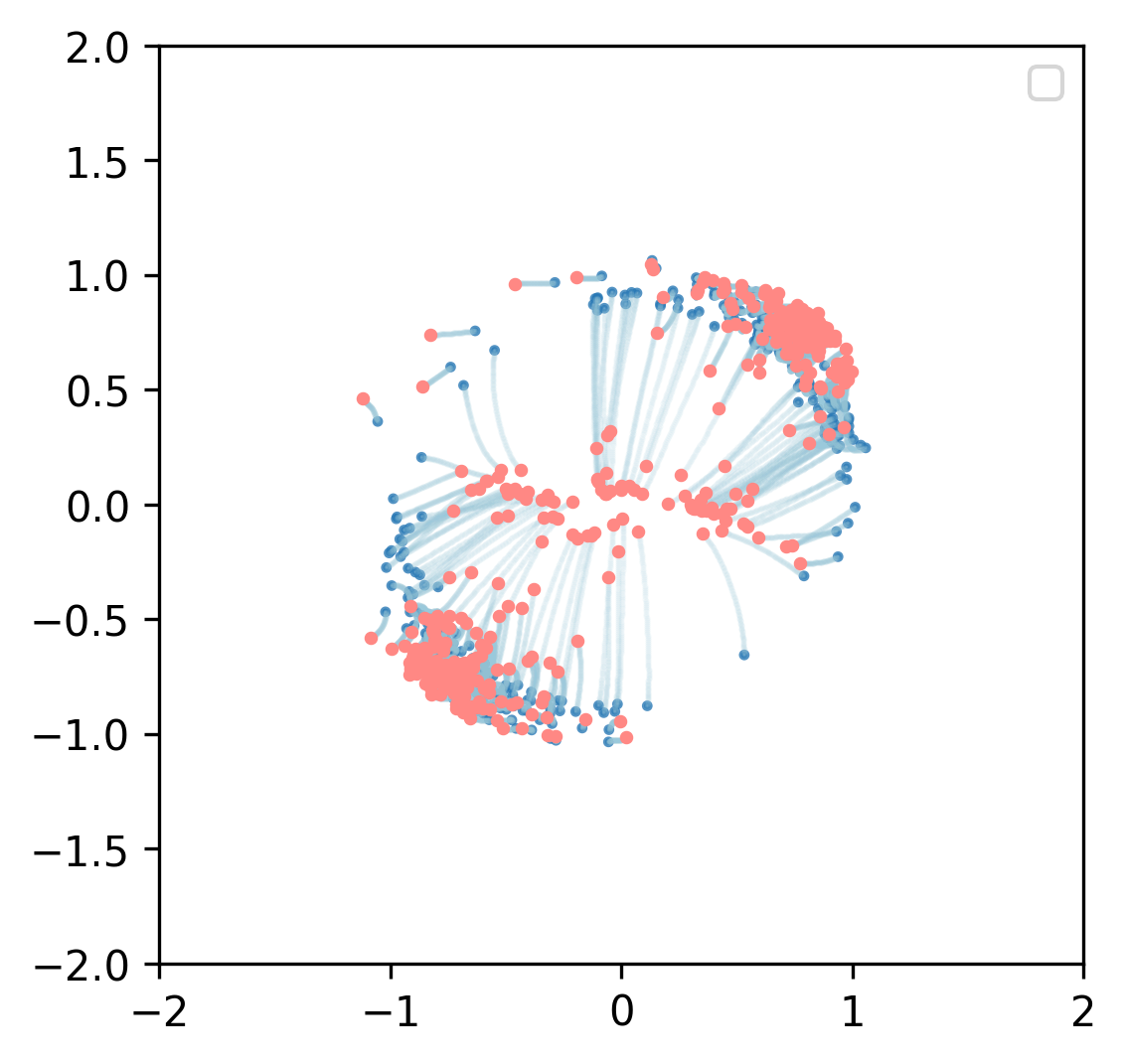} 
      & \includegraphics[width=0.15\linewidth,valign=t]{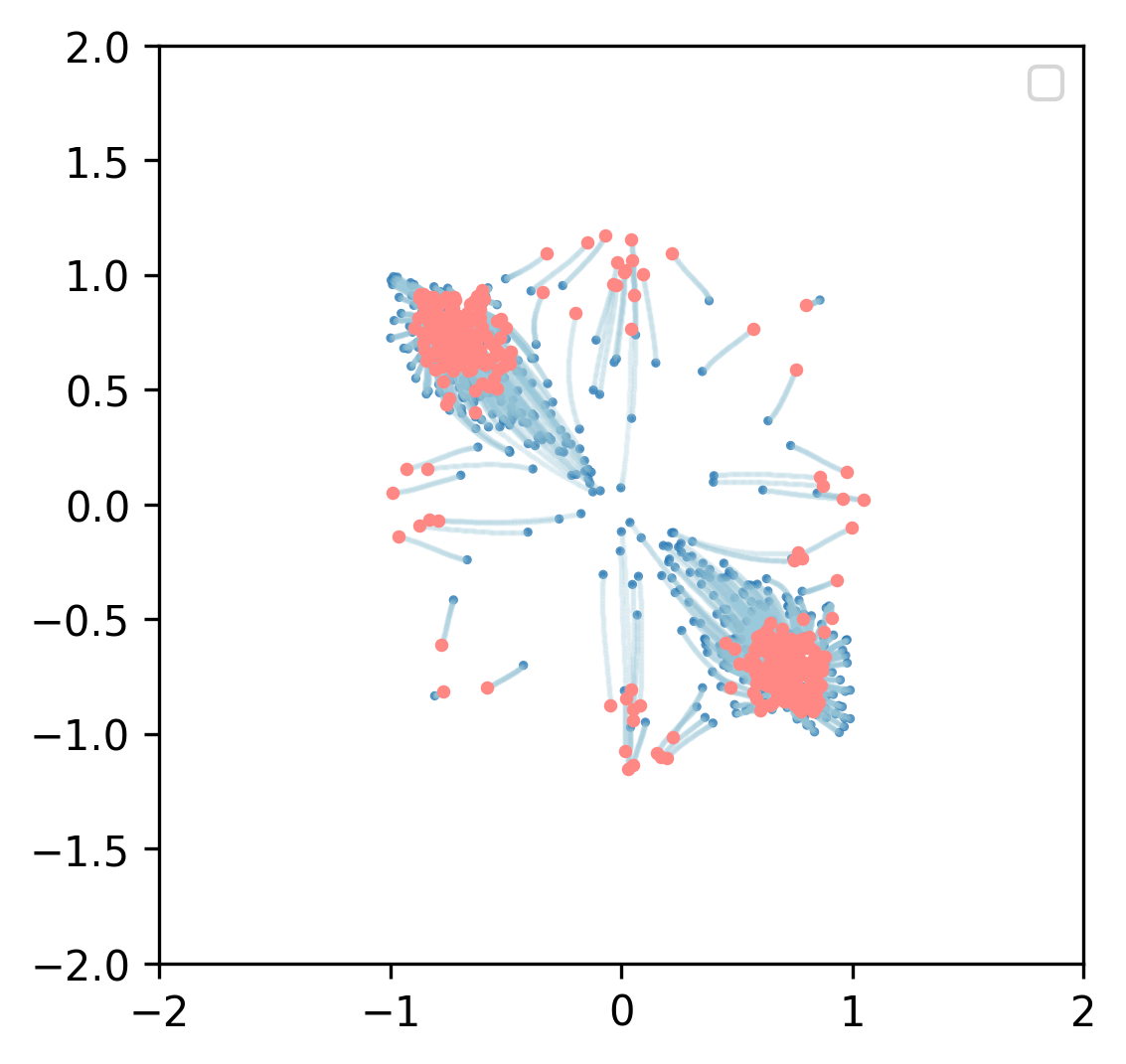} \\
  \end{tabular}
  \caption{Results of synthetic datasets with different source distributions (marked in blue) and target distributions (marked in red).}
  \label{fig:my-10row-comparison}
\end{figure}

\newpage 

We provide additional experiments that evaluate the sample quality with other pairs of source and target distributions. Since the distributions are simple and low-dimensional, we adopt IS as the sampling method and refer to our guidance method as SGFM-IS.
Figure~\ref{fig:2D:comp_moon} presents the generation results with an 8-Gaussian source distribution and a moon target distribution. 
We observe that SGFM-IS achieves superior sample quality across varying running times. However, both baseline methods perform poorly and more running time did not help.
The primary reason for failure is the multi-modal structure of this generation task, which makes samplers trapped in local optima, as illustrated in Figure~\ref{fig:my-10row-comparison}.
These results underscore the flexibility of our guidance framework, which allows for the tailored selection of advanced sampling strategies to suit different tasks.

\begin{figure}[ht]
\begin{center}
\centerline{\includegraphics[width=0.5\columnwidth]{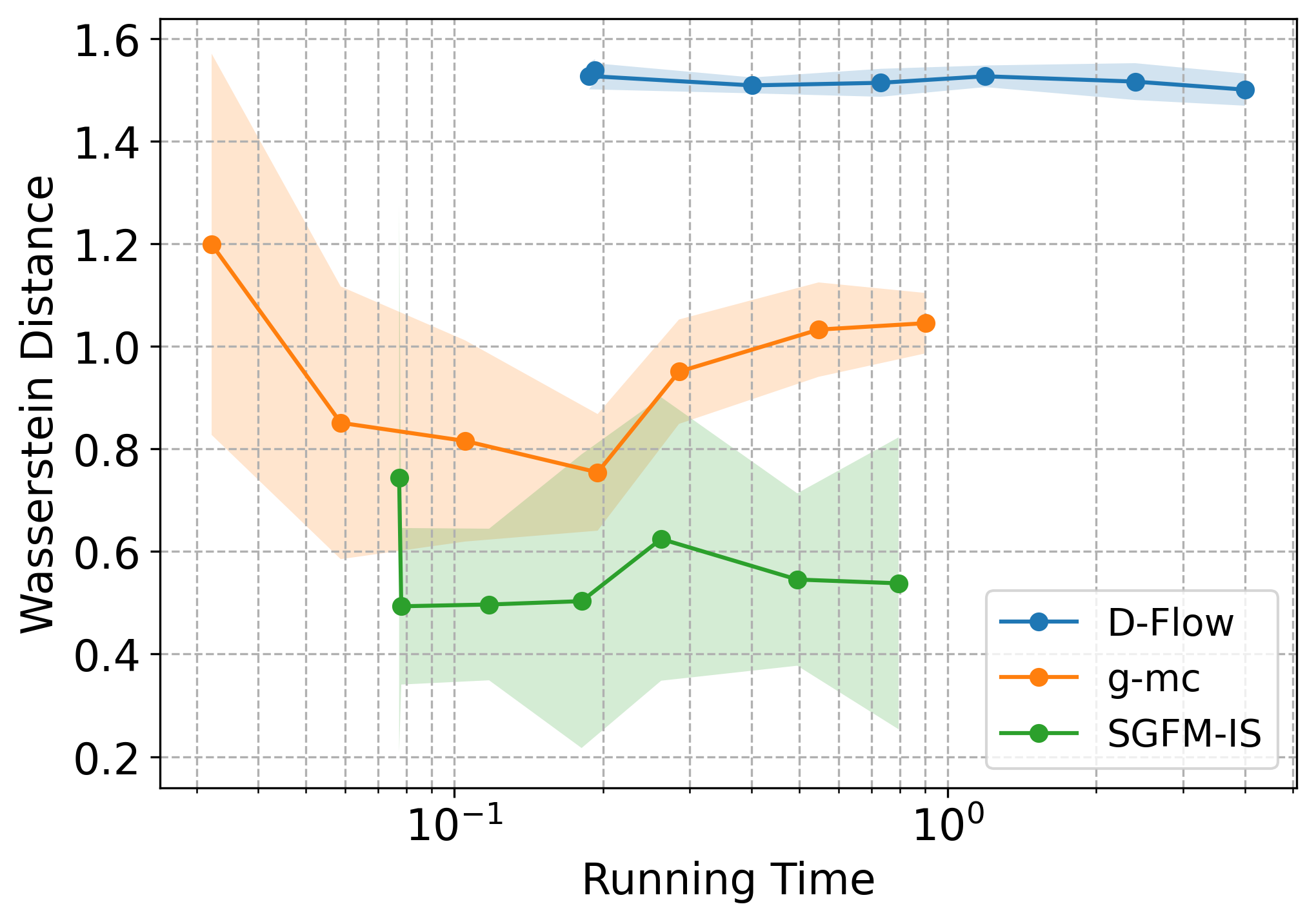}}
\caption{Comparison of sample quality and running time in 2D example, with an 8-gaussian source distribution and a moon target distribution.}
\label{fig:2D:comp_moon}
\end{center}
\end{figure}

Figure~\ref{fig:2D:comp_scurve} presents the generation results with a circle source distribution and an S-curve target distribution. 
We observe that D-Flow performs even worse with increasing running time.
As shown in Figure~\ref{fig:my-10row-comparison}, D-Flow tends to overly transport points to the line $x[1] = x[2]$, indicating that D-Flow overemphasizes minimizing the loss $J$ and loses sample diversity.

\begin{figure}[ht]
\begin{center}
\centerline{\includegraphics[width=0.5\columnwidth]{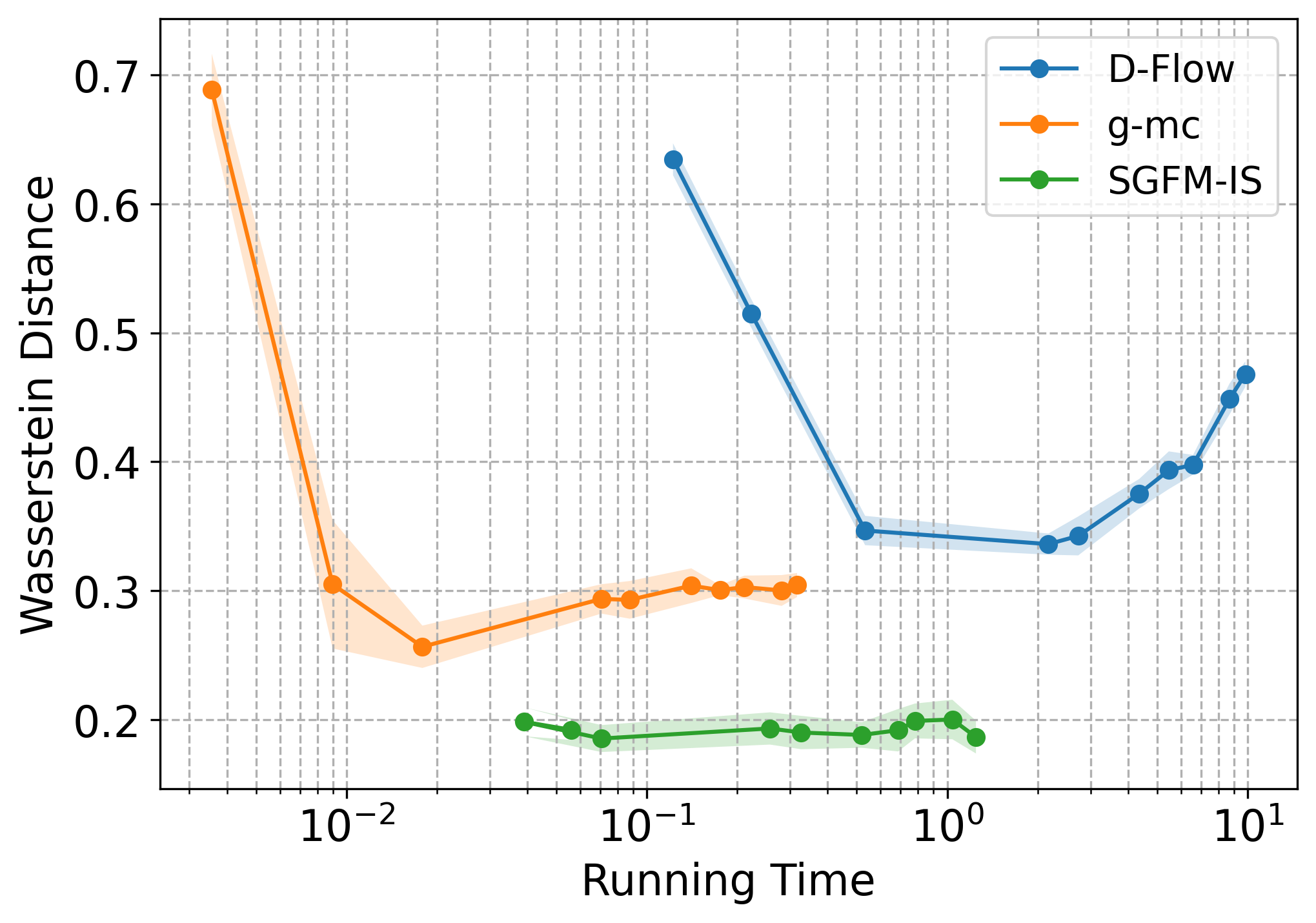}}
\caption{Comparison of sample quality and running time in 2D example, with a circle source distribution and an S-curve target distribution.}
\label{fig:2D:comp_scurve}
\end{center}
\end{figure}

\subsection{PDE solution operator}\label{sec:app:darcy}
In this section, we provide more details on the physics-informed inverse problem in Section \ref{sec:experiments:darcy}, including an outline of the Darcy flow equations, implementation details, and additional sample outcomes.
\subsubsection{Darcy flow equations}\label{sec:app:darcyeq}
Darcy flow is an elliptic PDE describing fluid flow through a porous medium,
\begin{equation}
\begin{aligned}
\label{eq:darcy}
    \mathbf{u}(\mathbf{x})&=-K(\mathbf{x})\nabla p(\mathbf{x}), &x\in \Omega\\
    \nabla\mathbf{u}(\mathbf{x})&=f(\mathbf{x}), &x\in \Omega\\
    \mathbf{u}(\mathbf{x})\cdot\hat{\mathbf{n}}(x)&=0,&x\in\partial\Omega\\
    \int_\Omega p(\mathbf{x})d\mathbf{x}&=0,&
\end{aligned}
\end{equation}
where $K$ is the permeability field, $f$ is a source function, and $p$ is the resulting pressure field. In alignment with \cite{bastek2024physics, jacobsen2025cocogen}, we consider the equations on a square domain $\Omega=[0,1]^2$ with resolution $64\times 64$ and let $f$ be a constant function. In this setting, a dataset of pairwise solutions $(K,p)$ is offered by \cite{bastek2024physics}, which is generated by translating \eqref{eq:darcy} to a linear system using finite difference approximations of the derivatives, and then solving this system. 
\subsubsection{Implementation details}\label{sec:app:darcyimp}
\paragraph{Flow matching model:} The vector field defining the flow-matching model is approximated using a U-Net architecture adopted from \cite{tong2023improving}. The source distribution is a standard Gaussian distribution. In addition to the flow matching objective, the loss is regularized by the physics-residual following \cite{bastek2024physics}. The residual is computed using $\hat{x}_1$ from \cite[Eq. 4]{feng2025guidance} as data-space estimate and we select $\Sigma_t=\frac{1-t}{t}$ and $c=10^{-2}$. The model is trained on $10^4$ samples for 200 epochs using the Adam optimizer with an initial learning rate $\eta=10^{-4}$ which decays exponentially with a factor $\gamma=0.99$. 

\paragraph{Conditional sampling:} All methods are initialized by the same set of samples from the unmodified source distribution. To balance the scale of the cost function $J$ and the prior probability $\log{q_0}$, the cost is scaled by a factor $\frac{1}{\lambda}$ where $\lambda=10^{-3}$. To simulate a setting where true solutions are unavailable, all methods were tuned before observing the validity scores of the outcomes.
\paragraph{Implementation of SGFM-HMC:} SGFM-HMC is implemented by running the HMC algorithm for $N_{HMC}=100$ steps with $L=3$ leapfrog steps, where the step size is randomly selected in each Markov chain iteration as $\epsilon= 5\times (10^{-4} + \zeta\times 10^{-3})$ where $\zeta\sim \chi^2(2)$ with $\chi^2(2)$ being the chi-squared distribution with two degrees of freedom. We found that this setting gives good acceptance ratios while allowing for a significant number of HMC iterations to be performed without having too long runtimes. The transport map is obtained by integrating the neural ODE associated with the vector field for two steps using the Dormand-Prince (Dopri5) method. We use the same transportation map both for the density computation in the HMC iterations and to map the sampled source point to the target space.
\paragraph{Implementation of SGFM-OPT and SGFM-OPT $\chi^2$:} SGFM-OPT and SGFM-OPT $\chi^2$ are implemented using L-BFGS optimization with learning rate $\eta=1$, maximum iterations of 20, and history size 100. The method is allowed to run for the same amount of runtime as HMC (which corresponds to approximately 15 optimization steps), but usually converges before that. The transport map is designed as in SGFM-HMC.
\paragraph{Implementation of $g^{\text{cov-A}}$:} $g^{\text{cov-A}}$ \cite{feng2025guidance} is implemented using a linear schedule $\lambda_t^{\text{covA}}=10\times \lambda$. We found that $\lambda_t^{\text{covA}}=\lambda$ was not sufficient to observe a significant change in the guidance cost, while this choice achieves the lowest guidance cost of all methods. The ODE is integrated for three steps with the Dopri5 method, additional steps had no effect on performance.

\subsubsection{Additional samples}\label{sec:app:darcy_multioutput}
To extend the results in Figure \ref{fig:darcy}, we present further outcomes of the conditionally sampled permeability field and the associated solutions of the pressure field in Figure \ref{fig:darcy_multioutput}.
\begin{figure}
    \centering    
    \subfigure[SGFM-HMC]{
    \includegraphics[scale=0.25]{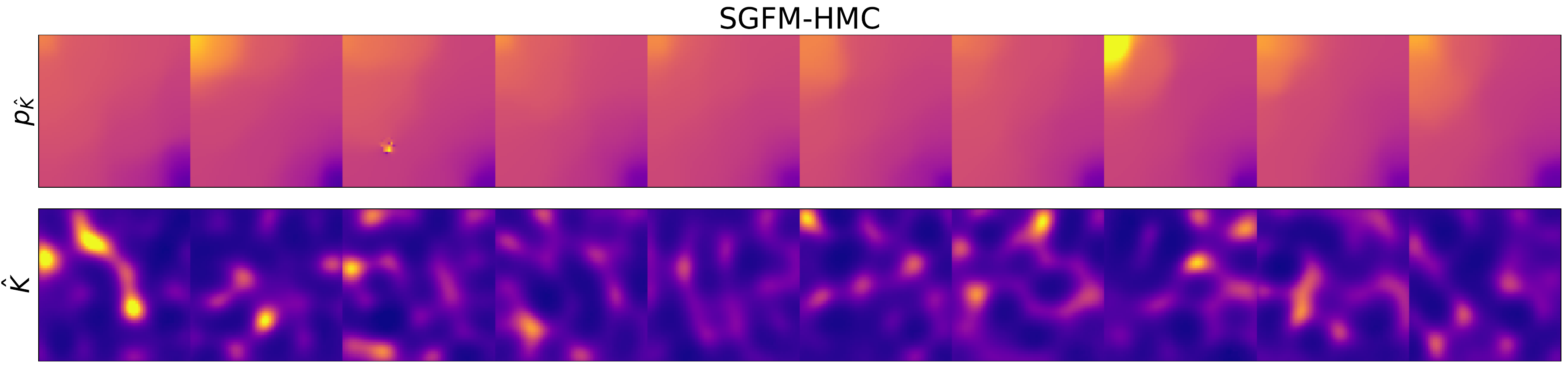}}
    \subfigure[SGFM-OPT]{
    \includegraphics[scale=0.25]{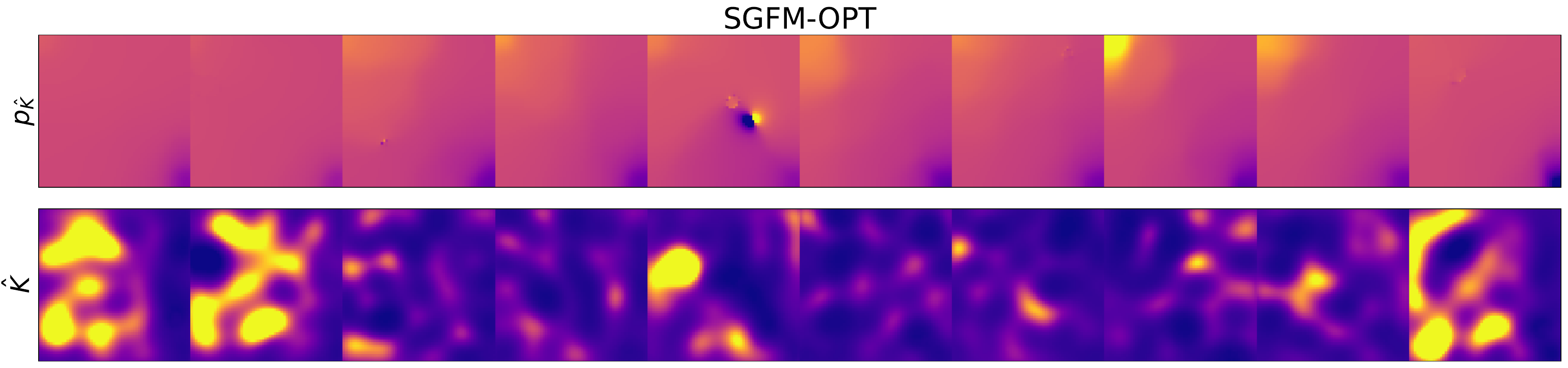}}
    \subfigure[SGFM-OPT with $\chi^2$ trick (D-Flow)]{
    \includegraphics[scale=0.25]{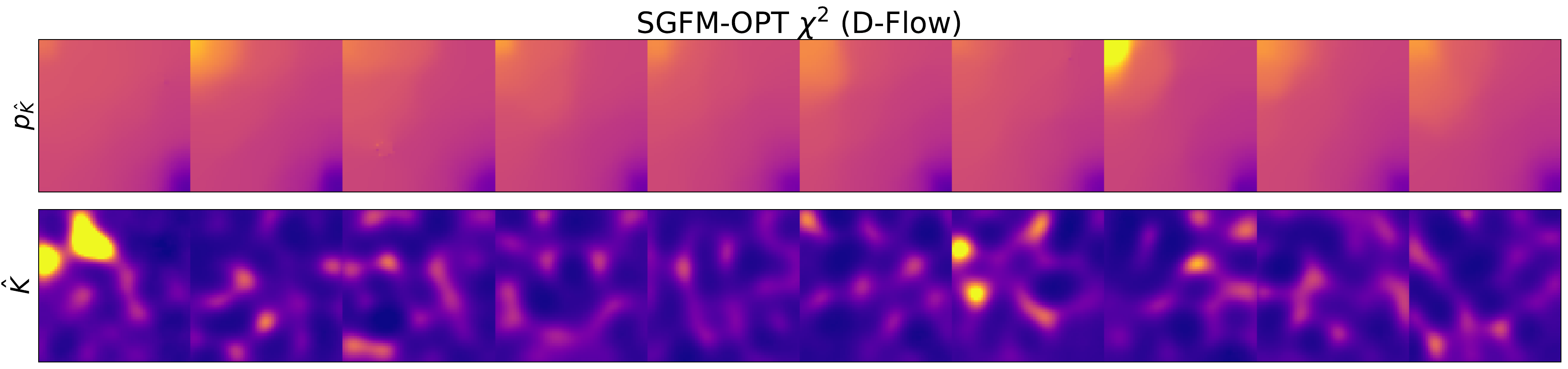}}
    \subfigure[$g^{\text{cov-A}}$]{
    \includegraphics[scale=0.25]{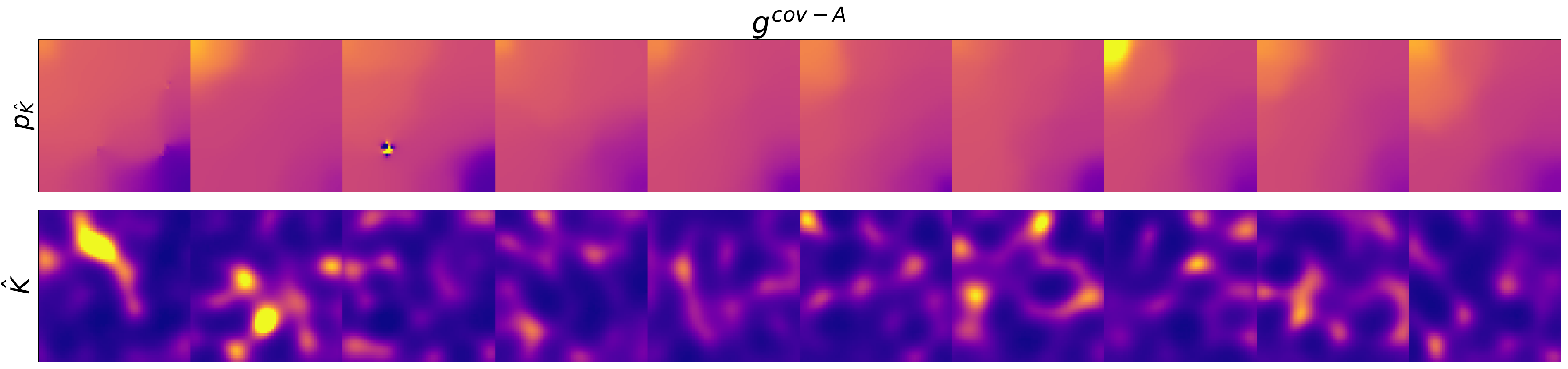}}
    \caption{Solutions to the inverse problem of the Darcy flow equations - additional outcomes following the target in Figure \ref{fig:darcy}. Top: true solution $p_{\hat{K}}$ corresponding to inverse estimate $\hat{K}$; bottom: inverse estimate $\hat{K}$ of the permeability field generated by conditional sampling.}
    \label{fig:darcy_multioutput}
\end{figure}

\subsection{Image inverse problem on CelebA}\label{sec:app:celeba}
\subsubsection{Implementation details}

The regularizers or constraints used in different variants of optimization-based sampling are summarized in Table~\ref{tab:regularizer}. For regularization-based methods, we introduce a weighting coefficient and tune it to achieve optimal performance. For the constraint-based method, we project the solution onto the hyperspherical shell after each update.

\begin{table}[h]
\centering
\caption{Regularizer or constraint for variants of optimization-based sampling.}
\begin{tabular}{ll}
\hline
Method & Regularizer or constraint \\
\hline
SGFM-OPT-1 \eqref{eq:opt_problem1:our}  &  $R_1(x_0) = \left\|x_0\right\|^2 $\\
SGFM-OPT-2 \eqref{eq:opt_problem2:our} &  $R_2(x_0)=- \ln p_{\chi{X}^2}(\|x_0\|^2)=-(d-2) \log\left\| x_0\right\| + \frac{\left\|x_0 \right\|^2}{2}$\\
SGFM-OPT-3 \eqref{eq:opt_problem3:our} &  $R_3(x_0)=( \left\|x_0\right\|^2 - d)^2$\\
SGFM-OPT-4 \eqref{eq:opt_problem3:our} &  $R_4(x_0)=\big|  \left\|x_0\right\|^2 - d \big|$\\
SGFM-OPT-5 \eqref{eq:opt_problem3:our} &  $R_5(x_0)=(\left\|x_0\right\| - \sqrt{d})^2$\\
SGFM-OPT-6 \eqref{eq:opt_problem3:our} &  Constraint: $|\|x_0\|^2-d|\leq \sqrt{2d}$\\
\hline
\end{tabular}\label{tab:regularizer}
\end{table}

\clearpage
\subsubsection{Generated CelebA samples}

\begin{figure}[ht]
\begin{center}
\centerline{\includegraphics[width=0.61\columnwidth]{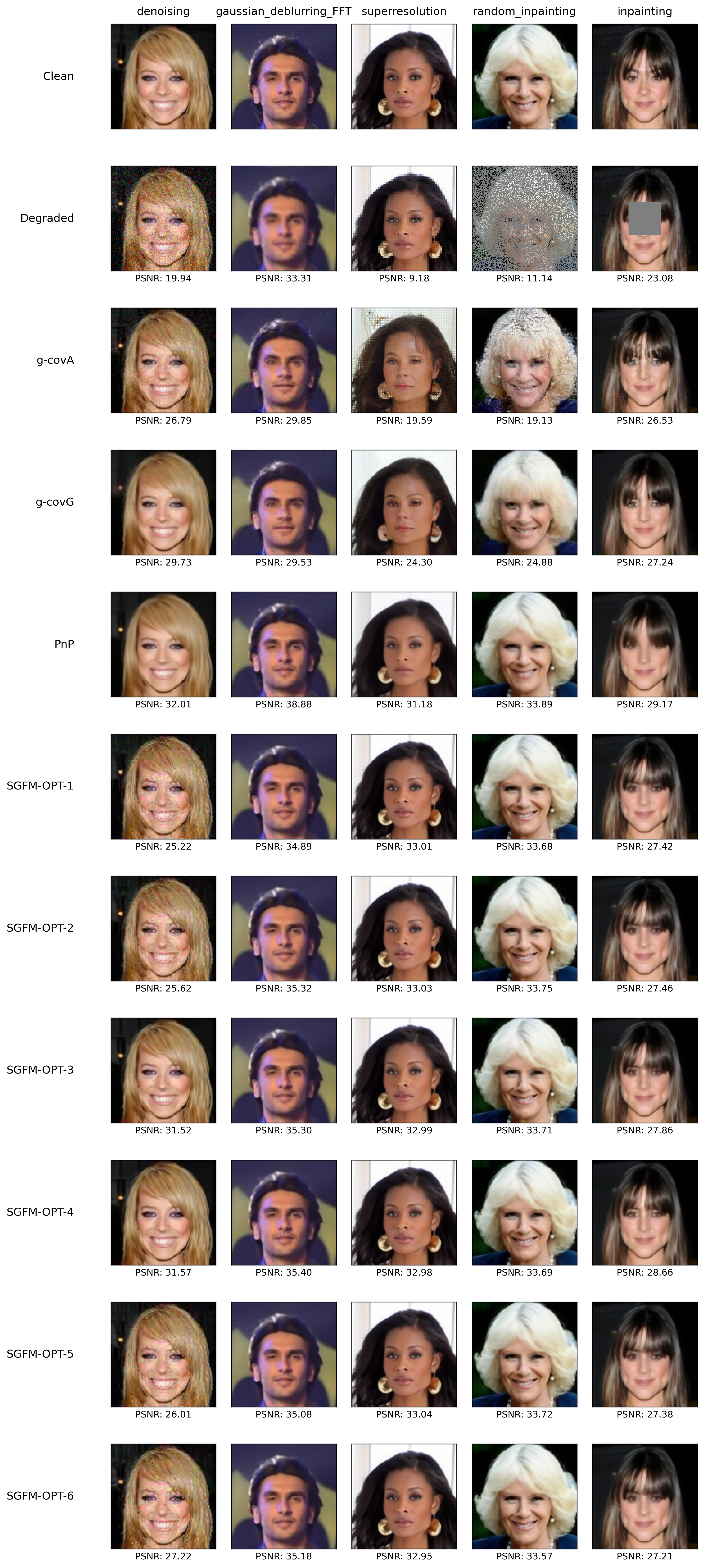}}
\caption{Comparison of image restoration methods on CelebA.}
\label{fig:celeba}
\end{center}
\end{figure}

\clearpage

\subsection{MNIST image generation}\label{sec:app:mnist}
\subsubsection{Conditional generation on MNIST}

\begin{wraptable}{r}{0.4\columnwidth}
\vspace{-0.5cm}
\centering
\begin{tabular}{|l|c|c|}
\hline
\textbf{Method} & SGFM-ULA  & $g^{\text{cov-A}}$ \\
\hline
\textbf{Accuracy} & 87.6\% & \textbf{98.5\%}  \\
\hline
\textbf{FID} & \textbf{46.7} & 57.1  \\
\hline
\end{tabular}
\caption{The label accuracy (higher is better) and FID (lower is better).}
\label{tab:mnist}
\end{wraptable}
We perform conditional image generation experiments on the MNIST dataset, where the generated samples are conditioned on provided labels. Given a target label, the loss function $J$ corresponds to the negative log-likelihood of the label computed via a classifier. We select ULA as our sampling method and benchmark its performance against the baseline method $g^{\text{cov-A}}$ in \cite{feng2025guidance}.  

Performance is evaluated using the Fréchet Inception Distance (FID) and label accuracy. A separate classifier determines accuracy to avoid overconfidence. 
The experimental results, detailed in Table~\ref{tab:mnist}, indicate that while our proposed method yields relatively lower label accuracy, it achieves a superior FID score. 
Figure~\ref{fig:diversity_mnist} presents illustrative examples of generated images from both methods. We observe that although $g^{\text{cov-A}}$ consistently generates images corresponding to the correct digits, the generated samples exhibit limited diversity, characterized by uniformly thick strokes and similar visual styles. In contrast, the ground-truth MNIST distribution inherently comprises digits exhibiting diverse shapes, styles, and stroke widths. This discrepancy is captured by the significant covariance mismatch between the baseline-generated distribution and the real data distribution, resulting in a higher FID for the baseline. Our approach thus demonstrates improved sample diversity, reflected by the lower FID score.

\begin{figure}[ht!]
    \centering
    \subfigure[Ground truth]{
        \includegraphics[width=0.32\textwidth]{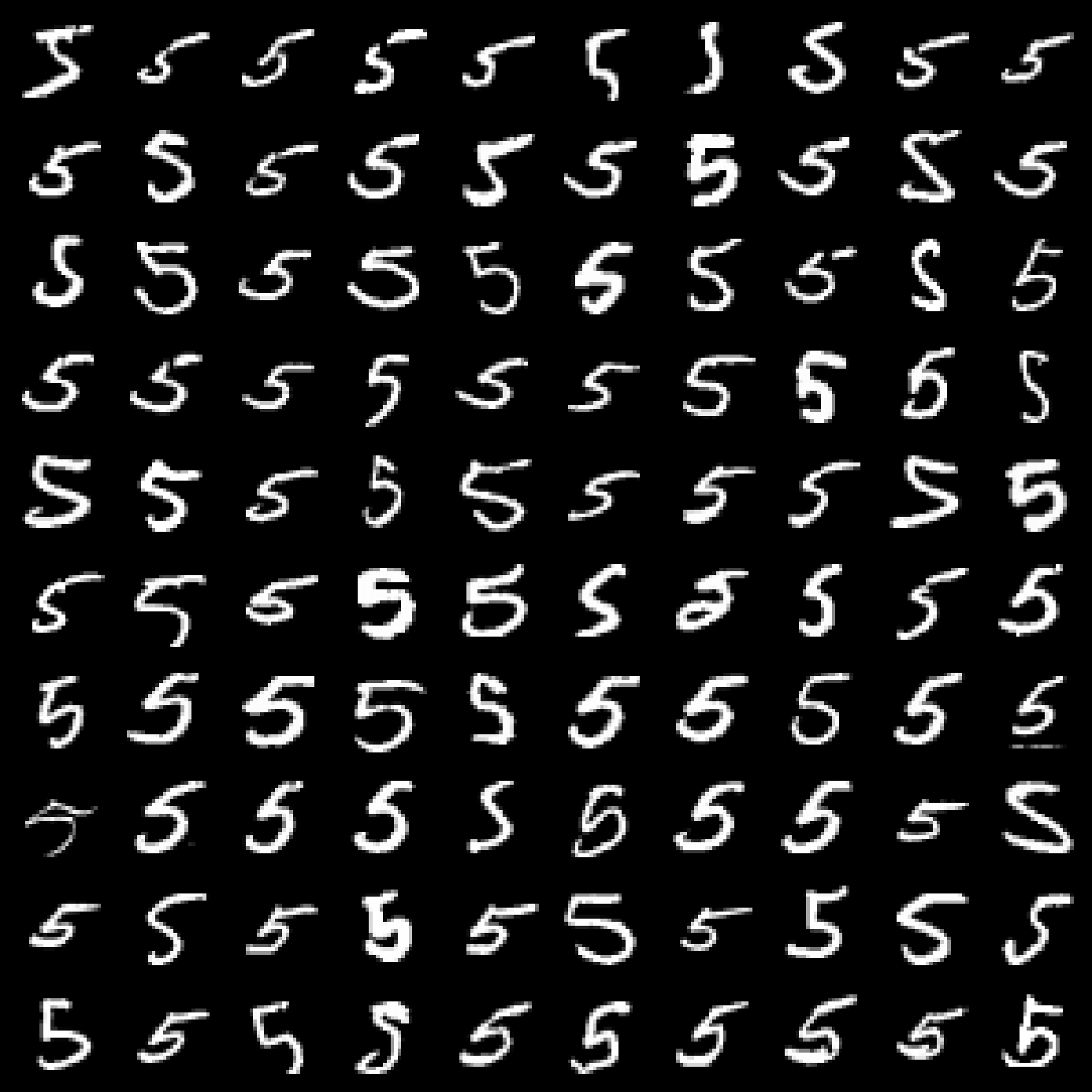}}
    \subfigure[$g^{\text{cov-A}}$]{
        \includegraphics[width=0.32\textwidth]{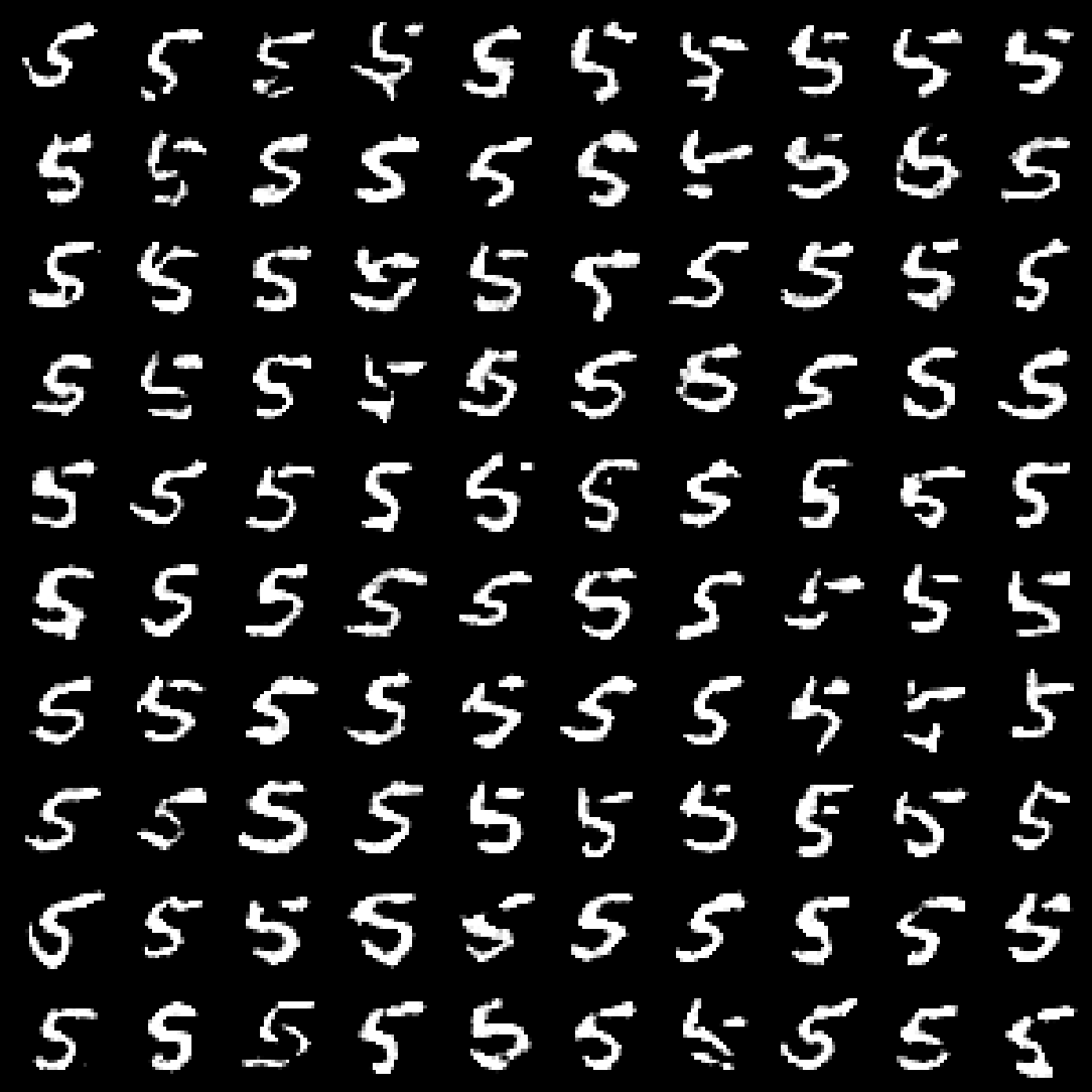}}
    \subfigure[SGFM-ULA]{
        \includegraphics[width=0.32\textwidth]{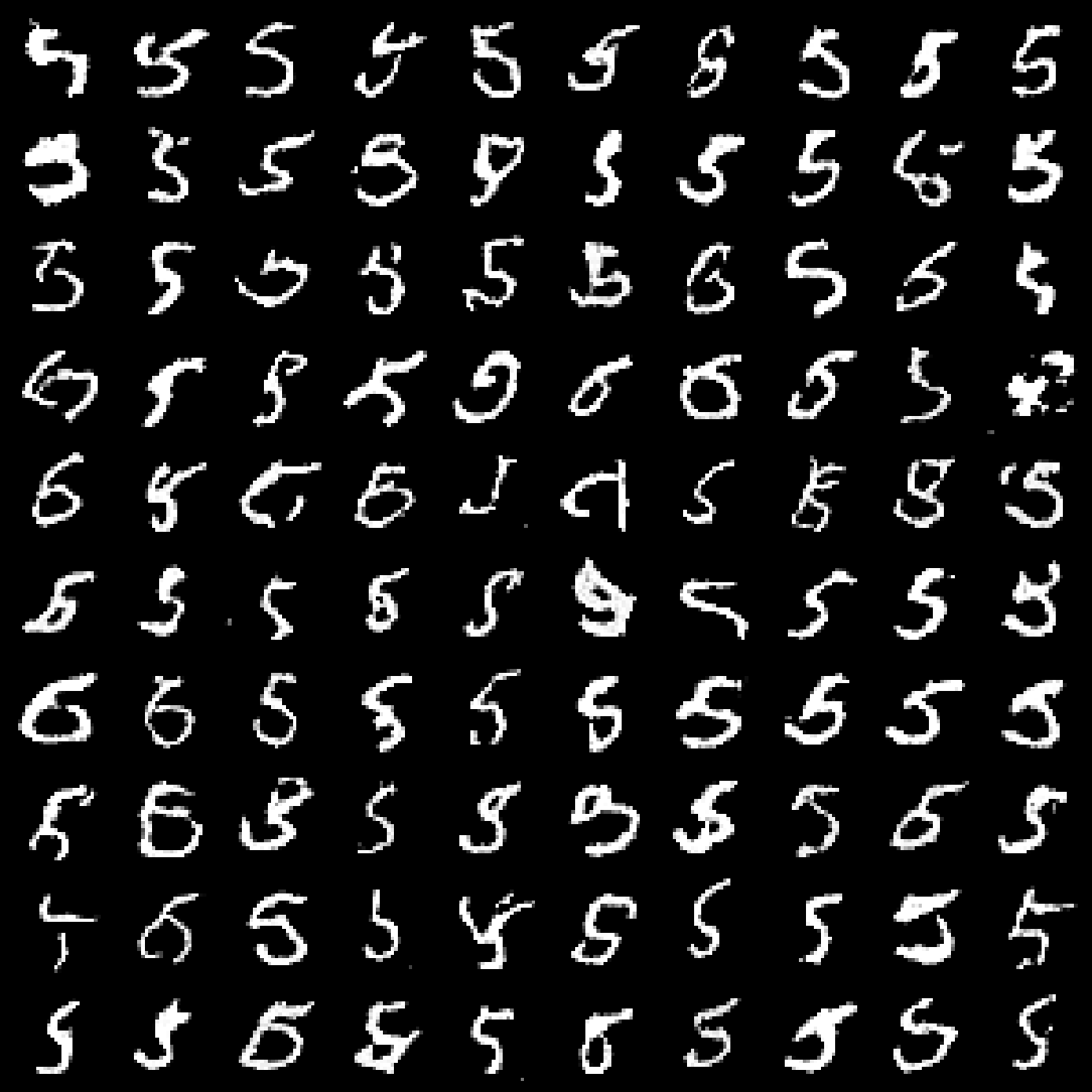}}
    \caption{MNIST sample generation conditioned on digit 5. }
    \label{fig:diversity_mnist}
\end{figure}
\subsubsection{Implementation details}

\paragraph{Pre-trained models:} For the classifier used for guidance, we train a convolutional neural network classifier on MNIST, which achieves an accuracy of 96.8\% on the standard test set. For the classifier used for evaluating the accuracy of generated samples, we adopt an independent pre-trained Vision Transformer classifier\footnote{https://github.com/sssingh/hand-written-digit-classification/tree/master}, which achieves higher robustness with an accuracy of 98.7\% on the testing distribution. Following \cite{tong2023improving}, the vector field model used in our experiments is trained using a U-Net architecture initialized from a Gaussian distribution. It was trained for three epochs, each consisting of 468 iterations.

\paragraph{Conditional generation:}  The objective of this task is to generate images conditioned on a specified label and stay close to the original dataset. The guidance for this conditional generation utilizes a loss function $J$ defined as the negative log-probability of the targeted label $i$:
$$J(x) =-\log {\rm{softmax}}(h(x))_i,$$ 
where $h(x)$ represents the logits returned by a pre-trained classifier. To balance the scale of the cost function and the probability density function, we scale the loss function $J$ by $\frac{1}{\lambda}$ with $\lambda= 10^{-3}$.

\paragraph{Implementation of the method in \cite{feng2025guidance}:} We select $g^{\text{cov-A}}$ as the baseline method in \cite{feng2025guidance}, which shows superior performance in image problems. We use a constant schedule $\lambda_t^{\text{cov-A}} = \lambda^{\text{cov-A}}$. The ODE is integrated for 100 steps using the Dopri5 method.

\paragraph{Implementation of SGFM-ULA:}
SGFM-ULA is implemented by running the ULA algorithm over a maximum of 150 steps with a batch size of 16. The step size in each iteration is selected as $5 \times 10^{-4} \times \zeta$, where $\zeta\sim \chi^2(2)$ with $\chi^2(2)$ being the chi-squared distribution with two degrees of freedom.
Each batch takes about 218 seconds to process.

\paragraph{Evaluation metric:} We assess the quality of generated images using the Fréchet Inception Distance (FID), which measures the similarity between two distributions based on their means and covariances. The reference images used in the FID calculation are subsets of the MNIST training dataset corresponding to each target label, and evaluations are performed using 400 samples.



\clearpage
\subsection{ODE solution operator}\label{sec:app:ode}
In this section, we present another example of conditional generation in the context of physics-informed generative modeling. Instead of the Darcy flow equations \eqref{eq:darcy}, we consider an ODE with truncated solution trajectories. Compared to the problem in Section \ref{sec:experiments:darcy}, this example has lower dimension which allows us to further explore how well the SGFM framework and benchmark methods approximate the target distribution. 
\subsubsection{Conditional generation of ODE trajectories via flow matching}
Consider the ODE
\begin{equation}\label{eq:ode}
    \dot{x}(t) = -\theta_ax(t) + \theta_b\sin(\theta_\omega t),\quad x(0)=0.\\
\end{equation}
where a flow matching model is trained to sample from the joint distribution of the ODE parameters $[\theta_a, \theta_b, \theta_\omega]\triangleq\theta \sim \mathcal{U}(1,3)^3 $ and the set of corresponding discretized solutions \( x_\theta \in \mathbb{R}^{100} \). The source distribution is a standard Gaussian distribution for the parameters $\theta$ and a Gaussian Process with zero mean and squared exponential kernel for $x_\theta$ to encourage smooth solutions.

The conditional sampling problem is to generate solution trajectories consistent with a partial observation of the ODE parameters $[\theta_a^*, \theta_b^*, \cdot]$, which defines a family of admissible solutions $\{x_{\theta_a^*, \theta_b^*, \cdot}\}_{\theta_\omega\in\lbrack 1,3 \rbrack}$. The cost function is the reconstruction error of the target parameters, $J(\theta, x_\theta) = \|\theta_a-\theta_a^*\|^2+\|\theta_b-\theta_b^*\|^2$. This corresponds to a soft constraint on $\theta_a$ and $\theta_b$, since the marginal target distributions for $\theta_a$ and $\theta_b$ then behave like posteriors with uniform priors and Gaussian likelihoods.
Figure \ref{fig:ode_ab_out} shows samples from the unconditional model and conditional samples using SGFM-HMC, SGFM-OPT (D-Flow)\footnote{In this example, the source distribution is not directly compatible with \eqref{eq:opt_problem2:our}, hence SGFM-OPT $\chi^2$ is omitted. However, \eqref{eq:opt_problem:our} remains applicable, and the corresponding method from D-Flow coincides with SGFM-OPT.} \eqref{eq:opt_problem:our} and $g^{\text{cov-A}}$.
\begin{figure}[ht!]
    \centering
    \includegraphics[width=\linewidth]{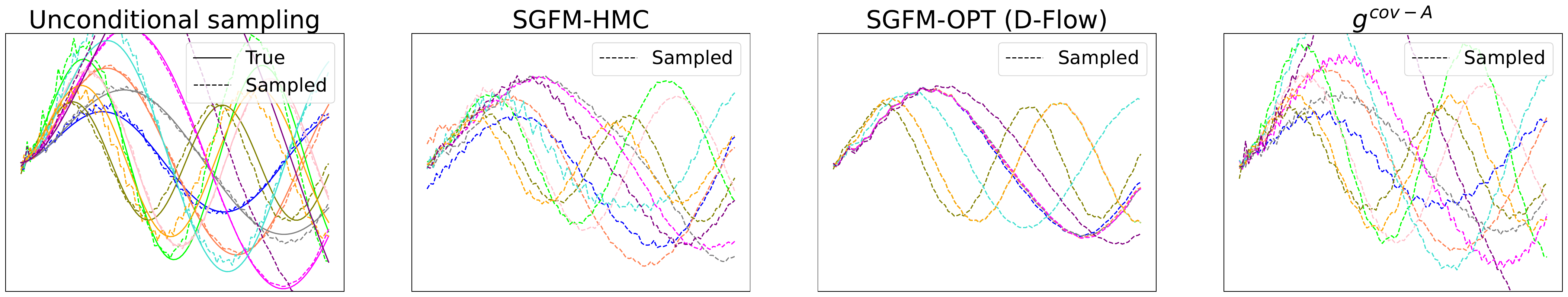}
    \caption{Solutions to the forward ODE problem. Samples $\lbrack\theta, x_\theta\rbrack \in \mathbb{R}^{103}$ consist of ODE parameters $\theta \sim \mathcal{U}(1,3)^3 $ and corresponding solution trajectories $ x_\theta \in \mathbb{R}^{100} $. The conditioning set consists of a partial observation of the ODE parameters, $\lbrack\theta_a^*, \theta_b^*, \cdot\rbrack$, which yields a family of admissible solutions $\{x_{\theta_a^*, \theta_b^*, \cdot}\}_{\theta_\omega\in\lbrack 1,3 \rbrack}$. 
    }
   \label{fig:ode_ab_out}
\end{figure}

We evaluate the methods by generating $10^3$ samples and assessing both their physical consistency and how well the empirical distribution approximates the target distribution. The latter is evaluated by comparing the parameter outcomes to equal-tailed credible intervals derived from the target distribution, which are obtained by MCMC simulation. The results in Figure~\ref{fig:ode_ab_fid} show that SGFM-based methods generate samples of higher physical consistency compared to $g^{\text{cov-A}}$. Furthermore, SGFM-HMC achieves the most representative distribution over the parameters. In contrast, SGFM-OPT collapses to the modes for the conditioned parameters $\theta_a,\theta_b$, and importantly fails to capture the full admissible range of the unconditioned parameter $\theta_\omega$. $g^{\text{cov-A}}$, on the other hand, captures the full range of admissible $\theta_\omega$ but excessively generates values outside of the credible intervals, sometimes outside of the training data distribution. Thus, we conclude that SGFM-HMC best approximates the target distribution.
\begin{figure}[ht!]
\begin{center}
\includegraphics[width=\linewidth]{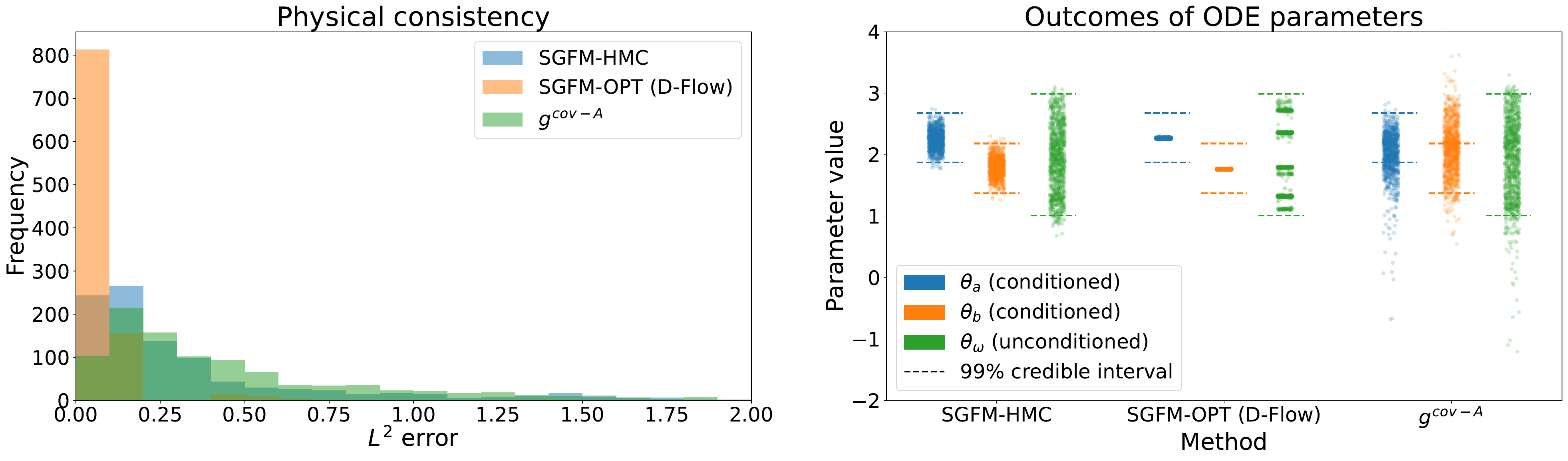}
\caption{Physical consistency of samples and closeness of the empirical distribution to the target distribution. The physical consistency is measured by the relative $L^2$ error between the sampled trajectory and the true trajectory under the jointly sampled ODE parameters.  Each method's ability to capture the target distribution is then assessed by analyzing the empirical distribution of the sampled ODE parameters. In the ideal case, 99\% of the samples fall within the credible interval indicated by the dashed lines. The true marginal target distributions for the conditioned parameters $\theta_a, \theta_b$ behave like posteriors with a uniform prior and Gaussian likelihood, so the outcomes are expected to distribute smoothly across the bounds. Similarly, the true marginal target distribution for the unconditioned parameter $\theta_\omega$ is a uniform distribution, so the outcomes are expected to be distributed uniformly across the bound.}
   \label{fig:ode_ab_fid}
\end{center}
\end{figure}
\subsubsection{Implementation details}
\paragraph{Flow matching model:} The vector field defining the flow-matching model is approximated using an MLP similar to \cite{tong2023improving} with four hidden layers of size 256 and SELU activation functions. Furthermore, we add a Gaussian smoothing filter with non-learnable parameters on the last layer to encourage smooth solutions. The source distribution is a standard Gaussian distribution for the ODE parameters and a Gaussian process with zero mean and squared exponential kernel having length scale $l=1$. The dataset consists of $10^4$ samples which are generated by sampling $\theta\sim\mathcal{U}(1,3)^3$ and integrating \eqref{eq:ode} for $t\in[0, 5]$ using the Euler method with $\Delta t=0.05$. The model is trained for $10^3$ epochs using the Adam optimizer with learning rate $\eta=10^{-3}$.

\paragraph{Conditional sampling:} All methods are initialized by the same set of samples from the unmodified source distribution. To balance the scale of the cost function $J$ and the prior probability $\log{q_0}$, the cost is scaled by a factor $\frac{1}{\lambda}$ where $\lambda=5\times 10^{-2}$. 
\paragraph{Implementation of SGFM-HMC:} SGFM-HMC is implemented by running the HMC algorithm for $N_{HMC}=200$ steps with $L=50$ leapfrog steps, where the step size is randomly selected in each Markov chain iteration as $\epsilon= 10^{-4}( 1 + \zeta\times 15)$ where $\zeta\sim \chi^2(2)$ with $\chi^2(2)$ being the chi-squared distribution with two degrees of freedom. We found that this setting gives good acceptance ratios while allowing for a significant number of HMC iterations to be performed without having too long runtimes. The transport map is obtained by integrating the neural ODE associated with the vector field for two steps using the Dormand-Prince (Dopri5) method. We use the same transportation map both for the density computation in the HMC iterations and to map the sampled source point to the target space.
\paragraph{Implementation of SGFM-OPT:} SGFM-OPT is implemented using L-BFGS optimization with learning rate $\eta=1$, maximum iterations of 20, and history size 100. The method is allowed to run for approximately the same amount of runtime as HMC (which corresponds to 160 optimization steps), but usually converges before that. The transport map is designed as in SGFM-HMC. We use the same transportation map both for the density computation in the D-Flow iterations and to map the sampled source point to the target space.
\paragraph{Implementation of $g^{\text{cov-A}}$:} $g^{\text{cov-A}}$ \cite{feng2025guidance} is implemented using a constant schedule $\lambda_t^{\text{covA}}=\lambda$. This choice cancels out the loss scaling factor $\frac{1}{\lambda}$, and other alternatives either degrade physical consistency or impose weak constraints on the conditioned parameters. The ODE is integrated for three steps using the Dopri5 method, additional steps had no effect on performance.
\section{Limitations and Future Work}
Similar to other guidance methods that optimize the source samples, one limitation of our method lies in the long runtime due to the need to backpropagate through the ODE. Consequently, it would be interesting to incorporate efficient backpropagation in our framework. Additionally, training the optimal vector field requires access to the OT coupling, which becomes particularly challenging in high-dimensional settings. Although we can approximate $\pi^*$ using mini-batch data or entropic OT solvers, these approximations can introduce bias and may not scale well. Developing a more efficient and scalable approach to training the optimal vector field is another important avenue for future research.

%% file: neurips_2025.bbl
\begin{thebibliography}{10}

\bibitem{lipman2022flow}
Yaron Lipman, Ricky~TQ Chen, Heli Ben-Hamu, Maximilian Nickel, and Matt Le.
\newblock Flow matching for generative modeling.
\newblock {\em arXiv preprint arXiv:2210.02747}, 2022.

\bibitem{chen2023flow}
Ricky~TQ Chen and Yaron Lipman.
\newblock Flow matching on general geometries.
\newblock {\em arXiv preprint arXiv:2302.03660}, 2023.

\bibitem{zheng2023guided}
Qinqing Zheng, Matt Le, Neta Shaul, Yaron Lipman, Aditya Grover, and Ricky~TQ Chen.
\newblock Guided flows for generative modeling and decision making.
\newblock {\em arXiv preprint arXiv:2311.13443}, 2023.

\bibitem{tong2023improving}
Alexander Tong, Kilian Fatras, Nikolay Malkin, Guillaume Huguet, Yanlei Zhang, Jarrid Rector-Brooks, Guy Wolf, and Yoshua Bengio.
\newblock Improving and generalizing flow-based generative models with minibatch optimal transport.
\newblock {\em TMLR}, 2023.

\bibitem{Pooladian2023Multisample}
Aram-Alexandre Pooladian, Heli Ben-Hamu, Carles Domingo-Enrich, Brandon Amos, Yaron Lipman, and Ricky T.~Q. Chen.
\newblock Multisample flow matching: straightening flows with minibatch couplings.
\newblock In {\em Proceedings of the 40th International Conference on Machine Learning}, ICML'23. JMLR.org, 2023.

\bibitem{dhariwal2021diffusion}
Prafulla Dhariwal and Alexander Nichol.
\newblock Diffusion models beat gans on image synthesis.
\newblock {\em Advances in neural information processing systems}, 34:8780--8794, 2021.

\bibitem{du2023reduce}
Yilun Du, Conor Durkan, Robin Strudel, Joshua~B Tenenbaum, Sander Dieleman, Rob Fergus, Jascha Sohl-Dickstein, Arnaud Doucet, and Will~Sussman Grathwohl.
\newblock Reduce, reuse, recycle: Compositional generation with energy-based diffusion models and mcmc.
\newblock In {\em International conference on machine learning}, pages 8489--8510. PMLR, 2023.

\bibitem{graikos2022diffusion}
Alexandros Graikos, Nikolay Malkin, Nebojsa Jojic, and Dimitris Samaras.
\newblock Diffusion models as plug-and-play priors.
\newblock {\em Advances in Neural Information Processing Systems}, 35:14715--14728, 2022.

\bibitem{ho2022classifier}
Jonathan Ho and Tim Salimans.
\newblock Classifier-free diffusion guidance.
\newblock {\em arXiv preprint arXiv:2207.12598}, 2022.

\bibitem{song2020score}
Yang Song, Jascha Sohl-Dickstein, Diederik~P Kingma, Abhishek Kumar, Stefano Ermon, and Ben Poole.
\newblock Score-based generative modeling through stochastic differential equations.
\newblock {\em arXiv preprint arXiv:2011.13456}, 2020.

\bibitem{chung2022diffusion}
Hyungjin Chung, Jeongsol Kim, Michael~T Mccann, Marc~L Klasky, and Jong~Chul Ye.
\newblock Diffusion posterior sampling for general noisy inverse problems.
\newblock {\em arXiv preprint arXiv:2209.14687}, 2022.

\bibitem{song2023loss}
Jiaming Song, Qinsheng Zhang, Hongxu Yin, Morteza Mardani, Ming-Yu Liu, Jan Kautz, Yongxin Chen, and Arash Vahdat.
\newblock Loss-guided diffusion models for plug-and-play controllable generation.
\newblock In {\em International Conference on Machine Learning}, pages 32483--32498. PMLR, 2023.

\bibitem{ye2024tfg}
Haotian Ye, Haowei Lin, Jiaqi Han, Minkai Xu, Sheng Liu, Yitao Liang, Jianzhu Ma, James~Y Zou, and Stefano Ermon.
\newblock Tfg: Unified training-free guidance for diffusion models.
\newblock {\em Advances in Neural Information Processing Systems}, 37:22370--22417, 2024.

\bibitem{uehara2024fine}
Masatoshi Uehara, Yulai Zhao, Kevin Black, Ehsan Hajiramezanali, Gabriele Scalia, Nathaniel~Lee Diamant, Alex~M Tseng, Tommaso Biancalani, and Sergey Levine.
\newblock Fine-tuning of continuous-time diffusion models as entropy-regularized control.
\newblock {\em arXiv preprint arXiv:2402.15194}, 2024.

\bibitem{tang2024fine}
Wenpin Tang.
\newblock Fine-tuning of diffusion models via stochastic control: entropy regularization and beyond.
\newblock {\em arXiv preprint arXiv:2403.06279}, 2024.

\bibitem{ben2024d}
Heli Ben-Hamu, Omri Puny, Itai Gat, Brian Karrer, Uriel Singer, and Yaron Lipman.
\newblock D-flow: Differentiating through flows for controlled generation.
\newblock {\em arXiv preprint arXiv:2402.14017}, 2024.

\bibitem{wang2024training}
Luran Wang, Chaoran Cheng, Yizhen Liao, Yanru Qu, and Ge~Liu.
\newblock Training free guided flow matching with optimal control.
\newblock {\em arXiv preprint arXiv:2410.18070}, 2024.

\bibitem{liu2023flowgrad}
Xingchao Liu, Lemeng Wu, Shujian Zhang, Chengyue Gong, Wei Ping, and Qiang Liu.
\newblock Flowgrad: Controlling the output of generative odes with gradients.
\newblock In {\em Proceedings of the IEEE/CVF Conference on Computer Vision and Pattern Recognition}, pages 24335--24344, 2023.

\bibitem{domingo2024adjoint}
Carles Domingo-Enrich, Michal Drozdzal, Brian Karrer, and Ricky~TQ Chen.
\newblock Adjoint matching: Fine-tuning flow and diffusion generative models with memoryless stochastic optimal control.
\newblock {\em arXiv preprint arXiv:2409.08861}, 2024.

\bibitem{feng2025guidance}
Ruiqi Feng, Tailin Wu, Chenglei Yu, Wenhao Deng, and Peiyan Hu.
\newblock On the guidance of flow matching.
\newblock {\em arXiv preprint arXiv:2502.02150}, 2025.

\bibitem{figalli2021invitation}
Alessio Figalli and Federico Glaudo.
\newblock {\em An invitation to optimal transport, Wasserstein distances, and gradient flows}.
\newblock EMS, 2021.

\bibitem{villani2008optimal}
C{\'e}dric Villani et~al.
\newblock {\em Optimal transport: old and new}, volume 338.
\newblock Springer, 2008.

\bibitem{villani2021topics}
C{\'e}dric Villani.
\newblock {\em Topics in optimal transportation}, volume~58.
\newblock American Mathematical Soc., 2021.

\bibitem{kornilov2024optimal}
Nikita Kornilov, Petr Mokrov, Alexander Gasnikov, and Aleksandr Korotin.
\newblock Optimal flow matching: Learning straight trajectories in just one step.
\newblock {\em Advances in Neural Information Processing Systems}, 37:104180--104204, 2024.

\bibitem{liu2022flow}
Xingchao Liu, Chengyue Gong, and Qiang Liu.
\newblock Flow straight and fast: Learning to generate and transfer data with rectified flow.
\newblock {\em arXiv preprint arXiv:2209.03003}, 2022.

\bibitem{Chopin2020}
Nicolas Chopin and Omiros Papaspiliopoulos.
\newblock {\em Importance Sampling}, pages 81--103.
\newblock Springer International Publishing, Cham, 2020.

\bibitem{Villani2009}
C{\'e}dric Villani.
\newblock {\em The Wasserstein distances}, pages 93--111.
\newblock Springer Berlin Heidelberg, Berlin, Heidelberg, 2009.

\bibitem{neal2011mcmc}
Radford~M Neal et~al.
\newblock {{MCMC}} using {{Hamiltonian}} dynamics.
\newblock {\em Handbook of Markov Chain Monte Carlo}, 2011.

\bibitem{chen2019optimal}
Zongchen Chen and Santosh~S Vempala.
\newblock Optimal convergence rate of {{Hamiltonian Monte Carlo}} for strongly logconcave distributions.
\newblock {\em arXiv preprint arXiv:1905.02313}, 2019.

\bibitem{robert1999monte}
Christian~P Robert, George Casella, and George Casella.
\newblock {\em Monte Carlo statistical methods}, volume~2.
\newblock Springer, 1999.

\bibitem{bastek2024physics}
Jan-Hendrik Bastek, WaiChing Sun, and Dennis~M Kochmann.
\newblock Physics-informed diffusion models.
\newblock {\em arXiv preprint arXiv:2403.14404}, 2024.

\bibitem{jacobsen2025cocogen}
Christian Jacobsen, Yilin Zhuang, and Karthik Duraisamy.
\newblock Cocogen: Physically consistent and conditioned score-based generative models for forward and inverse problems.
\newblock {\em SIAM Journal on Scientific Computing}, 47(2):C399--C425, 2025.

\bibitem{martin2024pnp}
S{\'e}gol{\`e}ne Martin, Anne Gagneux, Paul Hagemann, and Gabriele Steidl.
\newblock Pnp-flow: Plug-and-play image restoration with flow matching.
\newblock {\em arXiv preprint arXiv:2410.02423}, 2024.

\bibitem{santambrogio2015optimal}
Filippo Santambrogio.
\newblock {\em Optimal transport for applied mathematicians}, volume~87.
\newblock Springer, 2015.

\bibitem{benton2023error}
Joe Benton, George Deligiannidis, and Arnaud Doucet.
\newblock Error bounds for flow matching methods.
\newblock {\em arXiv preprint arXiv:2305.16860}, 2023.

\bibitem{guo2024gradient}
Yingqing Guo, Hui Yuan, Yukang Yang, Minshuo Chen, and Mengdi Wang.
\newblock Gradient guidance for diffusion models: An optimization perspective.
\newblock {\em arXiv preprint arXiv:2404.14743}, 2024.

\bibitem{wu2023practical}
Luhuan Wu, Brian Trippe, Christian Naesseth, David Blei, and John~P Cunningham.
\newblock Practical and asymptotically exact conditional sampling in diffusion models.
\newblock {\em Advances in Neural Information Processing Systems}, 36:31372--31403, 2023.

\bibitem{xu2024provably}
Xingyu Xu and Yuejie Chi.
\newblock Provably robust score-based diffusion posterior sampling for plug-and-play image reconstruction.
\newblock {\em arXiv preprint arXiv:2403.17042}, 2024.

\bibitem{wallace2023end}
Bram Wallace, Akash Gokul, Stefano Ermon, and Nikhil Naik.
\newblock End-to-end diffusion latent optimization improves classifier guidance.
\newblock In {\em Proceedings of the IEEE/CVF International Conference on Computer Vision}, pages 7280--7290, 2023.

\bibitem{tang2024inference}
Zhiwei Tang, Jiangweizhi Peng, Jiasheng Tang, Mingyi Hong, Fan Wang, and Tsung-Hui Chang.
\newblock Inference-time alignment of diffusion models with direct noise optimization.
\newblock {\em arXiv preprint arXiv:2405.18881}, 2024.

\bibitem{novack2024ditto}
Zachary Novack, Julian McAuley, Taylor Berg-Kirkpatrick, and Nicholas~J Bryan.
\newblock Ditto: Diffusion inference-time t-optimization for music generation.
\newblock {\em arXiv preprint arXiv:2401.12179}, 2024.

\bibitem{karunratanakul2024optimizing}
Korrawe Karunratanakul, Konpat Preechakul, Emre Aksan, Thabo Beeler, Supasorn Suwajanakorn, and Siyu Tang.
\newblock Optimizing diffusion noise can serve as universal motion priors.
\newblock In {\em Proceedings of the IEEE/CVF Conference on Computer Vision and Pattern Recognition}, pages 1334--1345, 2024.

\end{thebibliography}
